\documentclass{article}

    \PassOptionsToPackage{numbers, compress}{natbib}

    \usepackage[final]{neurips_2025}

\usepackage{microtype}
\usepackage{graphicx}
\usepackage{subfigure}
\usepackage{wrapfig,lipsum,booktabs}       
\usepackage{amsfonts}       
\usepackage{nicefrac}       
\usepackage[dvipsnames]{xcolor}  
\usepackage{colortbl}
\usepackage{amssymb}
\usepackage{pifont}
\usepackage{rotating}
\usepackage{makecell}

\usepackage{bbm}

\usepackage{caption}
\usepackage{subcaption}
\usepackage{amsmath}
\usepackage{xspace}
\usepackage{multirow}

\usepackage[titletoc]{appendix}

\usepackage{pdflscape}

\usepackage{tabularray}

\usepackage{soul}

\usepackage[colorlinks=true,citecolor=brown,urlcolor=gray,linkcolor=BlueViolet]{hyperref}

\allowdisplaybreaks

\usepackage[compact]{titlesec}
\usepackage{sidecap}

\usepackage{enumitem}
\usepackage{arydshln} 
\usepackage[utf8]{inputenc} 
\usepackage[T1]{fontenc}    

\newcommand{\cmark}{\ding{51}}%
\newcommand{\xmark}{\ding{55}}%

\newcommand{\ours}{\texttt{TFB}\xspace}
\newcommand{\blob}{BLoB\xspace}
\newcommand{\loraens}{ENS\xspace}
\newcommand{\loramle}{MLE\xspace}
\newcommand{\loramap}{MAP\xspace}
\newcommand{\loramcd}{MCD\xspace}
\newcommand{\loralap}{LAP\xspace}
\newcommand{\loranaivebbb}{BBB\xspace}

\usepackage{siunitx}
\newlength\savewidth

\newcolumntype{C}{>{\centering\let\newline\\\arraybackslash\hspace{0pt}}m{2cm}}

\definecolor{lightergray}{HTML}{e5e5e5}
\sethlcolor{lightergray}

\usepackage{amsmath,amsfonts,bm,eqnarray}
\usepackage{cancel}
\usepackage{amsthm}
\usepackage{tikz}
\usetikzlibrary{tikzmark}

\newtheorem{assumption}{Assumption}[section]

\newtheorem{theorem}{Theorem}[section]

\newtheorem{lemma}[theorem]{Lemma}
\newtheorem*{remark}{Remark}
\newtheorem{proposition}{Proposition}[theorem]

\newtheorem{innercustomgeneric}{\customgenericname}
\providecommand{\customgenericname}{}
\newcommand{\newcustomtheorem}[2]{%
  \newenvironment{#1}[1]
  {%
   \renewcommand\customgenericname{#2}%
   \renewcommand\theinnercustomgeneric{##1}%
   \innercustomgeneric
  }
  {\endinnercustomgeneric}
}

\newcustomtheorem{customThm}{Theorem}
\newcustomtheorem{customLemma}{Lemma}
\newcustomtheorem{customCor}{Corollary}
\newcustomtheorem{customProposition}{Proposition}

\newcommand{\vectorize}{\operatorname{vec}}
\newcommand{\diag}{\operatorname{diag}}

\def\Assumptionref#1{Assumption~\ref{#1}}

\def\Tabref#1{Table~\ref{#1}}

\def\Lmmref#1{Lemma~\ref{#1}}
\def\Thmref#1{Theorem~\ref{#1}}

\def\Figref#1{Fig.~\ref{#1}}

\def\appref#1{Appendix~\ref{#1}}
\def\Secref#1{Sec.~\ref{#1}}

\def\eqref#1{equation~\ref{#1}}
\def\Eqref#1{Eqn.~\ref{#1}}

\def\Algref#1{Algorithm~\ref{#1}}

\def\1{\bm{1}}

\def\vzero{{\bm{0}}}

\def\vmu{{\bm{\mu}}}
\def\vtheta{{\bm{\theta}}}

\def\vb{{\bm{b}}}

\def\vd{{\bm{d}}}

\def\vh{{\bm{h}}}

\def\vx{{\bm{x}}}
\def\vy{{\bm{y}}}
\def\vz{{\bm{z}}}

\def\mA{{\bm{A}}}
\def\mB{{\bm{B}}}

\def\mE{{\bm{E}}}

\def\mI{{\bm{I}}}

\def\mM{{\bm{M}}}

\def\mP{{\bm{P}}}

\def\mU{{\bm{U}}}
\def\mV{{\bm{V}}}
\def\mW{{\bm{W}}}
\def\mX{{\bm{X}}}

\def\mSigma{{\bm{\Sigma}}}
\def\mOmega{{\bm{\Omega}}}

\DeclareMathAlphabet{\mathsfit}{\encodingdefault}{\sfdefault}{m}{sl}
\SetMathAlphabet{\mathsfit}{bold}{\encodingdefault}{\sfdefault}{bx}{n}

\def\gD{{\mathcal{D}}}

\def\gN{{\mathcal{N}}}

\newcommand{\E}{\mathbb{E}}

\newcommand{\R}{\mathbb{R}}

\renewcommand{\tilde}{\widetilde}
\renewcommand{\hat}{\widehat}
\renewcommand{\frac}{\tfrac}

\DeclareUrlCommand\Code{\urlstyle{rm}}
\expandafter\def\expandafter\UrlBreaks\expandafter{\UrlBreaks  
\do\/\do\a\do\b\do\c\do\d\do\e\do\f\do\g\do\h\do\i\do\j\do\k
\do\l\do\m\do\n\do\o\do\p\do\q\do\r\do\s\do\t\do\u\do\v
\do\w\do\x\do\y\do\z
\do\A\do\B\do\C\do\D\do\E\do\F\do\G\do\H\do\I\do\J\do\K
\do\L\do\M\do\N\do\O\do\P\do\Q\do\R\do\S\do\T\do\U\do\V
\do\W\do\X\do\Y\do\Z}

\title{Training-Free Bayesianization for Low-Rank Adapters of Large Language Models}

\author{%
  Haizhou Shi\thanks{Equal Contribution. 
  $^1$Rutgers University.
  $^2$University of Illinois Urbana-Champaign (UIUC).
  $^3$Red Hat AI Innovation.
  $^{\text{\textdagger}}$Correspondence to: Haizhou Shi <haizhou.shi@rutgers.edu>, 
  Hao Wang <hw488@cs.rutgers.edu>.}\space\space{$^{1}$}
  \\
  \And
  Yibin Wang $^{* 2}$\\
  \And
  Ligong Han $^{3}$\\
  \And
  Huan Zhang $^{2}$\\
  \And 
  Hao Wang $^{\text{\textdagger} 1}$\\
}

\begin{document}

\maketitle

\begin{abstract}
    Estimating the uncertainty of responses from Large Language Models~(LLMs) remains a critical challenge. While recent Bayesian methods have demonstrated effectiveness in quantifying uncertainty through low-rank weight updates, they typically require complex fine-tuning or post-training procedures. 
    In this paper, we propose \textbf{T}raining-\textbf{F}ree \textbf{B}ayesianization~({\ours}), a simple yet theoretically grounded framework that {efficiently} transforms trained low-rank adapters into Bayesian ones without additional training. 
    \ours systematically searches for the maximally acceptable level of variance in the weight posterior, constrained within a family of low-rank isotropic Gaussian distributions. 
    Our theoretical analysis shows that under mild conditions, this search process is equivalent to KL-regularized variational optimization, a generalized form of variational inference. 
    Through comprehensive experiments, we show that \ours achieves superior uncertainty estimation and generalization compared to existing methods while eliminating the need for complex Bayesianization training procedures. 
    {Code is available at \url{https://github.com/Wang-ML-Lab/bayesian-peft}.}
\end{abstract}

\section{Introduction}
\label{sec:intro}
Despite recent advances in Large Language Models (LLMs) showing great capacity for generating responsive answers to human instructions~\cite{biderman2023pythia,wei2022emergent,wei2021finetuned,min2022rethinking,chowdhery2023palm,anil2023palm,touvron2023llama,touvron2023llama2,radford2019language,brown2020language,achiam2023gpt,achiam2022chatgpt}, the reliability of such large models remains a critical concern~\cite{wang2024probabilistic,wang2024variational}, as untruthful yet confident answers could cause significant damage to individuals and society~\cite{gupta2024language, nikitin2024kernel, yadkori2024believe, kapoor2024large}. The accurate estimation of uncertainty in LLMs has thus emerged as an urgent challenge. Current approaches mainly follow two paths: one focuses on 
directly asking the model to elicit its internal internal (verbalized) uncertainty~\cite{xiong2023can,tian2023just, kapoor2024large}, 
while the other employs complex fine-tuning techniques~\cite{kapoor2024large,yang2023bayesian,wang2024blob}.



Both approaches suffer from inherent limitations. Verbalized uncertainty, while simple to implement, remains controversial in terms of its empirical reliability and \textbf{\emph{theoretical soundness}}~\cite{kadavath2022language, kuhn2023semantic}.
On the other hand, low-rank adapters (LoRA~\cite{hu2022lora}), which offer a parameter-efficient way to adapt LLMs by adding a small set of low-rank weight matrices, have emerged as a promising direction for fine-tuning models. However, while LoRA efficiently adapts large models to new tasks, it does not itself provide a mechanism for principled uncertainty estimation. In response, recent Bayesianization attempts~\cite{yang2023bayesian, wang2024blob}, integrate Bayesian methods with LoRA, but they still require complex training procedures and sophisticated hyperparameter tuning, \textbf{\emph{limiting their practicality}}. These constraints motivate the following research question:


\begin{centering}
\textit{Can we ``Bayesianize'' LLM low-rank adapters in a  \textbf{theoretically sound} yet \textbf{empirically simple} way?}
\end{centering}


In this paper, 
we diverge from conventional fine-tuning and post-training approaches. Instead, we develop a Training-Free Bayesianization (\ours) technique applicable to \emph{any} given low-rank LLM adapter. 
\ours constrains the family of full-weight approximate posteriors produced by LoRA adapters to low-rank isotropic Gaussian distributions. 
Given a trained LoRA adapter, it systematically searches for the maximally acceptable variance of the variational distribution of the weight posterior, without the need for complex fine-tuning procedures. 
{\ours's search range and stopping criteria can be determined using \emph{any} 
in-distribution ``anchor dataset,'' e.g., a small subset of the training dataset. Note that (1) this eliminates the need for an additional calibration or validation dataset; (2) this flexibility extends to both supervised and unsupervised data, even regardless of whether it was used in the original LoRA training.}

Despite its simplicity, we theoretically demonstrate that, \ours's process of finding the maximal variance of the low-rank isotropic Gaussian posterior is equivalent to generalized variational inference, under mild conditions.

We verify \ours's effectiveness through extensive empirical evaluation across various settings, datasets, LLM backbones, LoRA weights, and LoRA variants. Our comprehensive experiments demonstrate that this novel training-free Bayesianization framework consistently achieves superior generalization and more accurate uncertainty estimation. To summarize, the main contributions of this paper are:
\begin{itemize}[nosep]
    \item We propose Training-Free Bayesianization (\ours), the first framework to transform trained LoRAs into Bayesian ones without re-training, continued training, or gradient estimation.
    \item We establish theoretical connections between \ours and generalized variational inference, proving their equivalence under mild conditions.
    \item We develop an efficient implementation of \ours requiring only an anchor dataset for search, making it widely applicable across different application scenarios.
    \item Through comprehensive experiments, we demonstrate that \ours consistently improves uncertainty estimation for off-the-shelf LoRA adapters, and overall surpasses the state-of-the-art counterparts of Bayesian LoRA. 
\end{itemize}

\section{Related Work}
\label{sec:related}

\textbf{LLM Uncertainty Estimation.}\quad
To estimate the uncertainty of LLMs, the models are often employed to generate and evaluate their own uncertainty~\cite{lin2022teaching, kadavath2022language}. However, such approaches typically rely on task-specific labels and require additional training. Semantic entropy~\cite{kuhn2023semantic} leverages the invariance of language stemming from shared meanings to estimate uncertainty, while mutual information is used to compute a lower bound on model uncertainty by sampling from the model's output distribution~\cite{yadkori2024believe}. Despite their contributions, these methods fail to accurately capture true model uncertainty, as they do not model the probability distribution over the LLM parameters~\cite{hullermeier2021aleatoric, abdar2021review, gawlikowski2023survey}.

\textbf{Bayesian Low-Rank Adaptation.}\quad
The Bayesian framework provides a powerful approach for capturing and estimating uncertainty during fine-tuning by defining prior distributions and approximating posterior distributions over the model parameters~\cite{neal2012bayesian,hernandez2015probabilistic,gal2016dropout, wang2016towards}. 
Recent research has explored combining Bayesian methods with LoRA to mitigate the additional computational overhead associated with modeling parameter distributions across the entire parameter space.
\citet{yang2023bayesian} applies a Kronecker-factorized Laplace approximation to fine-tuned LoRA parameters. More recently, \blob~\cite{wang2024blob} advances the field by simultaneously estimating both the mean and covariance of LLM parameters within a single fine-tuning stage.
Our proposed training-free Bayesianization represents a significant departure from these existing methods. 
Unlike approaches that require re-training~\cite{gal2016dropout, wang2023lora, balabanov2024uncertainty, wang2024blob} or rely on continued training and gradient estimation~\cite{yang2023bayesian}, our method achieves uncertainty estimation without any additional training steps, substantially improving the simplicity and efficiency for Bayesian learning of LLMs.


\section{Training-Free Bayesianization~(\ours)}
\label{sec:method}

This section introduces our Training-Free Bayesianization~(\ours). 
\Secref{sec:method-pre} introduces the problem setup. 
\Secref{sec:method-posterior} and \Secref{sec:method-tfb} present the two key parts of \ours: low-rank Gaussian variational distribution family and a novel approach for converting deterministic weights to probabilistic distributions without training. 
The complete algorithmic implementation is provided in \Secref{sec:method-final-algo}, with theoretical foundations addressed in a separate section~(\Secref{sec:theory}).

\textbf{Notation.}\quad
Scalars, vectors, and matrices are denoted by lowercase letters, lowercase boldface letters, and uppercase boldface letters, respectively. 
For a matrix $\mX=[\vx_1, \cdots, \vx_n]\in \mathbb{R}^{m\times n}$, we use $\vectorize(\mX)= 
[\vx_1^\top, \vx_2^\top, \cdots, \vx_n^\top]^\top
\in \mathbb{R}^{(mn)\times 1}$ to denote vectorization. 
$\otimes$ and $\circ$ denote the Kronecker and element-wise product, respectively. 
We use $\mathbf{0}_{n}\in \mathbb{R}^{n\times n}$ to denote a zero matrix.

\subsection{Preliminaries}
\label{sec:method-pre}

\textbf{Low-Rank Adaptation~(LoRA).}\quad
Given a pre-trained neural network layer with weight matrix $\mW_0$, Low-Rank Adaptation~(LoRA)~\cite{hu2022lora} confines weight updates to a low-rank subspace during fine-tuning, expressing the update as $\Delta\mW = \mB \mA$, where $\Delta\mW\in \mathbb{R}^{m\times n}$, $\mB\in \mathbb{R}^{m\times r}$, and $\mA\in \mathbb{R}^{r\times n}$.
For input $\vh$ and output $\vz$ of the LoRA layer, the forward pass computation is then given by:
\begin{align}
    \vz &= \mW_0\vh + \Delta\mW \vh = \mW_0\vh + \mB\mA \vh.
\end{align}
%
%

\begin{figure}[t] 
\centering 
\includegraphics[width=1\textwidth]{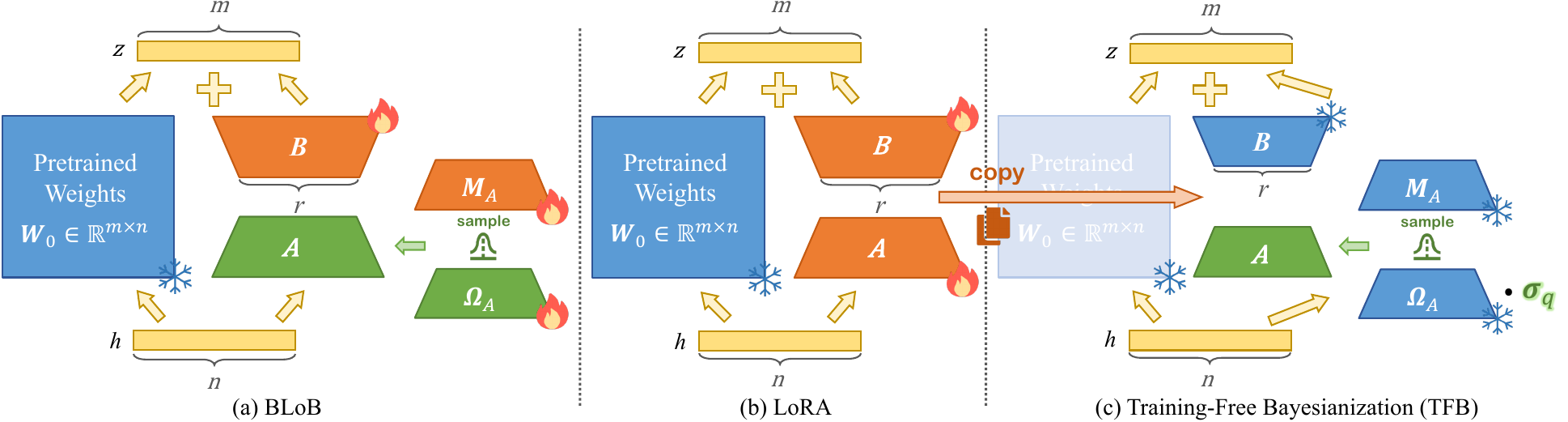} 
\caption{
{Overview of \textbf{T}raining-\textbf{F}ree \textbf{B}ayesianization~(\ours,~Ours,~\textbf{right}), as well as comparison with existing methods such as LoRA \textbf{(middle)} and BLoB \textbf{(left)}. }
} 
\label{fig:tfb_overview} 
\end{figure}

\textbf{LoRA Bayesianization with Low-Rank Gaussian Distribution.}\quad
\blob~\cite{wang2024blob}, a pioneering work in low-rank Bayesianization for LLMs, empirically demonstrates that modeling $\mA$'s elements with independent Gaussian variables suffices for effective uncertainty estimation in LoRA.
Specifically, the probability density of each element of $\mA$ follows
$q(A_{ij})=\gN(A_{ij}|M_{ij}, \Omega_{ij}^2), \forall i \in [r], \forall j \in [n]$, where matrices $\mM$ and $\mOmega$, sharing the dimensions of $\mA$, represent the mean and standard deviation of the random variable $\mA$, respectively. 
This formulation is equivalent to approximating the Bayesianized low-rank adapter's posterior in the full-weight space of $\mW$ with a low-rank degenerate distribution:
\begin{align}
    q(\vectorize(\mW)|\mB, \vtheta) &= \gN(\vectorize(\mW)|\vmu_q, \mSigma_q),\label{eq:blob-posterior}
\end{align}
where $\vtheta=\{\mM, \mOmega\}$ denotes the set of parameters for modeling $\mA$'s posterior distribution, 
$\vmu_q=\vectorize(\mW_0+\mB\mM)$ is its mean,
and $\mSigma_q = [\mI_n \otimes \mB] [\diag(\vectorize(\mOmega)^2)] [\mI_n \otimes \mB^\top]$ is its low-rank degenerate covariance matrix.
In this paper, we adopt a similar approach for modeling the variational distribution of the weight posterior, focusing exclusively on Bayesianizing the weight update matrix $\mA$. 




\subsection{\ours's Variational Low-Rank Isotropic Gaussians}

\label{sec:method-posterior}

\textbf{Variational Distribution Family.}\quad
In \ours, we constrain the variational distributions of the weight posterior to a more compact family of Gaussians than \blob: specifically, we employ full-space isotropic Gaussian distributions projected onto the low-rank space:
\begin{align}
    q(\vectorize(\mW)|\mB, \vtheta) &= \gN(\vectorize(\mW)|\vmu_q, \operatorname{proj}(\sigma_q^2\mI)),\label{eq:tfb-posterior}
\end{align}
where $\vmu_q$ is defined as in \Eqref{eq:blob-posterior}.
Here, $\sigma_q^2\mI\in \mathbb{R}^{mn\times mn}$ represents a full-rank isotropic covariance matrix with standard deviation $\sigma_q$, and $\operatorname{proj}(\cdot)$ denotes a linear projection operator that maps the full-space covariance matrix onto the low-rank space {(see Proposition~\ref{lemma:tfb-proj} for details)}.
\footnote{
    $\operatorname{proj}(\cdot)$ only depends on the rank $r$ of the trained LoRA.
}

\textbf{\ours as Generalized Variational Inference.}\quad
The choice of low-rank isotropic Gaussian approximate posteriors serves both \emph{theoretical} and \emph{empirical} purposes: it provides a single-parameter family that enables converting the generalized variational inference into a variance maximization problem~(more details in \Secref{sec:method-tfb}, \Thmref{thm:tfb}, and \appref{app:visualization}), and empirically outperforms alternative distribution families~(\Secref{sec:ablation}).
Below, we present a practically efficient implementation for Bayesianizing LoRA under the constraint specified in \Eqref{eq:tfb-posterior}, with detailed theoretical analysis provided in \Thmref{thm:posterior}.


\textbf{\ours in Practice.}\quad 
Consider a LoRA layer with weight updates $\mB\in\mathbb{R}^{m\times r},\mA\in\mathbb{R}^{r\times n}$ and a standard deviation scale $\sigma_q>0$.
We begin by computing the compact Singular Value Decomposition~(SVD)~\cite{klema1980singular} of $\mB$:
\begin{align}
    \mB &= \mU \diag(\vd) \mV^\top,\label{eq:svd}
\end{align}
where $\mU\in\mathbb{R}^{m\times r}$ and $\mV\in\mathbb{R}^{r\times r}$ are orthonormal matrices, and $\vd=[d_1, d_2, \cdots, d_r]^\top$ is the vector consisting of singular values with all positive entries\footnote{
    By stating $\vd \succ \vzero$, we assume $\mB$ has the full column rank $r$, which usually holds for LLM adaptation. 
}.
We then transform the original weight matrices $\{\mB, \mA\}$ into an equivalent pair
\begin{align}
    \{\mB^\prime &= \mU \diag(\vd),\space \mA^\prime = \mV^\top \mA\},\label{eq:regroup}
\end{align}
maintaining the equality $\Delta\mW = \mB\mA = \mB^\prime\mA^\prime$.
Following \blob's Asymmetric Bayesianization scheme, we define the variational distribution for $\mA^\prime$ using the mean matrix $\mM=\mA^\prime$ and the standard deviation matrix $\mOmega\in\mathbb{R}^{r\times n}$, such that 
\begin{align}
    q(A^\prime_{ij})=\gN(A^\prime_{ij}|M_{ij}, \Omega_{ij}^2), \forall i \in [r], \forall j \in [n].
\end{align}
Unlike \blob, our $\mOmega$ is not freely parameterized but instead derived from projecting the full-space matrix $\sigma_q\mI$ onto the low-rank weight space:
\begin{align}
    \Omega_{ij} &= \nicefrac{\sigma_q}{d_{i}}, \quad \forall i \in [r], \forall j \in [n],\label{eq:bayesianization}
\end{align}
where $\vd$ is defined in \Eqref{eq:svd}.
This solution can be expressed compactly as $\mOmega=[\nicefrac{\sigma_q}{\vd}, \cdots, \nicefrac{\sigma_q}{\vd}]$, comprising $n$ repeated vectors. 
To summarize, our \ours 
\begin{itemize}[nosep]
\item \textbf{takes as input} a trained LoRA matrix pair $\{\mB=\mU\diag(\vd)\mV^\top, \mA\}$ and a predetermined standard deviation $\sigma_q$, and 
\item \textbf{outputs} a ``Bayesianized'' LoRA adapter $\{\mB^\prime, \mA^\prime\}$, where 
$\mB^\prime=\mU\diag(\vd)$, and $\mA^\prime$ becomes a distribution 
$q(\mA^\prime)=\prod_{i\in [r],j\in [n]}\gN(A^\prime_{ij}|M_{ij}, 
\Omega_{ij}^2)$, with $\mM=\mV^\top\mA$, 
and $\mOmega=[\nicefrac{\sigma_q}{\vd}, \cdots, \nicefrac{\sigma_q}{\vd}].$
\end{itemize}

Note that the formulation in \Eqref{eq:bayesianization} significantly improves memory efficiency during inference, reducing the storage for standard deviation parameters from $O(rn)$ to $O(r)$. 
While alternative parameterization approaches are possible, they must be capable of generating the low-rank isotropic Gaussian noises as demonstrated in \Thmref{thm:posterior}. We have selected the current method (implementation) to ensure maximum compatibility with existing codebases~\cite{wang2024blob}.
In \ours, we use a single $\sigma_q$ shared across all LoRA layers. 
\subsection{\ours as Variance Maximization}
\label{sec:method-tfb}

The previous section presents a straightforward Bayesianization scheme 
\textbf{for a predetermined value of $\sigma_q$.} In this section, we describe a practical method for determining $\sigma_q$.


\textbf{A General Bayesianization Framework.}\quad
Consider an in-distribution ``anchor'' dataset $\gD$, an associated evaluation metric $l$, and a performance change tolerance $\epsilon$.
\ours determines $\sigma_q$ by solving a constrained optimization problem:
    \begin{equation}
    \begin{aligned}
        \max \quad & \sigma_q\\
        \textit{s.t.} \quad & |l(\gD|\mB^\prime, \mM, \mOmega(\sigma_q)) - l(\gD|\mB, \mA)| \leq \epsilon,
    \end{aligned}
    \label{eq:tfb}
    \end{equation}
where $l(\gD|\mB, \mA)$ and $l(\gD|\mB^\prime, \mM, \mOmega(\sigma_q)) = \E_{\mE \sim \gN(0, \mOmega^2)}[l(\gD|\mB^\prime, \mM + \mE)]$ denote the pre- and post-Bayesianization performance, respectively.
This optimization maximizes the noise scale $\sigma_q$ applied to model weights $\mM$ while ensuring that the resulting performance change remains within an acceptable threshold $\epsilon$.


\textbf{Anchor Dataset $\gD$ and Evaluation Metric $l$.}\quad
Our \emph{general} \ours framework accommodates various choices of anchor dataset $\gD$ and evaluation metric $l$ based on practical requirements. Below, we consider two key scenarios (with $N$ being slightly overloaded in its notation).

\underline{\emph{For supervised dataset $\gD=\{\vx_n, y_n\}_{n=1}^N$}}:
The Negative-Log Likelihood~(NLL) serves as a natural evaluation metric {in \Eqref{eq:tfb}}: $l_{\text{nll}}(\gD|\vtheta) = -\frac{1}{N}\sum_{n=1}^{N}\log P_\vtheta(y_n|\vx_n)$,
as it theoretically corresponds to minimizing the KL-regularized variational objective (more details in \Secref{sec:theory}).
The anchor dataset $\gD$ can be either the original training set used for the LoRA model or an independent calibration dataset, as commonly employed in calibration-based methods~\cite{guo2017calibration,zhao2021calibrate}.
Alternative evaluation metrics such as accuracy or F1 score are also readily applicable.
In our experimental setup, to ensure fair comparisons across uncertainty estimation baselines, we use the original training data as $\gD$ (maintaining the same information access as baselines) and employ NLL as the evaluation metric. 
Additional results with accuracy as $l$ can be found in \appref{app:more-acc-as-metric}. 

\underline{\emph{For unsupervised dataset $\gD=\{\vx_n\}_{n=1}^N$}}:
One approach is to generate pseudo-labels $\hat{y}$ using the model before Bayesianization, effectively converting the problem to the supervised case with $\gD=\{\vx_n, \hat{y}_n\}_{n=1}^N$.
{\ours can also directly incorporate purely unsupervised metrics such as the expected embedding norm $l_{\text{emb}}(\gD |\vtheta)=\E_{\vx \sim \gD}[\|\text{emb}(\vx | \vtheta)\|]$, where we are only concerned with properties of the representations themselves rather than any supervised signal.}
Hence our \ours offers substantially more flexibility compared to pure calibration methods, which typically rely on a labeled unseen calibration dataset. 
As a general framework, \ours also supports alternative evaluation metrics and statistical measures specifically designed for unsupervised data.

\textbf{Performance Change Tolerance $\epsilon$.}\quad
The selection of performance change tolerance $\epsilon$ is critical in \ours.
While our experiments demonstrate that a fixed relative change rate, i.e., $\nicefrac{\epsilon}{p_0}=0.3\%$ for NLL and $\nicefrac{\epsilon}{p_0}=1\%$ for accuracy, where $p_0$ denotes the pre-Bayesianization performance, can achieve effective uncertainty estimation across various datasets and LoRA checkpoints, an adaptive $\epsilon$ can further improve the performance of \ours. 
{Users can determine the appropriate value for $\epsilon$ by considering multiple factors simultaneously,}
among which the most important is the given LoRA checkpoint. For instance, an overfitted LoRA can typically accommodate a larger tolerance $\epsilon$ when using the training dataset (or its subset) as the anchor dataset. Additional properties of the data, model, and adaptation tasks can inform the choice of $\epsilon$ as well.

\subsection{\ours: Final Algorithm}
\label{sec:method-final-algo}

\textbf{Final \ours Algorithm: Automatically Determining $\sigma_q$.}\quad Our final algorithm, presented in \Algref{alg:tfb} and \Figref{fig:tfb_overview}, employs binary search to determine the optimal $\sigma_q^*$ within an initial range $[{\sigma_q}_{\text{min}}, {\sigma_q}_{\text{max}}]$. 
After identifying the optimal $\sigma_q^*$, we Bayesianize all LoRA layers using this value. 

\textbf{Prediction.}\quad For prediction, we average multiple outputs produced by samples from \ours's posterior:
\begin{equation}
    \begin{aligned}
        P_{\vtheta}(y|\vx)
        &= \E_{q(\mW|\vtheta)}[P(y|\vx,\mW)] 
        \approx \frac{1}{N} \sum\nolimits_{n=1}^{N} P(\vy|\vx,\mW_n), \quad \mW_n \sim q(\mW|\vtheta),
    \end{aligned}
    \label{eq:prediction}
\end{equation}
where $q(\mW|\vtheta)$ denotes the variational distribution defined in \Eqref{eq:tfb-posterior}, and we set the number of test-time samples to $N=10$, following \blob's protocol~\cite{wang2024blob}.

\textbf{Remark on \ours's Efficiency.}\quad
While \ours with binary search is efficient in terms of both time and memory~(\appref{app:more-efficiency}), and yields near-optimal solution of $\sigma_q$\footnote{While traditional search algorithms require monotonicity within the search range to guarantee optimal solutions, empirically a near-optimal $\sigma_q$ is sufficient for effective uncertainty estimation.}, more efficient parallel searching technique can be applied in practice. For instance, in \appref{app:more-llama2}, we conduct a grid search across 8 different $\sigma_q$ values in parallel, construct an approximate function $\hat{\sigma}_q(p)$ through piecewise linear interpolation of the observed performance, and estimate $\sigma_q^*\approx\hat{\sigma}_q(p_0 - \epsilon)$, where $p_0$ denotes the model's performance before \ours. 

\section{Theoretical Analysis}
\label{sec:theory}

In this section, we discuss 
our theoretical analysis, with complete proofs provided in Appendix~\ref{app:proof}.
First, we demonstrate that our \ours's Bayesianization scheme, defined in Equations~\ref{eq:svd}, \ref{eq:regroup}, and \ref{eq:bayesianization}, projects a full-rank isotropic Gaussian distribution onto the low-rank space.
We then prove that \Eqref{eq:tfb} is equivalent to generalized variational inference for LLMs' weights under specific, achievable conditions, offering solid theoretical grounding for \ours.


\begin{assumption}
\label{assumption:tfb}
The evaluation metric $l_\gD:\mathbb{R}_+ \rightarrow \mathbb{R}_+$ is the Negative Log Likelihood~(NLL) evaluated on the data distribution $\gD$ for the variational standard deviation $\sigma_q$:
\begin{align}
    l_\gD(\sigma_q) &= - \mathbb{E}_{(\vx, y) \sim \gD, \mW \sim q(\cdot|\sigma_q)}
    [\log P(y|\vx, \mW)].
\end{align}
Furthermore, we assume $l_\gD$ is locally convex, i.e., $\exists \epsilon_0 > 0$ s.t. $l_\gD^{\prime\prime}(\sigma_q)>0$, $\forall \sigma_q \in [0, \epsilon_0)$.
\end{assumption}


\begin{remark}
    The local convexity of the loss function is not unrealistic~\cite{NEURIPS2019_b33128cb}. For instance, a local minimum $\mW_0$ of a twice-differentiable loss function $l$ will imply the local convexity around $\mW_0$, as assumed in Laplace Approximation~\cite{tierney1986accurate,bishop2006pattern}. 
\end{remark}


\begin{theorem}[\textbf{Equivalent Variational Distribution of the Full Weight $\mW$ in \ours}]\label{thm:posterior}
    With the pre-trained weight matrix $\mW_0 \in\mathbb{R}^{m\times n}$, 
    the low-rank weight update matrix $\{\mB^\prime \in \mathbb{R}^{m\times r}, \mA^\prime \in \mathbb{R}^{r\times n}\}$ transformed from the given matrices $\{\mB, \mA\}$ following \Eqref{eq:svd} and \ref{eq:regroup}, 
    suppose that the variational distribution of $\mA^\prime$ is Gaussian $q(\mA^\prime|\vtheta)=\prod_{ij} \gN (A_{ij}|M_{ij},\Omega_{ij}^2)$, 
    where $\mM=[M_{ij}=A^\prime_{ij}] \in \mathbb{R}^{r\times n}$ is its mean and $\mOmega=[\Omega_{ij}] \in \mathbb{R}^{r\times n}$ is the  standard deviation calculated as in \Eqref{eq:bayesianization}. The equivalent variational distribution $q(\vectorize(\mW)|\sigma_q)$ defined on the full weight $\mW$ is
    \begin{equation}
    \begin{aligned}
        q(\vectorize(\mW)|\sigma_q) &= \gN(\vectorize(\mW)|\vmu_q, \mSigma_q), \\
        \text{where }\quad 
        \vmu_q &= \vectorize(\mW_0+\mB^\prime\mM), \\
        \mSigma_q &= \sigma_q^2  \mI_n \otimes 
        \begin{bmatrix}
            \mI_r & \\
            & \mathbf{0}_{m-r}
        \end{bmatrix}.
    \end{aligned}
    \label{eq:sigma_q}
    \end{equation}
\end{theorem}
\Thmref{thm:posterior} establishes that for any given $\sigma_q$, our algorithm for regrouping $\mB,\mA$ and computing the standard deviation matrix $\mOmega$ successfully constrains the corresponding full-weight variational distributions to the family of low-rank isotropic Gaussian distributions. 
This lays the foundation for the equivalence between our \ours and generalized variational inference to approximate the posterior distribution of LLM parameters (details in~\Thmref{thm:tfb}). 

While alternative families of Gaussian distributions parameterized by a single scale $\sigma_q$ are possible, our empirical results demonstrate that our approach achieves superior performance~(\Secref{sec:ablation}). 


\begin{theorem}[\textbf{\ours as Generalized Variational Inference}]\label{thm:tfb}
    Suppose the evaluation metric $l_\gD(\sigma_q)$ defined following \Assumptionref{assumption:tfb} is locally convex within the range of $\sigma_q\in[0,\epsilon_0)$. 
    Suppose the approximate distribution of $\mW$ given $\sigma_q$ is defined following \Thmref{thm:posterior}. 
    Suppose we have the prior distribution $P(\vectorize(\mW))=\gN(\vectorize(\mW)|\vmu_p, \mSigma_p)$, where $\vmu_p=\vmu_q=\vectorize(\mW_0+\mB^\prime\mM)$, and $\mSigma_p=\sigma_p^2\mI$ with $\sigma_p > \epsilon_0$.
    Then for $\forall \lambda>0$, $\exists \tilde{\epsilon}$, s.t. the following two optimization problems
    (i)~Generalized Variational Inference~\cite{blundell2015BBB,higgins2017beta,khan2018fast,knoblauch2019generalized}
        \begin{equation}
        \begin{aligned}
            \min_{\sigma_q} \quad l_\gD(\sigma_q) + \lambda \operatorname{KL}[q(\mW|\sigma_q) \parallel P(\mW)],
        \end{aligned}
        \end{equation}
    and 
    (ii)~Training-Free Bayesianization~(\ours)
        \begin{equation}
        \begin{aligned}
            \max \quad & \sigma_q \\
            \textit{s.t.} \quad & l_\gD(\sigma_q) \leq \tilde{\epsilon},
        \end{aligned}
        \end{equation}
    are equivalent, i.e., the two optimization problems have the same optimal solution,
    where $\lambda$ is the regularization coefficient of the KL-divergence. 
\end{theorem}
This theorem provides the primary theoretical foundation for \ours.
It demonstrates that under specific conditions -- namely, local convexity within $[0, \epsilon_0)$ and prior standard deviation $\sigma_p > \epsilon_0$ -- maximizing the scale $\sigma_q$ of the standard deviation matrix is equivalent to generalized variational inference~\cite{knoblauch2019generalized}, which approximates the posterior distribution of LLM parameters. 
Notably, when $\lambda=\nicefrac{1}{|\gD|}$ is set to the reciprocal of the dataset size, generalized variational inference reduces to variational inference.


\begin{remark}
    \ours maintains theoretical soundness (through its equivalence to variational optimization) while offering practical simplicity, as it eliminates the need to explicitly specify the prior distribution's standard deviation $\sigma_p$.
    The condition $\sigma_p > \epsilon_0$ is naturally satisfied by common choices such as the standard normal distribution ($\sigma_p=1$) or uniform distribution ($\sigma_p \rightarrow +\infty$).
\end{remark}


\section{Experiments}
\label{sec:experiments}

We evaluate \ours through comprehensive experiments.

\subsection{Settings}
\label{sec:experiments-setting}
\textbf{Models, Datasets, and Evaluation.}\quad 
We use the latest open-source \texttt{Meta-Llama-3.1-8B} as our primary LLM backbone while also providing additional results on other recent LLM architectures in \Secref{sec:experiments-other-llms}, including \texttt{llama-2-7b-hf}, \texttt{Meta-Llama-3-8B}, and \texttt{Mistral-7B-v0.3} from the Llama~\cite{dubey2024llama} and Mistral~\cite{jiang2023mistral} families.

\renewcommand{\thefootnote}{\fnsymbol{footnote}}
\begin{table*}[t]
\caption{
    \textbf{Performance of different methods applied to LoRA on Llama3.1-8B pre-trained weights,} where Accuracy~(\textbf{ACC}) and Expected Calibration Error~(\textbf{ECE}) are reported in percentages. 
    \textbf{``TF?''} denotes whether a method is \textbf{T}raining-\textbf{F}ree.
    The evaluation is done across six common-sense reasoning tasks with a shared hyper-parameter setting after fine-tuning of 5 epochs.
    We use $N=10$ samples during inference in all sampling-based methods including \textbf{\blob~\cite{wang2024blob}} and \textbf{\ours}. 
    \hl{Rows with shading} indicate training-free Bayesianization methods that use a pre-trained LoRA as their mean. 
    \textcolor{blue}{Cells highlighted in \textbf{BLUE}} indicate improved performance achieved by \ours compared to the weight mean.
    ``$\uparrow$'' and ``$\downarrow$'' indicate that higher and lower values are preferred, respectively. 
    \textbf{Boldface} and \underline{underlining} denote the best and the second-best performance, respectively. 
}
\vspace{-1em}
\begin{center}
\resizebox{1\linewidth}{!}{%
\setlength{\tabcolsep}{3pt}
\begin{tabular}{clcccc ccc cccc}
	\toprule[0.12em]
	\multirow{3}{*}[-0.25em]{\textbf{Metric}} 
    & \multirow{3}{*}[-0.25em]{\textbf{Method}} 
    & \multirow{3}{*}[-0.25em]{\textbf{TF?}}
    & \multicolumn{6}{c}{\multirow{2}{*}[-0.25em]{\textbf{In-Distribution Datasets}}} 
    & \multicolumn{4}{c}{\textbf{Out-of-Distribution Datasets} (OBQA$\rightarrow$X)}
     \\
     \cmidrule(lr){10-13}
     
     & & &
     & & & & & 
     & \multicolumn{2}{c}{\emph{Small Shift}}
     & \multicolumn{2}{c}{\emph{Large Shift}}
     \\
     
     \cmidrule(lr){4-9} \cmidrule(lr){10-11} \cmidrule(lr){12-13}
     & & & WG-S %
     & ARC-C%
     & ARC-E%
     & WG-M%
     & OBQA
     & BoolQ%
     & ARC-C%
     & ARC-E%
     & Chem
     & Phy
     \\
     \midrule

     \multirow{12}{*}{ACC~($\uparrow$)} 
     & \loramcd
     & \xmark 
     & \underline{78.03\scriptsize{$\pm$0.61}}  & 81.64\scriptsize{$\pm$1.79}  & 91.37\scriptsize{$\pm$0.38}  & 83.18\scriptsize{$\pm$0.84}  & 87.20\scriptsize{$\pm$1.02}  & 89.93\scriptsize{$\pm$0.16}  & 81.42\scriptsize{$\pm$1.38}  & 87.27\scriptsize{$\pm$0.84}  & 47.92\scriptsize{$\pm$2.25}  & \textbf{46.53\scriptsize{$\pm$0.49}}
     \\
     
     & \loraens
     & \xmark 
     & \textbf{78.82\scriptsize{$\pm$0.52}}  & 82.55\scriptsize{$\pm$0.42}  & \textbf{91.84\scriptsize{$\pm$0.36}}  & \textbf{83.99\scriptsize{$\pm$0.74}}  & 87.37\scriptsize{$\pm$0.67}  & \textbf{90.50\scriptsize{$\pm$0.14}}  & 79.62\scriptsize{$\pm$0.57}  & 86.56\scriptsize{$\pm$0.60}  & \underline{49.65\scriptsize{$\pm$3.22}}  & 44.44\scriptsize{$\pm$1.96}
      \\
\end{tabular}
 }
\end{center}
\label{tab:main-llama}
\vspace{-0.5em}
\end{table*}


For in-distribution experiments, we evaluate model performance on six commonsense reasoning tasks: Winogrande-Small~(\textbf{WG-S}) and Winogrande-Medium~(\textbf{WG-M})~\cite{wg}, ARC-Challenge~(\textbf{ARC-C}) and ARC-Easy~(\textbf{ARC-E})~\cite{arc}, Open Book Question Answering~(\textbf{OBQA})~\cite{obqa}, and BoolQ~\cite{boolq}. 
Furthermore, we use models fine-tuned on OBQA~\cite{obqa} to evaluate their generalization ability on out-of-distribution datasets: college-level chemistry~(\textbf{Chem}) and physics~(\textbf{Phy}) subsets of MMLU~\cite{mmlu}.
Label spaces and prompt templates are detailed in \appref{app:implementation-dataset}.

To assess uncertainty estimation, we measure Expected Calibration Error~(\textbf{ECE}~\cite{naeini2015obtaining}) and Negative Log-Likelihood~(\textbf{NLL}) on the test dataset. We also report Accuracy~(\textbf{ACC}) to ensure models maintain strong performance.
Additional evaluation details are provided in \appref{app:implementation-evaluation}.

\textbf{Baselines.}\quad 
We compare \ours with state-of-the-art uncertainty estimation methods for LoRA-adapted LLMs, including 
ensemble-based method: Deep Ensemble (\textbf{\loraens})~\cite{lakshminarayanan2017simple,balabanov2024uncertainty,wang2023lora}, 
variational inference methods: Monte-Carlo Dropout (\textbf{\loramcd})~\cite{gal2016dropout}, 
{Monte Carlo-enhanced LoRA~(\textbf{MonteCLoRA})~\cite{sengupta2024robust},}
Bayesian LoRA by Backprop~(\textbf{\blob})~\cite{wang2024blob},
and post-training method: Laplace-LoRA~(\textbf{\loralap})~\cite{yang2023bayesian}.
For reference, we also include two standard {Parameter-Efficient Fine-Tuning~(PEFT)} baselines: Maximum Likelihood Estimation~(\textbf{MLE})~\cite{hu2022lora} and Maximum A Posteriori~(\textbf{MAP}).
All baselines are implemented following the protocols established in \blob, detailed in \appref{app:training-detail}.

\textbf{\ours Implementation.}\quad 
\ours can be directly applied to trained LoRA adapters without additional training. {As indicated by the \textbf{``TF?''} column in \Tabref{tab:main-llama}, \ours is \textbf{T}raining-\textbf{F}ree and requires only LLM inference~(\textbf{\cmark}), while the other methods need full retraining~(\textbf{\xmark}) or gradient estimation with Backpropagation~(\textbf{BP}).}
We evaluate \ours on three off-the-shelf LoRA checkpoints: \textbf{\loramle}, \textbf{\loramap}, and the mean component of \textbf{\blob} (obtained by discarding \blob's standard deviation matrix $\mOmega$).
More details are included in \appref{app:training-detail}.

\subsection{\ours Improves Accuracy and Uncertainty Estimation across Distributional Shifts}
\label{sec:experiments-main}

\Tabref{tab:main-llama} shows results on comprehensive metrics for various methods applied to LoRA on Llama3.1-8B pre-trained weights. More empirical results on Llama2-7B can be found in \appref{app:more-llama2}.

\textbf{In-Distribution Results.}\quad 
{The addition of \ours maintains competitive accuracy while substantially improving model calibration across in-distribution datasets. 
For ECE, \ours yields notable improvements when applied to different base methods: MLE+\ours reduces ECE to 5.14\% on ARC-E (from 7.00\%); similarly
MAP+\ours and BLoB-Mean+\ours reduce ECE to 9.70\% on OBQA (from 12.19\%) and 3.83\% on WG-M (from 9.37\%), respectively. 
For NLL, \ours consistently produces better-calibrated predictions, with BLoB-Mean+\ours achieving strong performance across datasets: 0.23 on ARC-E (from 0.29), 0.33 on OBQA (from 0.37), and 0.27 on BoolQ (from 0.32). These improvements in both ECE and NLL demonstrate \ours's effectiveness in enhancing model calibration while preserving accuracy on in-distribution tasks.}

\textbf{Out-of-Distribution Results.}\quad
For out-of-distribution datasets, which represent a more challenging evaluation scenario, \ours continues to show benefits, though the performance gaps are generally smaller. In both Small Shift and Large Shift scenarios, \ours-enhanced methods maintain relatively strong performance, particularly in the Small Shift cases (ARC-C and ARC-E). However, there's a noticeable performance drop in the Large Shift scenarios (Chem and Phy), which is expected given the significant domain difference. Even in these challenging cases, \ours-enhanced methods tend to maintain better calibration (lower ECE scores) compared to their base counterparts, suggesting improved reliability in out-of-distribution settings.

\subsection{{Computational Efficiency of \ours}}
\label{app:more-efficiency}
{{We compare the computational efficiency of \ours and \blob during the process of Bayesianization~\cite{wang2024blob} in \Tabref{tab:mem_cost}.} We also report the computational cost of the standard LoRA fine-tuning as reference. 
All three methods are evaluated on the configurations detailed in \appref{app:training-detail}. 
For LoRA and \blob, the evaluation of running time and maximum GPU memory is based on fine-tuning for 5 epochs. 
\ours uses a fixed number of 500 training examples to search for $\sigma_q^*$ across all datasets, and performs binary search for at most 5 rounds~(sequentially).}

\begin{table*}[t]
\caption{
    {\textbf{A comparison of running time and maximum GPU memory cost between \ours and \blob during the process of Bayesianizatioin.} 
    The experiments are conducted on a single NVIDIA A100 GPU. 
    The subscripts in the table calculate the relative cost of a method compared to that of LoRA, a non-Bayesian baseline method. 
    \textcolor{red}{\textbf{RED}} and \textcolor{ForestGreen}{\textbf{GREEN}} represent \textcolor{red}{\textbf{worse}} and \textcolor{ForestGreen}{\textbf{better}} efficiency, respectivley.  {Note that varying batch sizes do not impact the performance of \ours, as \Algref{alg:tfb} is independent of gradient and batch size}}
}
\vspace{-0em}
\begin{center}
\resizebox{1\linewidth}{!}{%
\setlength{\tabcolsep}{2pt}
\begin{tabular}{cc rr rr rr rr rr rr}
	\toprule[0.12em]
	\multirow{3}{*}[-0.25em]{\textbf{Method}} & \multirow{3}{*}[-0.25em]{\makecell{\textbf{Batch}\\\textbf{Size}}} & \multicolumn{12}{c}{\textbf{Datasets}}
     \\
     \cmidrule{3-14}
     & &  \multicolumn{2}{c}{WG-S}
     & \multicolumn{2}{c}{ARC-C}
     & \multicolumn{2}{c}{ARC-E}
     & \multicolumn{2}{c}{WG-M}
     & \multicolumn{2}{c}{OBQA} 
     & \multicolumn{2}{c}{BoolQ} \\

     \cmidrule(lr){3-4}
     \cmidrule(lr){5-6}
     \cmidrule(lr){7-8}
     \cmidrule(lr){9-10}
     \cmidrule(lr){11-12}
     \cmidrule(lr){13-14}

     & & 
     Time (s) & Mem. (MB) & 
     Time (s) & Mem. (MB) & 
     Time (s) & Mem. (MB) & 
     Time (s) & Mem. (MB) & 
     Time (s) & Mem. (MB) & 
     Time (s) & Mem. (MB) \\
     
     \midrule

    LoRA & 4 & 
        338~\scriptsize{\color{white}(0.00x)} &
        12,894~\scriptsize{\color{white}(0.00x)} &
        632~\scriptsize{\color{white}(0.00x)} &
        19,762~\scriptsize{\color{white}(0.00x)} &
        1,238~\scriptsize{\color{white}(0.00x)} &
        18,640~\scriptsize{\color{white}(0.00x)} &
        1,339~\scriptsize{\color{white}(0.00x)} &
        13,164~\scriptsize{\color{white}(0.00x)} &
        2,692~\scriptsize{\color{white}(0.00x)} &
        17,208~\scriptsize{\color{white}(0.00x)} &
        6,489~\scriptsize{\color{white}(0.00x)} &
        29,450~\scriptsize{\color{white}(0.00x)} 
      \\

      \midrule

     \blob & 4 & 
     371~\scriptsize{\color{red}(1.10x)} &
     13,194~\scriptsize{\color{red}(1.02x)} & 
    685~\scriptsize{\color{red}(1.08x)} &
    21,736~\scriptsize{\color{red}(1.10x)} &
    1,360~\scriptsize{\color{red}(1.10x)} &
    20,700~\scriptsize{\color{red}(1.11x)} &
    1,476~\scriptsize{\color{red}(1.10x)} &
    13,194~\scriptsize{\color{red}(1.00x)}&
    3,257~\scriptsize{\color{red}(1.21x)} &
    18,046~\scriptsize{\color{red}(1.05x)} & 
    7,251~\scriptsize{\color{red}(1.12x)} & 
    30,578~\scriptsize{\color{red}(1.04x)}
     
      \\

     \ours~(Ours) & 4 &
     1,203~\scriptsize{\color{red}(3.56x)} & 
     10,372~\scriptsize{\color{ForestGreen}(0.80x)} &
     1,257~\scriptsize{\color{red}(1.99x)} &
     11,966~\scriptsize{\color{ForestGreen}(0.61x)} &
     1,246~\scriptsize{\color{red}(1.01x)} &
     11,202~\scriptsize{\color{ForestGreen}(0.60x)} &
     1,237~\scriptsize{\color{ForestGreen}(0.92x)} &
     10,344~\scriptsize{\color{ForestGreen}(0.79x)} &
     1,238~\scriptsize{\color{ForestGreen}(0.46x)} &
     10,376~\scriptsize{\color{ForestGreen}(0.60x)} &
     1,452~\scriptsize{\color{ForestGreen}(0.22x)} & 
     16,340~\scriptsize{\color{ForestGreen}(0.55x)}
     \\

     \ours~(Ours) & 8 &
     628~\scriptsize{\color{red}(1.86x)} &
     10,666~\scriptsize{\color{ForestGreen}(0.83x)} &
     731~\scriptsize{\color{red}(1.16x)} &
     15,286~\scriptsize{\color{ForestGreen}(0.77x)} &
     702~\scriptsize{\color{ForestGreen}(0.57x)} &
     12,598~\scriptsize{\color{ForestGreen}(0.68x)} &
     634~\scriptsize{\color{ForestGreen}(0.47x)} &
     10,662~\scriptsize{\color{ForestGreen}(0.81x)} &
     642~\scriptsize{\color{ForestGreen}(0.24x)} &
     12,116~\scriptsize{\color{ForestGreen}(0.70x)} &
     1,015~\scriptsize{\color{ForestGreen}(0.16x)} &
     22,146~\scriptsize{\color{ForestGreen}(0.75x)}
     \\

     \ours~(Ours) & 12 &
   446~\scriptsize{\color{red}(1.31x)} &
        12,064~\scriptsize{\color{ForestGreen}(0.93x)} &
   599~\scriptsize{\color{ForestGreen}(0.94x)}&
       18,204~\scriptsize{\color{ForestGreen}(0.92x)} &
   540~\scriptsize{\color{ForestGreen}(0.43x)}&
        14,310~\scriptsize{\color{ForestGreen}(0.76x)}&
   441~\scriptsize{\color{ForestGreen}(0.32x)}&
        11,370~\scriptsize{\color{ForestGreen}(0.86x)}&
   487~\scriptsize{\color{ForestGreen}(0.18x)}&
        13,410~\scriptsize{\color{ForestGreen}(0.77x)}&
        908~\scriptsize{\color{ForestGreen}(0.13x)} &
        25,220~\scriptsize{\color{ForestGreen}(0.85x)}
     \\

    \bottomrule[0.12em]
    \end{tabular}
}
\end{center}
\label{tab:mem_cost}
\vspace{0em}
\end{table*}

{As shown in the table, \ours can be slower on small datasets (e.g., WG-S with $\sim$600 samples, nearly 3× slower than \blob under the same batch size). However, on larger datasets (e.g., BoolQ with $\sim$10,000 samples), \textbf{\ours is up to 5× faster while using only half the GPU memory.} Since \ours avoids gradient estimation, memory use is substantially reduced, allowing larger batch sizes; increasing from 4 to 12 already yields lower time and memory than \blob on most datasets. Remarkably, \ours is even more efficient than standard LoRA fine-tuning, thanks to its training-free nature.
Importantly, \emph{our efficiency results cover the entire process of searching $\sigma$}, whereas baselines report only successful runs and \emph{exclude hyperparameter tuning costs,} biasing the comparison in their favor. Despite this, \ours still achieves superior time and memory efficiency, highlighting the advantages of its training-free approach and flexibility to trade off speed and memory under resource constraints.}

\subsection{\ours Beyond the Low-Rank Isotropic Gaussians}
\label{sec:ablation}
\begin{table*}[h]
\caption{
    \textbf{Performance of \ours with different variational distribution families applied to \blob-Mean on Llama3.1-8B pre-trained weights.} 
    \textbf{FR:} Full-rank isotropic Gaussian noises are applied to $\Delta\mW$; 
    \textbf{C-STD:} Standard deviation matrix $\mOmega=[\Omega_{ij}=\sigma_q]$ is constant. 
    The evaluation protocol strictly follows \Tabref{tab:main-llama}.
    \textbf{``Rk.''}: Average ranking of each method when compared to all other approaches on in-distribution datasets.
    ``$\uparrow$'' and ``$\downarrow$'' indicate that higher and lower values are preferred, respectively. 
    \textbf{Boldface} and \underline{underlining} denote the best and the second-best performance, respectively.
}
\vspace{-1.0em}
\begin{center}
\resizebox{1\linewidth}{!}{%
\setlength{\tabcolsep}{3pt}

\begin{tabular}{clcccccc c cccc}
	\toprule[0.12em]
	\multirow{3}{*}[-0.25em]{\textbf{Metric}} & \multirow{3}{*}[-0.25em]{\textbf{Method}} & \multicolumn{7}{c}{\multirow{2}{*}[-0.25em]{\textbf{In-Distribution Datasets}}} & \multicolumn{4}{c}{\textbf{Out-of-Distribution Datasets} (OBQA$\rightarrow$X)}
     \\
     \cmidrule(lr){10-13}
     
     & & 
     & & & & & &
     & \multicolumn{2}{c}{\emph{Small Shift}}
     & \multicolumn{2}{c}{\emph{Large Shift}}
     \\
     
     \cmidrule(lr){3-9} \cmidrule(lr){10-11} \cmidrule(lr){12-13}
     & & WG-S %
     & ARC-C%
     & ARC-E%
     & WG-M%
     & OBQA
     & BoolQ%
     & Rk.~($\downarrow$)
     & ARC-C%
     & ARC-E%
     & Chem
     & Phy
     \\
     \midrule

     \multirow{4}{*}{ACC~($\uparrow$)} 
     
     & \blob-Mean
      & \underline{77.72\scriptsize{$\pm$0.12}}  & 82.60\scriptsize{$\pm$0.60}  & \underline{91.64\scriptsize{$\pm$0.55}}  & \textbf{83.92\scriptsize{$\pm$0.48}}  & 88.00\scriptsize{$\pm$0.80}  & \underline{89.86\scriptsize{$\pm$0.05}}  & \underline{2.50}  & \underline{82.06\scriptsize{$\pm$1.15}}  & \textbf{88.54\scriptsize{$\pm$0.31}}  & 39.93\scriptsize{$\pm$5.20}  & \underline{39.93\scriptsize{$\pm$4.02}}
     \\

     & \cellcolor{lightergray} \quad + \ours~(FR) & \cellcolor{lightergray}75.57\scriptsize{$\pm$0.25}  & \cellcolor{lightergray}\underline{83.20\scriptsize{$\pm$0.65}}  & \cellcolor{lightergray}91.58\scriptsize{$\pm$0.67}  & \cellcolor{lightergray}82.19\scriptsize{$\pm$1.09}  & \cellcolor{lightergray}\textbf{88.73\scriptsize{$\pm$0.41}}  & \cellcolor{lightergray}89.46\scriptsize{$\pm$0.17} & \cellcolor{lightergray}2.83
     & \cellcolor{lightergray}81.33\scriptsize{$\pm$0.82}  & \cellcolor{lightergray}88.06\scriptsize{$\pm$0.75}  & \cellcolor{lightergray}\underline{42.00\scriptsize{$\pm$2.16}}  & \cellcolor{lightergray}\textbf{41.33\scriptsize{$\pm$5.44}} 
     \\

     & \cellcolor{lightergray} \quad + \ours~(C-STD)
	& \cellcolor{lightergray}76.35\scriptsize{$\pm$0.08}  & \cellcolor{lightergray}\underline{83.20\scriptsize{$\pm$0.33}}  & \cellcolor{lightergray}91.33\scriptsize{$\pm$0.70}  & \cellcolor{lightergray}81.79\scriptsize{$\pm$0.51}  & \cellcolor{lightergray}\underline{88.20\scriptsize{$\pm$0.57}}  & \cellcolor{lightergray}89.65\scriptsize{$\pm$0.08} & \cellcolor{lightergray}3.00 &\cellcolor{lightergray}81.73\scriptsize{$\pm$0.68}  & \cellcolor{lightergray}\underline{88.18\scriptsize{$\pm$0.65}}  & \cellcolor{lightergray}\textbf{43.00\scriptsize{$\pm$1.41}}  & \cellcolor{lightergray}39.33\scriptsize{$\pm$3.86} 
     \\
  
     & \cellcolor{lightergray} \quad + \ours~(Final)
	 & \cellcolor{lightergray}\textbf{77.81\scriptsize{$\pm$0.36}}  & \cellcolor{lightergray}\textbf{83.33\scriptsize{$\pm$0.19}}  & \cellcolor{lightergray}\textbf{91.76\scriptsize{$\pm$0.48}}  & \cellcolor{lightergray}\underline{83.81\scriptsize{$\pm$0.39}}  & \cellcolor{lightergray}87.80\scriptsize{$\pm$0.16}  & \cellcolor{lightergray}\textbf{90.11\scriptsize{$\pm$0.28}} & \cellcolor{lightergray}\textbf{1.67} & \cellcolor{lightergray}\textbf{82.93\scriptsize{$\pm$1.54}}  & \cellcolor{lightergray}{87.64\scriptsize{$\pm$0.51}}  & \cellcolor{lightergray}39.67\scriptsize{$\pm$7.32}  & \cellcolor{lightergray}37.33\scriptsize{$\pm$6.65} 
     \\

     \midrule
    
     \multirow{4}{*}{ECE~($\downarrow$)}

     & \blob-Mean
     & 15.43\scriptsize{$\pm$0.15}  & 12.41\scriptsize{$\pm$1.52}  & 4.91\scriptsize{$\pm$0.28}  & 9.37\scriptsize{$\pm$1.33}  & 6.44\scriptsize{$\pm$0.15}  & 6.26\scriptsize{$\pm$0.29}   & 4.00 & 11.22\scriptsize{$\pm$0.38}  & 6.34\scriptsize{$\pm$0.71}  & 26.65\scriptsize{$\pm$3.06}  & 25.40\scriptsize{$\pm$5.40} \\

     & \cellcolor{lightergray} \quad + \ours~(FR) & \cellcolor{lightergray}10.42\scriptsize{$\pm$0.29}  & \cellcolor{lightergray}7.45\scriptsize{$\pm$0.88}  & \cellcolor{lightergray}\textbf{2.01\scriptsize{$\pm$1.03}}  & \cellcolor{lightergray}4.36\scriptsize{$\pm$0.68}  & \cellcolor{lightergray}3.70\scriptsize{$\pm$1.04}  & \cellcolor{lightergray}3.62\scriptsize{$\pm$0.10} & \cellcolor{lightergray}2.67 & \cellcolor{lightergray}7.19\scriptsize{$\pm$1.40}  & \cellcolor{lightergray}\underline{3.29\scriptsize{$\pm$1.03}}  & \cellcolor{lightergray}\textbf{17.78\scriptsize{$\pm$1.01}}  & \cellcolor{lightergray}\textbf{19.14\scriptsize{$\pm$4.01}} 
     \\

     & \cellcolor{lightergray} \quad + \ours~(C-STD)
	 & \cellcolor{lightergray}\underline{9.23\scriptsize{$\pm$0.20}}  & \cellcolor{lightergray}\textbf{5.98\scriptsize{$\pm$0.32}}  & \cellcolor{lightergray}2.94\scriptsize{$\pm$0.67}  & \cellcolor{lightergray}\underline{3.86\scriptsize{$\pm$0.45}}  & \cellcolor{lightergray}\underline{3.17\scriptsize{$\pm$0.21}}  & \cellcolor{lightergray}\textbf{2.82\scriptsize{$\pm$0.62}} & \cellcolor{lightergray}\underline{1.83} & \cellcolor{lightergray}\underline{6.89\scriptsize{$\pm$0.89}}  & \cellcolor{lightergray}\textbf{2.76\scriptsize{$\pm$0.88}}  & \cellcolor{lightergray}\underline{18.27\scriptsize{$\pm$2.52}}  & \cellcolor{lightergray}\underline{19.45\scriptsize{$\pm$3.46}}
     \\

     & \cellcolor{lightergray} \quad + \ours~(Final)
	 & \cellcolor{lightergray}\textbf{8.16\scriptsize{$\pm$0.48}}  & \cellcolor{lightergray}\underline{6.48\scriptsize{$\pm$0.36}}  & \cellcolor{lightergray}\underline{2.44\scriptsize{$\pm$0.50}}  & \cellcolor{lightergray}\textbf{3.83\scriptsize{$\pm$0.43}}  & \cellcolor{lightergray}\textbf{2.67\scriptsize{$\pm$0.18}}  & \cellcolor{lightergray}\underline{3.10\scriptsize{$\pm$0.59}} & \cellcolor{lightergray}\textbf{1.50} & \cellcolor{lightergray}\textbf{6.69\scriptsize{$\pm$1.63}}  & \cellcolor{lightergray}{3.61\scriptsize{$\pm$0.87}}  & \cellcolor{lightergray}18.45\scriptsize{$\pm$6.75}  & \cellcolor{lightergray}20.53\scriptsize{$\pm$6.27} 
     \\

     \midrule
    
     \multirow{4}{*}{NLL~($\downarrow$)} 

     & \blob-Mean
    & 0.74\scriptsize{$\pm$0.02}  & 0.73\scriptsize{$\pm$0.04}  & 0.29\scriptsize{$\pm$0.03}  & 0.47\scriptsize{$\pm$0.03}  & \underline{0.37\scriptsize{$\pm$0.02}}  & 0.32\scriptsize{$\pm$0.02}   & 3.67 & 0.67\scriptsize{$\pm$0.07}  & 0.39\scriptsize{$\pm$0.03}  & 1.53\scriptsize{$\pm$0.13}  & 1.54\scriptsize{$\pm$0.15} 
     \\

     & \cellcolor{lightergray} \quad + \ours~(FR) & \cellcolor{lightergray}0.60\scriptsize{$\pm$0.01}  & \cellcolor{lightergray}\underline{0.53\scriptsize{$\pm$0.03}}  & \cellcolor{lightergray}\underline{0.23\scriptsize{$\pm$0.02}}  & \cellcolor{lightergray}\underline{0.43\scriptsize{$\pm$0.01}}  & \cellcolor{lightergray}\textbf{0.33\scriptsize{$\pm$0.02}}  & \cellcolor{lightergray}\underline{0.27\scriptsize{$\pm$0.01}}  & \cellcolor{lightergray}2.00 & \cellcolor{lightergray}0.57\scriptsize{$\pm$0.04}  & \cellcolor{lightergray}\underline{0.34\scriptsize{$\pm$0.02}}  & \cellcolor{lightergray}\textbf{1.34\scriptsize{$\pm$0.07}}  & \cellcolor{lightergray}\underline{1.42\scriptsize{$\pm$0.09}} 
     \\

     & \cellcolor{lightergray} \quad + \ours~(C-STD)
	 & \cellcolor{lightergray}\underline{0.57}\scriptsize{$\pm$0.01}  & \cellcolor{lightergray}\textbf{0.51\scriptsize{$\pm$0.02}}  & \cellcolor{lightergray}\textbf{0.22\scriptsize{$\pm$0.01}}  & \cellcolor{lightergray}\underline{0.43\scriptsize{$\pm$0.01}}  & \cellcolor{lightergray}\textbf{0.33\scriptsize{$\pm$0.01}}  & \cellcolor{lightergray}\textbf{0.26\scriptsize{$\pm$0.01}}  & \cellcolor{lightergray}\textbf{1.33} & \cellcolor{lightergray}\underline{0.56\scriptsize{$\pm$0.04}}  & \cellcolor{lightergray}\textbf{0.33\scriptsize{$\pm$0.02}}  & \cellcolor{lightergray}\textbf{1.34\scriptsize{$\pm$0.08}}  & \cellcolor{lightergray}\textbf{1.41\scriptsize{$\pm$0.09}} 
     \\

     & \cellcolor{lightergray} \quad + \ours~(Final)
	& \cellcolor{lightergray}\textbf{0.55\scriptsize{$\pm$0.01}}  & \cellcolor{lightergray}\underline{0.53\scriptsize{$\pm$0.04}}  & \cellcolor{lightergray}\underline{0.23\scriptsize{$\pm$0.02}}  & \cellcolor{lightergray}\textbf{0.40\scriptsize{$\pm$0.01}}  & \cellcolor{lightergray}\textbf{0.33\scriptsize{$\pm$0.02}}  & \cellcolor{lightergray}\underline{0.27\scriptsize{$\pm$0.01}} & \cellcolor{lightergray}\underline{1.50} & \cellcolor{lightergray}\textbf{0.52\scriptsize{$\pm$0.05}}  & \cellcolor{lightergray}{0.35\scriptsize{$\pm$0.02}}  & \cellcolor{lightergray}\underline{1.36\scriptsize{$\pm$0.13}}  & \cellcolor{lightergray}1.46\scriptsize{$\pm$0.11} 
     \\

    \bottomrule[0.12em]
    \end{tabular}
 }
\end{center}
\label{tab:main-ablation}
\vspace{0em}
\end{table*}

In this section, we consider two simple \ours variants with other families of Gaussians for modeling the variational distributions of $\mW$:
(i)~{Full-Rank Isotropic Gaussian~(\textbf{FR}, $\mSigma_q = \sigma_q^2\mI$}), and (ii)~{Constant Low-Rank Standard Deviation~(\textbf{C-STD}, $\mOmega=[\Omega_{ij}=\sigma_q]$}). 
{Similar to our final \ours, both distributions are controlled by a single $\sigma_q$ parameter and fit the maximal variance search in \Eqref{eq:tfb}.}
For fair comparison, we adopt the same optimal $\sigma_q^*$ search protocol as described in \Secref{sec:experiments-setting}. 
\Tabref{tab:main-ablation} shows the performances of \ours and its variants applied to the mean of \blob (more in \Tabref{tab:main-ablation-cls} of \appref{app:more-ablation-full}). 


\begin{wrapfigure}{R}{0.4\textwidth}
\noindent\begin{minipage}{\linewidth}
\vspace{-1.5em}
\begin{table}[H]
\centering
    \caption{
        Performance of different \textbf{LLM backbones} on the combined dataset. 
    }
    \begin{center}
    \resizebox{1\linewidth}{!}{%
        \setlength{\tabcolsep}{8pt}
        \begin{tabular}{lccc}
        \toprule[0.12em]
        \textbf{Method}
        & \textbf{ACC}~($\uparrow$) 
        & \textbf{ECE}~($\downarrow$) 
        & \textbf{NLL}~($\downarrow$) \\

        \midrule
        
        Llama2-7B 
        & \textbf{81.41\scriptsize{$\pm$0.64} }
        & 4.50\scriptsize{$\pm$0.37} 
        & \textbf{0.43\scriptsize{$\pm$0.00}} \\
        
        \cellcolor{lightergray} \quad + \ours~(Ours) 
        & \cellcolor{lightergray}81.32\scriptsize{$\pm$0.51} 
        & \textbf{\cellcolor{lightergray}1.24\scriptsize{$\pm$0.22}} 
        & \textbf{\cellcolor{lightergray}0.43\scriptsize{$\pm$0.00}} \\
        
        \midrule
        
        Llama3-8B 
        & \textbf{86.93\scriptsize{$\pm$0.09}} 
        & 4.28\scriptsize{$\pm$0.54} 
        & \textbf{0.34\scriptsize{$\pm$0.00}} 
        \\
        
        \cellcolor{lightergray} \quad + \ours~(Ours) 
        & \cellcolor{lightergray}86.61\scriptsize{$\pm$0.20} 
        & \textbf{\cellcolor{lightergray}1.64\scriptsize{$\pm$0.64} }
        & \textbf{\cellcolor{lightergray}0.34\scriptsize{$\pm$0.00} }
        \\
        
        \midrule
        
        Llama3.1-8B 
        & \textbf{86.70\scriptsize{$\pm$0.08}} 
        & 4.74\scriptsize{$\pm$0.28} 
        & 0.35\scriptsize{$\pm$0.00} \\
        
        \cellcolor{lightergray} \quad + \ours~(Ours) 
        & \cellcolor{lightergray}86.45\scriptsize{$\pm$0.33} 
        & \textbf{\cellcolor{lightergray}1.05\scriptsize{$\pm$0.06}} 
        & \textbf{\cellcolor{lightergray}0.34\scriptsize{$\pm$0.00}} 
        \\
        
        \midrule
        
        Mistral-7B-v0.3 
        & \textbf{86.88\scriptsize{$\pm$0.51}} 
        & 5.05\scriptsize{$\pm$0.88} 
        & 0.35\scriptsize{$\pm$0.02} 
        \\
        
        \cellcolor{lightergray} \quad + \ours~(Ours) 
        & \cellcolor{lightergray}86.64\scriptsize{$\pm$0.28} 
        & \textbf{\cellcolor{lightergray}1.68\scriptsize{$\pm$0.53}} 
        & \textbf{\cellcolor{lightergray}0.33\scriptsize{$\pm$0.01}} 
        \\
        
        \bottomrule[0.12em]
        \end{tabular}
    }
    \end{center}
    \vspace{-.5em}
    \label{tab:llama-others}
\end{table}
\end{minipage}
\end{wrapfigure}
These results show that our final \ours outperforms both variants \textbf{FR} and \textbf{C-STD} across multiple metrics on in-distribution datasets, with notable improvements in calibration (ECE reduced by up to 15.77\%) and accuracy (e.g., 77.81\% on WG-S). 
While \textbf{C-STD} shows better NLL scores, the improvement comes at the cost of a significantly degraded overall performance---particularly in accuracy, where it performs the worst---making it impractical for real-world applications.
Although our final \ours maintains strong performance on datasets with smaller distributional shifts, its advantages diminish on datasets with larger shifts in the domains of Physics and Chemistry.

\textbf{Advantages of Final \ours's Variational Low-Rank Isotropic Gaussians.}\quad
Compared to \emph{\ours (FR)} and \emph{\ours (C-STD)}, 
\emph{\ours (Final)} offers additional advantages. It is computationally more efficient than \textbf{FR} with noise complexity of $O(rn)$ versus $O(mn)$. Furthermore, unlike \textbf{C-STD} whose variational distributions vary with different but equivalent LoRA matrix pairs~(see \appref{app:more-ablation-full} for details), \emph{\ours (Final)} produces consistent Bayesianization for all equivalent LoRAs satisfying $\mB\mA=\Delta \mW$.

\subsection{\ours Beyond the Llama3.1-8B Backbone}
\label{sec:experiments-other-llms}
We conduct comprehensive experiments across multiple LLM backbones to validate our approach. Our experiments span several models from the Llama family~\cite{touvron2023llama2,dubey2024llama}, including \texttt{llama-2-7b-hf}, \texttt{Meta-Llama-3-8B}, and \texttt{Meta-Llama-3.1-8B}. While we initially considered \texttt{Llama-3.2-1B}, we ultimately discarded its results due to poor adaptation performance with the smaller model architecture. We also extend our analysis to include \texttt{Mistral-7B-v0.3}~\cite{jiang2023mistral}.

Following commonsense-170k~\cite{hu2023llm,wang2024milora}, we combine the 6 reasoning sub-tasks from the main experiments with one shared label space~(\textbf{Combined}) and train the base LoRA adapters with MLE. Break-down statistics of each sub-dataset are available in \appref{app:more-other-llms}. 
\Tabref{tab:llama-others} shows the results, demonstrating \ours's effectiveness: it dramatically reduces ECE across all models (e.g., from 4.74\% to 1.05\% for Llama3.1-8B) while maintaining strong ACC and NLL scores.

\subsection{\ours Beyond the Naive LoRA}
\label{sec:experiments-other-pefts}
\begin{wrapfigure}{R}{0.46\textwidth}
\noindent\begin{minipage}{\linewidth}
\vspace{-1.5em}
\begin{table}[H]
\caption{
    Performance of different \textbf{LoRA-like PEFTs} on the combined dataset.
}
\begin{center}
\resizebox{\linewidth}{!}{%
\setlength{\tabcolsep}{6pt}
\begin{tabular}{lccc}
    \toprule[0.12em]
    \textbf{Method} 
    & \textbf{ACC~($\uparrow$)} 
    & \textbf{ECE~($\downarrow$)} 
    & \textbf{NLL~($\downarrow$)} 
    \\
    
    \midrule
    
    LoRA 
    & \textbf{86.70\scriptsize{$\pm$0.08} }
    & 4.74\scriptsize{$\pm$0.28} 
    & 0.35\scriptsize{$\pm$0.00} 
    \\
    
    \cellcolor{lightergray} \quad + \ours~(Ours) 
    & \cellcolor{lightergray}86.45\scriptsize{$\pm$0.33} 
    & \textbf{\cellcolor{lightergray}1.05\scriptsize{$\pm$0.06}} 
    & \textbf{\cellcolor{lightergray}0.34\scriptsize{$\pm$0.00}} 
    \\
    
    \midrule
    
    
    VeRA 
    & \textbf{84.93\scriptsize{$\pm$0.50}} 
    & 5.11\scriptsize{$\pm$0.55} 
    & 0.39\scriptsize{$\pm$0.01} 
    \\
    
    \cellcolor{lightergray} \quad + \ours~(Ours) 
    & \cellcolor{lightergray}84.28\scriptsize{$\pm$0.48} 
    & \textbf{\cellcolor{lightergray}1.44\scriptsize{$\pm$0.44}} 
    & \textbf{\cellcolor{lightergray}0.38\scriptsize{$\pm$0.01}} 
    \\
    
    \midrule
    
    PiSSA 
    & \textbf{86.83\scriptsize{$\pm$0.51}} 
    & 4.26\scriptsize{$\pm$0.14} 
    & 0.35\scriptsize{$\pm$0.00} 
    \\
    
    \cellcolor{lightergray} \quad + \ours~(Ours) 
    & \cellcolor{lightergray}86.61\scriptsize{$\pm$0.43} 
    & \textbf{\cellcolor{lightergray}1.17\scriptsize{$\pm$0.22}} 
    & \textbf{\cellcolor{lightergray}0.33\scriptsize{$\pm$0.00}} 
    \\
    
    \bottomrule[0.12em]
\end{tabular}
}
\end{center}
\vspace{-4em}
\label{tab:lora-others}
\end{table}
\end{minipage}
\end{wrapfigure}
Our proposed \ours is general, compatible with various LoRA-based methods that use different initialization strategies~\cite{meng2024pissa}, parameter sharing schemes~\cite{kopiczko2023vera}, and optimization approaches~\cite{zhang2023adalora}. 
We evaluate \ours on two representative LoRA variants~(see \appref{app:more-other-pefts} for details):
\begin{itemize}[nosep,leftmargin=18pt]
    \item VeRA~\cite{kopiczko2023vera}: Uses shared low-rank matrices $\mB$ and $\mA$ across layers, with the layer-specific trainable scalar vector $\vd$ and the bias vector $\vb$.
    \item PiSSA~\cite{meng2024pissa}: Employs an alternative initialization while maintaining LoRA's training process.
\end{itemize}

\Tabref{tab:lora-others} shows the results, demonstrating \ours's broad applicability: it substantially reduces ECE across all LoRA-like PEFT methods (e.g., from 5.11 to 1.44 for VeRA). Importantly, it maintains strong ACC and NLL with minimal performance degradation, validating the effectiveness of low-rank isotropic Gaussian distributions for variational inference of LLMs.

\subsection{Improving Inference-Time Efficiency of \ours}
\label{sec:experiments-ll-tfb}
\begin{wrapfigure}{R}{0.6\textwidth}
\noindent\begin{minipage}{\linewidth}
\vspace{-1.8em}
\begin{table}[H]
\caption{
    Performance of \textbf{Last-Layer \ours (\texttt{LL}~\ours)} applied to the combined dataset. 
}
\begin{center}
\resizebox{\linewidth}{!}{%
\setlength{\tabcolsep}{4pt}
\begin{tabular}{lcccccc}
    \toprule[0.12em]
    \textbf{Method} 
    & \textbf{\#Sample~(N)} 
    & \textbf{ACC~($\uparrow$)} 
    & \textbf{ECE~($\downarrow$)} 
    & \textbf{NLL~($\downarrow$)} 
    & \textbf{\makecell{Time~(s)}} 
    
    \\
    \midrule
    MLE & - & 86.70\scriptsize{$\pm$0.08} & 4.74\scriptsize{$\pm$0.28} & 0.35\scriptsize{$\pm$0.00} & 118 \\
    \cellcolor{lightergray} \quad + \ours
    & \cellcolor{lightergray}10 
    & \cellcolor{lightergray}86.45\scriptsize{$\pm$0.33} 
    & \cellcolor{lightergray}1.05\scriptsize{$\pm$0.06} 
    & \cellcolor{lightergray}0.34\scriptsize{$\pm$0.00} 
    & \cellcolor{lightergray}1,114 \\

    \cellcolor{lightergray} \quad + \texttt{LL}~\ours 
    & \cellcolor{lightergray}10 
    & \cellcolor{lightergray}86.53\scriptsize{$\pm$0.12}
    & \cellcolor{lightergray}1.14\scriptsize{$\pm$0.03} 
    & \cellcolor{lightergray}0.34\scriptsize{$\pm$0.00} 
    & \cellcolor{lightergray}182 \\
    
    \cellcolor{lightergray} \quad + \texttt{LL}~\ours 
    & \cellcolor{lightergray}100 
    & \cellcolor{lightergray}\textbf{86.75\scriptsize{$\pm$0.09}} 
    & \cellcolor{lightergray}\textbf{0.92\scriptsize{$\pm$0.10}} 
    & \cellcolor{lightergray}\textbf{0.33\scriptsize{$\pm$0.00} }
    & \cellcolor{lightergray}924 \\
    
    \bottomrule[0.12em]
\end{tabular}
}
\end{center}
\vspace{-0em}
\label{tab:tfb-ll-main}
\end{table}
\end{minipage}
\end{wrapfigure}
Inspired by \citet{harrisonvariational}, in this section, we study the compactibility of \ours and Last-Layer Bayesianization to speed up inference. Specifically, we employ the same standard deviation $\sigma_q^*$ as identified through the full-model \ours, but apply the Bayesianization to the last layer only. By limiting inference to sampling from the variational distribution of the last layer, the outputs of earlier layers can be reused, significantly improving inference speed. Although restricting Bayesianization to the last layer sacrifices some richness of the variational distribution family, this approach allows for a higher number of posterior samples, resulting in more accurate approximations.

\Tabref{tab:tfb-ll-main} shows the results. When only $N=10$ samples are used, the uncertainty estimation of last-layer Bayesianization performs worse than that of full-model Bayesianization. Nevertheless, the improved posterior estimation with $N=100$ samples enables last-layer Bayesianization to achieve better performance.
For a comprehensive analysis of its performance, please refer to \appref{app:tfb-ll}.



\section{Conclusion}
\label{sec:conclusion}
In this paper, we introduce Training-Free Bayesianization~(\ours), a novel framework that transforms trained LoRA adapters into Bayesian ones without additional training. By systematically searching for the maximally acceptable variance in the weight posterior within a family of low-rank isotropic Gaussian distributions, \ours provides a practical solution to uncertainty estimation in LLMs. 
Our theoretical analysis shows that \ours's variance maximization process is equivalent to generalized variational inference under mild conditions. Our empirical results verify its superior performance across various settings and model configurations. Our framework's simplicity and effectiveness, requiring only an anchor dataset for search, makes it widely applicable across different domains. As LLMs continue to evolve, \ours represents a significant step toward more reliable and uncertainty-aware AI systems, paving the way for future research in adaptive and trustworthy machine learning.
{For \textbf{Limitations}, please refer to \appref{sec:limitations}.}



\section*{Acknowledgement}
\label{sec:acknowledgement}
We thank all reviewers, AC, and SAC for their valuable comments. HW is supported by Amazon Faculty Research Award, Microsoft AI \& Society Fellowship, NSF CAREER Award IIS-2340125, NIH grant R01CA297832, and NSF grant IIS-2127918.

\bibliography{ref}

\begin{thebibliography}{70}
\providecommand{\natexlab}[1]{#1}
\providecommand{\url}[1]{\texttt{#1}}
\expandafter\ifx\csname urlstyle\endcsname\relax
  \providecommand{\doi}[1]{doi: #1}\else
  \providecommand{\doi}{doi: \begingroup \urlstyle{rm}\Url}\fi

\bibitem[Abdar et~al.(2021)Abdar, Pourpanah, Hussain, Rezazadegan, Liu, Ghavamzadeh, Fieguth, Cao, Khosravi, Acharya, et~al.]{abdar2021review}
Abdar, M., Pourpanah, F., Hussain, S., Rezazadegan, D., Liu, L., Ghavamzadeh, M., Fieguth, P., Cao, X., Khosravi, A., Acharya, U.~R., et~al.
\newblock A review of uncertainty quantification in deep learning: Techniques, applications and challenges.
\newblock \emph{Information fusion}, 76:\penalty0 243--297, 2021.

\bibitem[Achiam et~al.(2023)Achiam, Adler, Agarwal, Ahmad, Akkaya, Aleman, Almeida, Altenschmidt, Altman, Anadkat, et~al.]{achiam2023gpt}
Achiam, J., Adler, S., Agarwal, S., Ahmad, L., Akkaya, I., Aleman, F.~L., Almeida, D., Altenschmidt, J., Altman, S., Anadkat, S., et~al.
\newblock Gpt-4 technical report.
\newblock \emph{arXiv preprint arXiv:2303.08774}, 2023.

\bibitem[Anil et~al.(2023)Anil, Dai, Firat, Johnson, Lepikhin, Passos, Shakeri, Taropa, Bailey, Chen, et~al.]{anil2023palm}
Anil, R., Dai, A.~M., Firat, O., Johnson, M., Lepikhin, D., Passos, A., Shakeri, S., Taropa, E., Bailey, P., Chen, Z., et~al.
\newblock Palm 2 technical report.
\newblock \emph{arXiv preprint arXiv:2305.10403}, 2023.

\bibitem[Balabanov \& Linander(2024)Balabanov and Linander]{balabanov2024uncertainty}
Balabanov, O. and Linander, H.
\newblock Uncertainty quantification in fine-tuned llms using lora ensembles.
\newblock \emph{arXiv preprint arXiv:2402.12264}, 2024.

\bibitem[Biderman et~al.(2023)Biderman, Schoelkopf, Anthony, Bradley, O’Brien, Hallahan, Khan, Purohit, Prashanth, Raff, et~al.]{biderman2023pythia}
Biderman, S., Schoelkopf, H., Anthony, Q.~G., Bradley, H., O’Brien, K., Hallahan, E., Khan, M.~A., Purohit, S., Prashanth, U.~S., Raff, E., et~al.
\newblock Pythia: A suite for analyzing large language models across training and scaling.
\newblock In \emph{International Conference on Machine Learning}, pp.\  2397--2430. PMLR, 2023.

\bibitem[Bishop(2006)]{bishop2006pattern}
Bishop, C.~M.
\newblock Pattern recognition and machine learning.
\newblock \emph{Springer google schola}, 2:\penalty0 1122--1128, 2006.

\bibitem[Blundell et~al.(2015)Blundell, Cornebise, Kavukcuoglu, and Wierstra]{blundell2015BBB}
Blundell, C., Cornebise, J., Kavukcuoglu, K., and Wierstra, D.
\newblock Weight uncertainty in neural network.
\newblock In \emph{International conference on machine learning}, pp.\  1613--1622. PMLR, 2015.

\bibitem[Brown et~al.(2020)Brown, Mann, Ryder, Subbiah, Kaplan, Dhariwal, Neelakantan, Shyam, Sastry, Askell, et~al.]{brown2020language}
Brown, T., Mann, B., Ryder, N., Subbiah, M., Kaplan, J.~D., Dhariwal, P., Neelakantan, A., Shyam, P., Sastry, G., Askell, A., et~al.
\newblock Language models are few-shot learners.
\newblock \emph{Advances in neural information processing systems}, 33:\penalty0 1877--1901, 2020.

\bibitem[Chowdhery et~al.(2023)Chowdhery, Narang, Devlin, Bosma, Mishra, Roberts, Barham, Chung, Sutton, Gehrmann, et~al.]{chowdhery2023palm}
Chowdhery, A., Narang, S., Devlin, J., Bosma, M., Mishra, G., Roberts, A., Barham, P., Chung, H.~W., Sutton, C., Gehrmann, S., et~al.
\newblock Palm: Scaling language modeling with pathways.
\newblock \emph{Journal of Machine Learning Research}, 24\penalty0 (240):\penalty0 1--113, 2023.

\bibitem[Clark et~al.(2019)Clark, Lee, Chang, Kwiatkowski, Collins, and Toutanova]{boolq}
Clark, C., Lee, K., Chang, M.-W., Kwiatkowski, T., Collins, M., and Toutanova, K.
\newblock {B}ool{Q}: Exploring the surprising difficulty of natural yes/no questions.
\newblock In Burstein, J., Doran, C., and Solorio, T. (eds.), \emph{Proceedings of the 2019 Conference of the North {A}merican Chapter of the Association for Computational Linguistics: Human Language Technologies, Volume 1 (Long and Short Papers)}, pp.\  2924--2936, Minneapolis, Minnesota, June 2019. Association for Computational Linguistics.
\newblock \doi{10.18653/v1/N19-1300}.
\newblock URL \url{https://aclanthology.org/N19-1300}.

\bibitem[Clark et~al.(2018)Clark, Cowhey, Etzioni, Khot, Sabharwal, Schoenick, and Tafjord]{arc}
Clark, P., Cowhey, I., Etzioni, O., Khot, T., Sabharwal, A., Schoenick, C., and Tafjord, O.
\newblock Think you have solved question answering? try arc, the ai2 reasoning challenge, 2018.

\bibitem[Dubey et~al.(2024)Dubey, Jauhri, Pandey, Kadian, Al-Dahle, Letman, Mathur, Schelten, Yang, Fan, et~al.]{dubey2024llama}
Dubey, A., Jauhri, A., Pandey, A., Kadian, A., Al-Dahle, A., Letman, A., Mathur, A., Schelten, A., Yang, A., Fan, A., et~al.
\newblock The llama 3 herd of models.
\newblock \emph{arXiv preprint arXiv:2407.21783}, 2024.

\bibitem[Fan et~al.(2020)Fan, Zhang, Chen, and Zhou]{fan2020bayesian}
Fan, X., Zhang, S., Chen, B., and Zhou, M.
\newblock Bayesian attention modules.
\newblock \emph{Advances in Neural Information Processing Systems}, 33:\penalty0 16362--16376, 2020.

\bibitem[Gal \& Ghahramani(2016)Gal and Ghahramani]{gal2016dropout}
Gal, Y. and Ghahramani, Z.
\newblock Dropout as a bayesian approximation: Representing model uncertainty in deep learning.
\newblock In \emph{international conference on machine learning}, pp.\  1050--1059. PMLR, 2016.

\bibitem[Gawlikowski et~al.(2023)Gawlikowski, Tassi, Ali, Lee, Humt, Feng, Kruspe, Triebel, Jung, Roscher, et~al.]{gawlikowski2023survey}
Gawlikowski, J., Tassi, C. R.~N., Ali, M., Lee, J., Humt, M., Feng, J., Kruspe, A., Triebel, R., Jung, P., Roscher, R., et~al.
\newblock A survey of uncertainty in deep neural networks.
\newblock \emph{Artificial Intelligence Review}, 56\penalty0 (Suppl 1):\penalty0 1513--1589, 2023.

\bibitem[Guo et~al.(2017)Guo, Pleiss, Sun, and Weinberger]{guo2017calibration}
Guo, C., Pleiss, G., Sun, Y., and Weinberger, K.~Q.
\newblock On calibration of modern neural networks.
\newblock In \emph{International conference on machine learning}, pp.\  1321--1330. PMLR, 2017.

\bibitem[Gupta et~al.(2024)Gupta, Narasimhan, Jitkrittum, Rawat, Menon, and Kumar]{gupta2024language}
Gupta, N., Narasimhan, H., Jitkrittum, W., Rawat, A.~S., Menon, A.~K., and Kumar, S.
\newblock Language model cascades: Token-level uncertainty and beyond.
\newblock \emph{arXiv preprint arXiv:2404.10136}, 2024.

\bibitem[Gustafsson et~al.(2020)Gustafsson, Danelljan, and Schon]{gustafsson2020evaluating}
Gustafsson, F.~K., Danelljan, M., and Schon, T.~B.
\newblock Evaluating scalable bayesian deep learning methods for robust computer vision.
\newblock In \emph{Proceedings of the IEEE/CVF conference on computer vision and pattern recognition workshops}, pp.\  318--319, 2020.

\bibitem[Harrison et~al.(2024)Harrison, Willes, and Snoek]{harrisonvariational}
Harrison, J., Willes, J., and Snoek, J.
\newblock Variational bayesian last layers.
\newblock In \emph{The Twelfth International Conference on Learning Representations}, 2024.

\bibitem[Hendrycks et~al.(2021)Hendrycks, Burns, Basart, Zou, Mazeika, Song, and Steinhardt]{mmlu}
Hendrycks, D., Burns, C., Basart, S., Zou, A., Mazeika, M., Song, D., and Steinhardt, J.
\newblock Measuring massive multitask language understanding.
\newblock \emph{Proceedings of the International Conference on Learning Representations (ICLR)}, 2021.

\bibitem[Hern{\'a}ndez-Lobato \& Adams(2015)Hern{\'a}ndez-Lobato and Adams]{hernandez2015probabilistic}
Hern{\'a}ndez-Lobato, J.~M. and Adams, R.
\newblock Probabilistic backpropagation for scalable learning of bayesian neural networks.
\newblock In \emph{International conference on machine learning}, pp.\  1861--1869. PMLR, 2015.

\bibitem[Higgins et~al.(2017)Higgins, Matthey, Pal, Burgess, Glorot, Botvinick, Mohamed, and Lerchner]{higgins2017beta}
Higgins, I., Matthey, L., Pal, A., Burgess, C.~P., Glorot, X., Botvinick, M.~M., Mohamed, S., and Lerchner, A.
\newblock beta-vae: Learning basic visual concepts with a constrained variational framework.
\newblock \emph{ICLR (Poster)}, 3, 2017.

\bibitem[Hu et~al.(2022)Hu, Shen, Wallis, Allen-Zhu, Li, Wang, Wang, and Chen]{hu2022lora}
Hu, E.~J., Shen, Y., Wallis, P., Allen-Zhu, Z., Li, Y., Wang, S., Wang, L., and Chen, W.
\newblock Lo{RA}: Low-rank adaptation of large language models.
\newblock In \emph{International Conference on Learning Representations}, 2022.
\newblock URL \url{https://openreview.net/forum?id=nZeVKeeFYf9}.

\bibitem[Hu et~al.(2023)Hu, Lan, Wang, Xu, Lim, Lee, Bing, and Poria]{hu2023llm}
Hu, Z., Lan, Y., Wang, L., Xu, W., Lim, E.-P., Lee, R. K.-W., Bing, L., and Poria, S.
\newblock Llm-adapters: An adapter family for parameter-efficient fine-tuning of large language models.
\newblock \emph{arXiv preprint arXiv:2304.01933}, 2023.

\bibitem[H{\"u}llermeier \& Waegeman(2021)H{\"u}llermeier and Waegeman]{hullermeier2021aleatoric}
H{\"u}llermeier, E. and Waegeman, W.
\newblock Aleatoric and epistemic uncertainty in machine learning: An introduction to concepts and methods.
\newblock \emph{Machine learning}, 110\penalty0 (3):\penalty0 457--506, 2021.

\bibitem[Jiang et~al.(2023)Jiang, Sablayrolles, Mensch, Bamford, Chaplot, Casas, Bressand, Lengyel, Lample, Saulnier, et~al.]{jiang2023mistral}
Jiang, A.~Q., Sablayrolles, A., Mensch, A., Bamford, C., Chaplot, D.~S., Casas, D. d.~l., Bressand, F., Lengyel, G., Lample, G., Saulnier, L., et~al.
\newblock Mistral 7b.
\newblock \emph{arXiv preprint arXiv:2310.06825}, 2023.

\bibitem[Kadavath et~al.(2022)Kadavath, Conerly, Askell, Henighan, Drain, Perez, Schiefer, Hatfield-Dodds, DasSarma, Tran-Johnson, Johnston, El-Showk, Jones, Elhage, Hume, Chen, Bai, Bowman, Fort, Ganguli, Hernandez, Jacobson, Kernion, Kravec, Lovitt, Ndousse, Olsson, Ringer, Amodei, Brown, Clark, Joseph, Mann, McCandlish, Olah, and Kaplan]{kadavath2022language}
Kadavath, S., Conerly, T., Askell, A., Henighan, T., Drain, D., Perez, E., Schiefer, N., Hatfield-Dodds, Z., DasSarma, N., Tran-Johnson, E., Johnston, S., El-Showk, S., Jones, A., Elhage, N., Hume, T., Chen, A., Bai, Y., Bowman, S., Fort, S., Ganguli, D., Hernandez, D., Jacobson, J., Kernion, J., Kravec, S., Lovitt, L., Ndousse, K., Olsson, C., Ringer, S., Amodei, D., Brown, T., Clark, J., Joseph, N., Mann, B., McCandlish, S., Olah, C., and Kaplan, J.
\newblock Language models (mostly) know what they know, 2022.

\bibitem[Kapoor et~al.(2024)Kapoor, Gruver, Roberts, Collins, Pal, Bhatt, Weller, Dooley, Goldblum, and Wilson]{kapoor2024large}
Kapoor, S., Gruver, N., Roberts, M., Collins, K., Pal, A., Bhatt, U., Weller, A., Dooley, S., Goldblum, M., and Wilson, A.~G.
\newblock Large language models must be taught to know what they don't know.
\newblock \emph{arXiv preprint arXiv:2406.08391}, 2024.

\bibitem[Karush(1939)]{karush1939minima}
Karush, W.
\newblock Minima of functions of several variables with inequalities as side constraints.
\newblock \emph{M. Sc. Dissertation. Dept. of Mathematics, Univ. of Chicago}, 1939.

\bibitem[Khan et~al.(2018)Khan, Nielsen, Tangkaratt, Lin, Gal, and Srivastava]{khan2018fast}
Khan, M., Nielsen, D., Tangkaratt, V., Lin, W., Gal, Y., and Srivastava, A.
\newblock Fast and scalable bayesian deep learning by weight-perturbation in adam.
\newblock In \emph{International conference on machine learning}, pp.\  2611--2620. PMLR, 2018.

\bibitem[Kingma \& Ba(2014)Kingma and Ba]{kingma2014adam}
Kingma, D.~P. and Ba, J.
\newblock Adam: A method for stochastic optimization.
\newblock \emph{arXiv preprint arXiv:1412.6980}, 2014.

\bibitem[Kjeldsen(2000)]{kjeldsen2000contextualized}
Kjeldsen, T.~H.
\newblock A contextualized historical analysis of the kuhn--tucker theorem in nonlinear programming: the impact of world war ii.
\newblock \emph{Historia mathematica}, 27\penalty0 (4):\penalty0 331--361, 2000.

\bibitem[Klema \& Laub(1980)Klema and Laub]{klema1980singular}
Klema, V. and Laub, A.
\newblock The singular value decomposition: Its computation and some applications.
\newblock \emph{IEEE Transactions on automatic control}, 25\penalty0 (2):\penalty0 164--176, 1980.

\bibitem[Knoblauch et~al.(2019)Knoblauch, Jewson, and Damoulas]{knoblauch2019generalized}
Knoblauch, J., Jewson, J., and Damoulas, T.
\newblock Generalized variational inference: Three arguments for deriving new posteriors.
\newblock \emph{arXiv preprint arXiv:1904.02063}, 2019.

\bibitem[Kopiczko et~al.(2023)Kopiczko, Blankevoort, and Asano]{kopiczko2023vera}
Kopiczko, D.~J., Blankevoort, T., and Asano, Y.~M.
\newblock Vera: Vector-based random matrix adaptation.
\newblock \emph{arXiv preprint arXiv:2310.11454}, 2023.

\bibitem[Krishnan et~al.(2022)Krishnan, Esposito, and Subedar]{krishnan2022bayesiantorch}
Krishnan, R., Esposito, P., and Subedar, M.
\newblock Bayesian-torch: Bayesian neural network layers for uncertainty estimation.
\newblock \url{https://github.com/IntelLabs/bayesian-torch}, January 2022.
\newblock URL \url{https://doi.org/10.5281/zenodo.5908307}.

\bibitem[Kuhn \& Tucker(1951)Kuhn and Tucker]{kuhn1951nonlinear}
Kuhn, H.~W. and Tucker, A.~W.
\newblock Nonlinear programming.
\newblock In \emph{Berkeley Symposium on Mathematical Statistics and Probability}, pp.\  481--492. Berkeley: University of California Press, 1951.

\bibitem[Kuhn et~al.(2023)Kuhn, Gal, and Farquhar]{kuhn2023semantic}
Kuhn, L., Gal, Y., and Farquhar, S.
\newblock Semantic uncertainty: Linguistic invariances for uncertainty estimation in natural language generation.
\newblock In \emph{The Eleventh International Conference on Learning Representations}, 2023.

\bibitem[Lakshminarayanan et~al.(2017)Lakshminarayanan, Pritzel, and Blundell]{lakshminarayanan2017simple}
Lakshminarayanan, B., Pritzel, A., and Blundell, C.
\newblock Simple and scalable predictive uncertainty estimation using deep ensembles.
\newblock \emph{Advances in neural information processing systems}, 30, 2017.

\bibitem[Lin et~al.(2022)Lin, Hilton, and Evans]{lin2022teaching}
Lin, S., Hilton, J., and Evans, O.
\newblock Teaching models to express their uncertainty in words.
\newblock \emph{Transactions on Machine Learning Research}, 2022.

\bibitem[Meng et~al.(2024)Meng, Wang, and Zhang]{meng2024pissa}
Meng, F., Wang, Z., and Zhang, M.
\newblock Pissa: Principal singular values and singular vectors adaptation of large language models.
\newblock \emph{arXiv preprint arXiv:2404.02948}, 2024.

\bibitem[Mihaylov et~al.(2018)Mihaylov, Clark, Khot, and Sabharwal]{obqa}
Mihaylov, T., Clark, P., Khot, T., and Sabharwal, A.
\newblock Can a suit of armor conduct electricity? a new dataset for open book question answering.
\newblock In Riloff, E., Chiang, D., Hockenmaier, J., and Tsujii, J. (eds.), \emph{Proceedings of the 2018 Conference on Empirical Methods in Natural Language Processing}, pp.\  2381--2391, Brussels, Belgium, October-November 2018. Association for Computational Linguistics.
\newblock \doi{10.18653/v1/D18-1260}.
\newblock URL \url{https://aclanthology.org/D18-1260}.

\bibitem[Milne(2019)]{NEURIPS2019_b33128cb}
Milne, T.
\newblock Piecewise strong convexity of neural networks.
\newblock In Wallach, H., Larochelle, H., Beygelzimer, A., d\textquotesingle Alch\'{e}-Buc, F., Fox, E., and Garnett, R. (eds.), \emph{Advances in Neural Information Processing Systems}, volume~32. Curran Associates, Inc., 2019.
\newblock URL \url{https://proceedings.neurips.cc/paper_files/paper/2019/file/b33128cb0089003ddfb5199e1b679652-Paper.pdf}.

\bibitem[Min et~al.(2022)Min, Lyu, Holtzman, Artetxe, Lewis, Hajishirzi, and Zettlemoyer]{min2022rethinking}
Min, S., Lyu, X., Holtzman, A., Artetxe, M., Lewis, M., Hajishirzi, H., and Zettlemoyer, L.
\newblock Rethinking the role of demonstrations: What makes in-context learning work?
\newblock \emph{arXiv preprint arXiv:2202.12837}, 2022.

\bibitem[Naeini et~al.(2015)Naeini, Cooper, and Hauskrecht]{naeini2015obtaining}
Naeini, M.~P., Cooper, G., and Hauskrecht, M.
\newblock Obtaining well calibrated probabilities using bayesian binning.
\newblock In \emph{Proceedings of the AAAI conference on artificial intelligence}, volume~29, 2015.

\bibitem[Neal(2012)]{neal2012bayesian}
Neal, R.~M.
\newblock \emph{Bayesian learning for neural networks}, volume 118.
\newblock Springer Science \& Business Media, 2012.

\bibitem[Nikitin et~al.(2024)Nikitin, Kossen, Gal, and Marttinen]{nikitin2024kernel}
Nikitin, A., Kossen, J., Gal, Y., and Marttinen, P.
\newblock Kernel language entropy: Fine-grained uncertainty quantification for llms from semantic similarities.
\newblock \emph{arXiv preprint arXiv:2405.20003}, 2024.

\bibitem[OpenAI(2022)]{achiam2022chatgpt}
OpenAI.
\newblock Introducing chatgpt. [online]. available: \url{https://openai.com/blog/chatgpt}.
\newblock 2022.

\bibitem[Radford et~al.(2019)Radford, Wu, Child, Luan, Amodei, Sutskever, et~al.]{radford2019language}
Radford, A., Wu, J., Child, R., Luan, D., Amodei, D., Sutskever, I., et~al.
\newblock Language models are unsupervised multitask learners.
\newblock \emph{OpenAI blog}, 1\penalty0 (8):\penalty0 9, 2019.

\bibitem[Sakaguchi et~al.(2021)Sakaguchi, Bras, Bhagavatula, and Choi]{wg}
Sakaguchi, K., Bras, R.~L., Bhagavatula, C., and Choi, Y.
\newblock Winogrande: an adversarial winograd schema challenge at scale.
\newblock \emph{Commun. ACM}, 64\penalty0 (9):\penalty0 99–106, aug 2021.
\newblock ISSN 0001-0782.
\newblock \doi{10.1145/3474381}.
\newblock URL \url{https://doi.org/10.1145/3474381}.

\bibitem[Sengupta et~al.(2024)Sengupta, Seth, Pathak, Verma, Raman, Gopalakrishnan, Chatterjee, and Chakraborty]{sengupta2024robust}
Sengupta, A., Seth, V., Pathak, A., Verma, A., Raman, N., Gopalakrishnan, S., Chatterjee, N., and Chakraborty, T.
\newblock Robust and efficient fine-tuning of llms with bayesian reparameterization of low-rank adaptation.
\newblock \emph{arXiv preprint arXiv:2411.04358}, 2024.

\bibitem[Tian et~al.(2023)Tian, Mitchell, Zhou, Sharma, Rafailov, Yao, Finn, and Manning]{tian2023just}
Tian, K., Mitchell, E., Zhou, A., Sharma, A., Rafailov, R., Yao, H., Finn, C., and Manning, C.
\newblock Just ask for calibration: Strategies for eliciting calibrated confidence scores from language models fine-tuned with human feedback.
\newblock In Bouamor, H., Pino, J., and Bali, K. (eds.), \emph{Proceedings of the 2023 Conference on Empirical Methods in Natural Language Processing}, pp.\  5433--5442, Singapore, December 2023. Association for Computational Linguistics.
\newblock \doi{10.18653/v1/2023.emnlp-main.330}.
\newblock URL \url{https://aclanthology.org/2023.emnlp-main.330}.

\bibitem[Tierney \& Kadane(1986)Tierney and Kadane]{tierney1986accurate}
Tierney, L. and Kadane, J.~B.
\newblock Accurate approximations for posterior moments and marginal densities.
\newblock \emph{Journal of the american statistical association}, 81\penalty0 (393):\penalty0 82--86, 1986.

\bibitem[Touvron et~al.(2023{\natexlab{a}})Touvron, Lavril, Izacard, Martinet, Lachaux, Lacroix, Rozi{\`e}re, Goyal, Hambro, Azhar, et~al.]{touvron2023llama}
Touvron, H., Lavril, T., Izacard, G., Martinet, X., Lachaux, M.-A., Lacroix, T., Rozi{\`e}re, B., Goyal, N., Hambro, E., Azhar, F., et~al.
\newblock Llama: Open and efficient foundation language models.
\newblock \emph{arXiv preprint arXiv:2302.13971}, 2023{\natexlab{a}}.

\bibitem[Touvron et~al.(2023{\natexlab{b}})Touvron, Martin, Stone, Albert, Almahairi, Babaei, Bashlykov, Batra, Bhargava, Bhosale, et~al.]{touvron2023llama2}
Touvron, H., Martin, L., Stone, K., Albert, P., Almahairi, A., Babaei, Y., Bashlykov, N., Batra, S., Bhargava, P., Bhosale, S., et~al.
\newblock Llama 2: Open foundation and fine-tuned chat models.
\newblock \emph{arXiv preprint arXiv:2307.09288}, 2023{\natexlab{b}}.

\bibitem[Wang \& Yeung(2016)Wang and Yeung]{wang2016towards}
Wang, H. and Yeung, D.-Y.
\newblock Towards bayesian deep learning: A framework and some existing methods.
\newblock \emph{IEEE Transactions on Knowledge and Data Engineering}, 28\penalty0 (12):\penalty0 3395--3408, 2016.

\bibitem[Wang et~al.(2016)Wang, Shi, and Yeung]{wang2016natural}
Wang, H., Shi, X., and Yeung, D.-Y.
\newblock Natural-parameter networks: A class of probabilistic neural networks.
\newblock \emph{Advances in neural information processing systems}, 29, 2016.

\bibitem[Wang et~al.(2024{\natexlab{a}})Wang, Tan, Hong, Zhang, and Wang]{wang2024variational}
Wang, H., Tan, S., Hong, Z., Zhang, D., and Wang, H.
\newblock Variational language concepts for interpreting foundation language models.
\newblock In Al-Onaizan, Y., Bansal, M., and Chen, Y.-N. (eds.), \emph{Findings of the Association for Computational Linguistics: EMNLP 2024}, pp.\  8645--8671, Miami, Florida, USA, November 2024{\natexlab{a}}. Association for Computational Linguistics.
\newblock \doi{10.18653/v1/2024.findings-emnlp.505}.
\newblock URL \url{https://aclanthology.org/2024.findings-emnlp.505/}.

\bibitem[Wang et~al.(2024{\natexlab{b}})Wang, Tan, and Wang]{wang2024probabilistic}
Wang, H., Tan, S., and Wang, H.
\newblock Probabilistic conceptual explainers: trustworthy conceptual explanations for vision foundation models.
\newblock In \emph{Proceedings of the 41st International Conference on Machine Learning}, ICML'24. JMLR.org, 2024{\natexlab{b}}.

\bibitem[Wang et~al.(2024{\natexlab{c}})Wang, Xiao, Li, Wang, Chen, and Chen]{wang2024milora}
Wang, H., Xiao, Z., Li, Y., Wang, S., Chen, G., and Chen, Y.
\newblock Milora: Harnessing minor singular components for parameter-efficient llm finetuning.
\newblock \emph{arXiv preprint arXiv:2406.09044}, 2024{\natexlab{c}}.

\bibitem[Wang et~al.(2023)Wang, Aitchison, and Rudolph]{wang2023lora}
Wang, X., Aitchison, L., and Rudolph, M.
\newblock Lora ensembles for large language model fine-tuning, 2023.

\bibitem[Wang et~al.(2024{\natexlab{d}})Wang, Shi, Han, Metaxas, and Wang]{wang2024blob}
Wang, Y., Shi, H., Han, L., Metaxas, D., and Wang, H.
\newblock Blob: Bayesian low-rank adaptation by backpropagation for large language models.
\newblock \emph{arXiv preprint arXiv:2406.11675}, 2024{\natexlab{d}}.

\bibitem[Wei et~al.(2021)Wei, Bosma, Zhao, Guu, Yu, Lester, Du, Dai, and Le]{wei2021finetuned}
Wei, J., Bosma, M., Zhao, V.~Y., Guu, K., Yu, A.~W., Lester, B., Du, N., Dai, A.~M., and Le, Q.~V.
\newblock Finetuned language models are zero-shot learners.
\newblock \emph{arXiv preprint arXiv:2109.01652}, 2021.

\bibitem[Wei et~al.(2022)Wei, Tay, Bommasani, Raffel, Zoph, Borgeaud, Yogatama, Bosma, Zhou, Metzler, et~al.]{wei2022emergent}
Wei, J., Tay, Y., Bommasani, R., Raffel, C., Zoph, B., Borgeaud, S., Yogatama, D., Bosma, M., Zhou, D., Metzler, D., et~al.
\newblock Emergent abilities of large language models.
\newblock \emph{arXiv preprint arXiv:2206.07682}, 2022.

\bibitem[Xiong et~al.(2023)Xiong, Hu, Lu, Li, Fu, He, and Hooi]{xiong2023can}
Xiong, M., Hu, Z., Lu, X., Li, Y., Fu, J., He, J., and Hooi, B.
\newblock Can llms express their uncertainty? an empirical evaluation of confidence elicitation in llms.
\newblock \emph{arXiv preprint arXiv:2306.13063}, 2023.

\bibitem[Yadkori et~al.(2024)Yadkori, Kuzborskij, Gy{\"o}rgy, and Szepesv{\'a}ri]{yadkori2024believe}
Yadkori, Y.~A., Kuzborskij, I., Gy{\"o}rgy, A., and Szepesv{\'a}ri, C.
\newblock To believe or not to believe your llm.
\newblock \emph{arXiv preprint arXiv:2406.02543}, 2024.

\bibitem[Yang et~al.(2023)Yang, Robeyns, Wang, and Aitchison]{yang2023bayesian}
Yang, A.~X., Robeyns, M., Wang, X., and Aitchison, L.
\newblock Bayesian low-rank adaptation for large language models.
\newblock \emph{arXiv preprint arXiv:2308.13111}, 2023.

\bibitem[Zhang et~al.(2023)Zhang, Chen, Bukharin, Karampatziakis, He, Cheng, Chen, and Zhao]{zhang2023adalora}
Zhang, Q., Chen, M., Bukharin, A., Karampatziakis, N., He, P., Cheng, Y., Chen, W., and Zhao, T.
\newblock Adalora: Adaptive budget allocation for parameter-efficient fine-tuning.
\newblock \emph{arXiv preprint arXiv:2303.10512}, 2023.

\bibitem[Zhang et~al.(2021)Zhang, Fan, Chen, and Zhou]{zhang2021bayesian}
Zhang, S., Fan, X., Chen, B., and Zhou, M.
\newblock Bayesian attention belief networks.
\newblock In \emph{International Conference on Machine Learning}, pp.\  12413--12426. PMLR, 2021.

\bibitem[Zhao et~al.(2021)Zhao, Wallace, Feng, Klein, and Singh]{zhao2021calibrate}
Zhao, Z., Wallace, E., Feng, S., Klein, D., and Singh, S.
\newblock Calibrate before use: Improving few-shot performance of language models.
\newblock In \emph{International conference on machine learning}, pp.\  12697--12706. PMLR, 2021.

\end{thebibliography}
\bibliographystyle{icml2025}

\appendix
\clearpage
\onecolumn
\section*{\LARGE Appendix}
\markboth{Appendix}{Appendix}
In \appref{app:algo}, we present the full algorithmic description of \ours.
In \appref{sec:limitations}, we present the limitations of \ours.
Next, in \appref{app:related}, we present a more detailed introduction to recent advances of Bayesian Low-Rank Adaptation.
In \appref{app:proof}, we provide detailed proofs for all theorems presented in the main paper. 
In \appref{app:implementation}, we describe our experimental methodology. 
Finally, in \appref{app:more}, we present additional empirical results, including:
\begin{itemize}[nosep]
\item {a visual study demonstrating \ours functions as a general Bayesian Neural Network~(\appref{app:visualization}),} 
\item {test-time sample size analysis~(\appref{app:tts}),}
\item {anchor dataset size analysis~(\appref{app:anchor-size-analysis}),}
\item \ours with unlabeled test data as the anchor data~(\appref{app:more-eval-as-anchor}),
\item \ours with non-NLL evaluation metrics~(\appref{app:more-acc-as-metric}),
\item \ours with other single-parameter variational distribution families~(\appref{app:more-ablation-full}),
\item \ours with other LLM backbones than Llama3.1-8B~(\appref{app:more-other-llms}),
\item \ours with other PEFT methods than LoRA~(\appref{app:more-other-pefts}),
\item improving inference-time efficiency of \ours with last-layer Bayesianization~(\appref{app:tfb-ll}), and
\item the full results on a widely used LLM architecture Llama2-7B~(\appref{app:more-llama2}).

\end{itemize}

\section{\ours: Algorithm}
\label{app:algo}
\begin{algorithm}[H]
\caption{\textbf{T}raining-\textbf{F}ree \textbf{B}ayesianization~(\textbf{\ours})}\label{alg:tfb}
\begin{algorithmic}[1]
\INPUT
$\gD$: Anchor Dataset;\par
$\{\mB,\mA\}$: Low-Rank Component;\par
$l$: Model Evaluation Metric; \par
$\epsilon$: Performance Change Tolerance; \par
$[{\sigma_q}_{\text{min}}, {\sigma_q}_{\text{max}}]$: search range of $\sigma_q$. 
\STATE Evaluate the original performance: $p_0\leftarrow l(\gD|\mB,\mA)$.
\STATE Singular Value Decomposition on $\mB$:\par
    \quad $\mU,\diag(\vd),\mV\leftarrow \operatorname{SVD}(\mB)$.
    \hfill$\rhd$ \Eqref{eq:svd}.
\STATE Get an equivalent pair of the low-rank component: \par 
    \quad $\mB^\prime\leftarrow\mU\diag(\vd)$; $\mA^\prime\leftarrow\mV^\top\mA$. \hfill$\rhd$ \Eqref{eq:regroup}.
\WHILE{$\sigma_q$ not converged} 
    \STATE $\sigma_q \leftarrow \nicefrac{({\sigma_q}_{\text{max}} + {\sigma_q}_{\text{min}})}{2}$.
    \STATE Calculate the standard deviation matrix $\mOmega$ for $\mA^\prime$: \par
        \quad $\Omega_{ij}=\nicefrac{\sigma_q}{d_{i}}$. \hfill$\rhd$ \Eqref{eq:bayesianization}.
    \STATE Evaluate the performance: \par 
        \quad $p\leftarrow l(\gD|\mB^\prime,\mA^\prime,\mOmega)$.
    \IF{$|p-p_0|<\epsilon$}
        \STATE ${\sigma_q}_{\text{min}}\leftarrow \sigma_q$. 
    \ELSE
        \STATE ${\sigma_q}_{\text{max}}\leftarrow \sigma_q$.
    \ENDIF
\ENDWHILE
\OUTPUT
$\{\mB^\prime, \mA^\prime, \mOmega\}$: Bayesianized Low-Rank Adapter.
\end{algorithmic}
\end{algorithm}

\section{Broader Impact and Limitations}
\label{sec:limitations}
\textbf{Broader Impact.}\quad 
This research advances methods for making large language models more trustworthy and reliable through improved uncertainty estimation. While we focus on language models, the fundamental principles of our framework can enhance uncertainty quantification across the broader machine learning field. This wider applicability creates opportunities for improving model reliability and safety across diverse applications.
To the best of our knowledge, there are no ethical or other concerns that need to be addressed.

\textbf{Limitations.}\quad 
\ours is subject to several limitations. 
First, our approach relies on the availability of an anchor dataset for determining search range and stopping criteria. Although this dataset doesn't require supervision or prior use in LoRA training, its quality and representativeness could impact the effectiveness of uncertainty estimation.
Second, by constraining the family of full-weight posteriors to low-rank isotropic Gaussian distributions, \ours may not capture more complex uncertainty patterns that could be present in the data. 
At first blush, 
this seems to imply a trade-off between computational efficiency and model expressiveness. 
{However, in practice, the trade-off may not be necessary as \ours can often enjoy both computational efficiency and model expressiveness, getting the best of both worlds. Given \ours's proven effectiveness and inherent simplicity, we recommend implementing \ours as the initial approach when developing reliable LLMs with existing LoRA adapters. If \ours fails to meet specific requirements, practitioners can then consider alternative expensive Bayesian methods.}
Finally, while we have demonstrated the effectiveness of \ours in various settings, its performance in more complex generation tasks requires further investigation. Future work could explore extending the framework to handle more sophisticated language generation scenarios and broader applications.

\section{Related Work}
\label{app:related}

\textbf{Bayesian Low-Rank Adaptation.}\quad
The Bayesian framework provides a powerful approach for capturing and estimating uncertainty by defining prior distributions and approximating posterior distributions over the parameter space~\cite{neal2012bayesian,hernandez2015probabilistic,gal2016dropout, wang2016towards, gustafsson2020evaluating}. However, modeling parameter distributions across the entire parameter space during fine-tuning introduces significant computational overhead~\cite{fan2020bayesian, zhang2021bayesian}. To address this challenge, recent research has explored combining Bayesian methods with Parameter-Efficient Fine-Tuning (PEFT) techniques to improve the efficiency of uncertainty estimation.
Several notable approaches have emerged in this direction. \citet{wang2023lora} and \citet{balabanov2024uncertainty} demonstrate improved performance by training multiple LoRA modules and ensemble their predictions during inference. Taking a different approach, \citet{yang2023bayesian} applies a Kronecker-factorized Laplace approximation to fine-tuned LoRA parameters. More recently, \blob~\cite{wang2024blob} advances the field by simultaneously estimating both the mean and covariance of LLM parameters within a single fine-tuning stage, leading to substantial performance improvements.
Our proposed training-free Bayesianization represents a significant departure from these existing methods. Unlike approaches that require re-training~\cite{gal2016dropout, wang2023lora, balabanov2024uncertainty, wang2024blob} or rely on continued training and gradient estimation~\cite{yang2023bayesian}, our method achieves uncertainty estimation without any additional training steps, substantially improving the simplicity and efficiency for Bayesian learning of LLMs.

\section{Proof of Theorems}
\label{app:proof}
\begin{customThm}{4.1}[\textbf{Equivalent Variational Distribution of the Full Weight $\mW$ in \ours}]
    With the pre-trained weight matrix $\mW_0 \in\mathbb{R}^{m\times n}$, 
    the low-rank weight update matrix $\{\mB^\prime \in \mathbb{R}^{m\times r}, \mA^\prime \in \mathbb{R}^{r\times n}\}$ transformed from the given matrices $\{\mB, \mA\}$ following \Eqref{eq:svd} and \ref{eq:regroup}, 
    suppose that the variational distribution of $\mA^\prime$ is Gaussian $q(\mA^\prime|\vtheta)=\prod_{ij} \gN (A_{ij}|M_{ij},\Omega_{ij}^2)$, 
    where $\mM=[M_{ij}=A^\prime_{ij}] \in \mathbb{R}^{r\times n}$ is its mean and $\mOmega=[\Omega_{ij}] \in \mathbb{R}^{r\times n}$ is the  standard deviation calculated as in \Eqref{eq:bayesianization}. The equivalent variational distribution $q(\vectorize(\mW)|\sigma_q)$ defined on the full weight matrix $\mW$ is
    \begin{equation}
    \begin{aligned}
        q(\vectorize(\mW)|\sigma_q) &= \gN(\vectorize(\mW)|\vmu_q, \mSigma_q), \\
        \text{where }\quad 
        \vmu_q &= \vectorize(\mW_0+\mB^\prime\mM), \\
        \mSigma_q &= \sigma_q^2  \mI_n \otimes 
        \begin{bmatrix}
            \mI_r & \\
            & \mathbf{0}_{m-r}
        \end{bmatrix}.
    \end{aligned}
    \end{equation}
\end{customThm}

\begin{proof}
    We have the following lemma from \blob that calculates the covariance matrix of a given low-rank Bayesianization scheme $\{\mB, \mA, \mOmega\}$~\cite{wang2024blob}. 

\begin{lemma}\label{lemma:blob}
    {With the pre-trained weight matrix $\mW_0 \in\mathbb{R}^{m\times n}$ and the low-rank weight update matrix $\mB \in \mathbb{R}^{m\times r}$, suppose that the variational distribution of the other low-rank update matrix $\mA \in \mathbb{R}^{r\times n}$ is Gaussian with $q(\mA|\vtheta=\{\mM, \mOmega\})=\prod_{ij} \gN (A_{ij}|M_{ij},\Omega_{ij}^2)$, where $\mM=[M_{ij}] \in \mathbb{R}^{r\times n}$ and $\mOmega=[\Omega_{ij}] \in \mathbb{R}^{r\times n}$ are its mean and standard deviation, respectively. The equivalent variational distribution defined on the full weight matrix $\mW$ is given by
    \begin{equation}
    \begin{aligned}
        q(\vectorize(\mW)|\mB, \vtheta) &= \gN(\vectorize(\mW)|\vmu_q, \mSigma_q), \\
        \text{where }\quad 
        \vmu_q &= \vectorize(\mW_0+\mB\mM), \\
        \mSigma_q &= [\mI_n \otimes \mB]  [\diag(\vectorize(\mOmega)^2)]  [\mI_n \otimes \mB^\top].  
    \end{aligned}
    \end{equation}}
\end{lemma}

    Based on the assumption outlined in \Eqref{eq:svd}, \ref{eq:regroup}, and \ref{eq:bayesianization}, we have the following properties about $\mB^\prime$, $\mM$, and $\mOmega$ of \ours:
    \begin{align}
        \mB^\prime =& \mU \diag(\vd), \\
        \mOmega =& [1/\vd, \cdots, 1/\vd], \\
        \text{where }\space & \mU^\top\mU = \mI_r, \mU\mU^\top = 
        \begin{bmatrix}
            \mI_r & \\
            & \mathbf{0}_{m-r}
        \end{bmatrix}.
    \end{align}
    
    It now can be easily shown that the covariance matrix of \ours is: 
    \begin{align}
        \mSigma_q 
        &= [\mI_n \otimes \mB^\prime]  [\diag(\vectorize(\mOmega)^2)]  [\mI_n \otimes \mB^{\prime\top}] \\
        &= [\mI_n \otimes \mB^\prime]  [\mI_n \otimes \diag(1/\vd)^2]  [\mI_n \otimes \mB^{\prime\top}] \\
        &= \mI_n \otimes [\mB^\prime \diag(\sigma_q/\vd)^2 \mB^{\prime\top}] \\
        &= \mI_n \otimes [\mU\diag(\vd) \diag(\sigma_q/\vd)^2 \diag(\vd)^\top\mU^\top] \\
        &= \sigma_q^2  \mI_n \otimes 
        \begin{bmatrix}
            \mI_r & \\
            & \mathbf{0}_{m-r}
        \end{bmatrix},
    \end{align}
    which proves that $q(\vectorize(\mW))$ is a low-rank isotropic Gaussian distribution.
\end{proof}

\begin{proposition}\label{lemma:tfb-proj}
    The function $\operatorname{proj}(\cdot)$ defined in \Eqref{eq:tfb-posterior} projects the full-dimensional isotropic Gaussian to the low-rank subspace of LoRA. It can be formulated as
    \begin{align}
        \operatorname{proj}(\sigma_q^2 \mI_{mn}) &= \mP (\sigma_q^2 \mI_{mn}), \\
        \text{where}\quad \mP &= \mI_{n} \otimes 
        \begin{bmatrix}
            \mI_r & \\
            & \mathbf{0}_{m-r}
        \end{bmatrix}.
    \end{align}
\end{proposition}

\begin{proof}
    By \Thmref{thm:posterior}, we have
    \begin{align}
        \mP \mI_{mn} &= \mI_n \otimes 
        \begin{bmatrix}
            \mI_r & \\
            & \mathbf{0}_{m-r}
        \end{bmatrix}.
    \end{align}
    Hence it is trivial to have $\mP=\mI_n \otimes 
        \begin{bmatrix}
            \mI_r & \\
            & \mathbf{0}_{m-r}
        \end{bmatrix}.$
\end{proof}

\begin{customThm}{4.2}[\textbf{\ours as Generalized Variational Inference}]
    Suppose the evaluation metric $l_\gD(\sigma_q)$ defined following \Assumptionref{assumption:tfb} is locally convex within the range of $\sigma_q\in[0,\epsilon_0)$. 
    Suppose the approximate distribution of $\mW$ given $\sigma_q$ is defined following \Thmref{thm:posterior}. 
    Suppose we have the prior distribution $P(\vectorize(\mW))=\gN(\vectorize(\mW)|\vmu_p, \mSigma_p)$, where $\vmu_p=\vmu_q=\vectorize(\mW_0+\mB^\prime\mM)$, and $\mSigma_p=\sigma_p^2\mI$ with $\sigma_p > \epsilon_0$.
    Then for $\forall \lambda>0$, $\exists \tilde{\epsilon}$, s.t. the following two optimization problems
    (i)~Generalized Variational Inference~\cite{blundell2015BBB,higgins2017beta,khan2018fast}
        \begin{equation}
        \begin{aligned}
            \min_{\sigma_q} \quad l_\gD(\sigma_q) + \lambda \operatorname{KL}[q(\mW|\sigma_q) \parallel P(\mW)],
        \end{aligned}
        \end{equation}
    and 
    (ii)~Training-Free Bayesianization~(\ours)
        \begin{equation}
        \begin{aligned}
            \max \quad & \sigma_q \\
            \textit{s.t.} \quad & l_\gD(\sigma_q) \leq \tilde{\epsilon},
        \end{aligned}
        \end{equation}
    are equivalent, i.e., the two optimization problems have the same optimal solution,
    where $\lambda$ is the regularization coefficient of the KL-divergence. 
\end{customThm}

\begin{proof}
    First we prove the KL divergence term is convex w.r.t. $\sigma_q$. 
    For two Gaussian distributions $q$ and $p$ whose covariance matrices $\mSigma_q\in \mathbb{R}^{d\times d}$ and $\mSigma_p\in \mathbb{R}^{d\times d}$ are both full-rank, with their means as $\vmu_q\in \mathbb{R}^{d}$ and $\vmu_p\in \mathbb{R}^{d}$, we have their KL-divergence as
\begin{align}\label{eq:gaussian-kl}
    \operatorname{KL}[q \| p] &= \frac{1}{2} \left[ 
        \log \frac{|\mSigma_p|}{|\mSigma_q|} 
        - d 
        + \operatorname{tr}(\mSigma_p^{-1}\mSigma_q) 
        + (\vmu_q - \vmu_p)^\top \mSigma_p^{-1} (\vmu_q - \vmu_p)
    \right].
\end{align}

For \ours, to avoid unbounded KL divergence, we project the original assumed Gaussian prior $P$ into the same low-rank sub-space of the posterior $q$. 
We summarize the prior and variational distribution of the posterior as follows: 
\begin{equation}\label{eq:summarized-vfe}
\begin{aligned}
    q(\vectorize(\mW)|\sigma_q) &= \gN\left(\vectorize(\mW)|\vmu_q=\vectorize(\mW_0+\mB^\prime\mM), \mSigma_q=\sigma_q^2  \mI_n \otimes 
        \begin{bmatrix}
            \mI_r & \\
            & \mathbf{0}_{m-r}
        \end{bmatrix}
        \right), \\
    P(\vectorize(\mW)|\sigma_p) &= \gN\left(\vectorize(\mW)|\vmu_p=\vectorize(\mW_0+\mB^\prime\mM), \mSigma_p=\sigma_p^2  \mI_n \otimes 
        \begin{bmatrix}
            \mI_r & \\
            & \mathbf{0}_{m-r}
        \end{bmatrix}
        \right). 
\end{aligned}
\end{equation}

Substituting \Eqref{eq:summarized-vfe} back into \Eqref{eq:gaussian-kl}, we have
\begin{align}
    \operatorname{KL}[q(\vectorize(\mW)|\sigma_q) \| P(\vectorize(\mW|\sigma_p))] &= \frac{nr}{2}\left[
        \log(\sigma_p^2) - 1 + \left\{
            - \log(\sigma_q^2) + \frac{\sigma_q^2}{\sigma_p^2}
        \right\}
    \right], 
\end{align}
which is convex w.r.t. $\sigma_q$ and the global minimum of $\operatorname{KL}$ is achieved when $\sigma_q=\sigma_p$. 

With $\sigma_q\leq \epsilon_0$, the convexity of two terms ($\operatorname{KL}$ and $l_\gD$) holds. 
Hence we show by the Karush–Kuhn–Tucker theorem~\cite{kjeldsen2000contextualized,karush1939minima,kuhn1951nonlinear} that, for any given $\lambda$ there exists $\tilde{\epsilon}$ such that the following two optimization problems are equivalent:

\begin{enumerate}
\item Minimization of generalized variational inference in the \textbf{Lagrange-form optimization}
\begin{align}\label{eq:lagrange}
    \min_{\sigma_q} \quad \operatorname{KL}[q(\vectorize(\mW)|\sigma_q) \parallel P(\vectorize(\mW)|\sigma_p)] + \frac{1}{\lambda} l_\gD(\sigma_q);
\end{align}
\item The \textbf{constrained-form optimization} corresponding to \Eqref{eq:lagrange}
        \begin{equation}\label{eq:constrained}
        \begin{aligned}
            \min \quad & \operatorname{KL}[q(\vectorize(\mW)|\sigma_q) \parallel P(\vectorize(\mW)|\sigma_p)] \\
            \textit{s.t.} \quad & l_\gD(\sigma_q) \leq \tilde{\epsilon}.
        \end{aligned}
        \end{equation}
\end{enumerate}

Since the $\operatorname{KL}$ term is monotonically decreasing when $\sigma_q \in [0, \sigma_p)$, and due to the fact that $\sigma_p > \epsilon_0$, the optimization in \Eqref{eq:constrained} is equivalent to our final Training-Free Bayesianization~(\ours):
\begin{equation}
        \begin{aligned}
            \max \quad & \sigma_q \\
            \textit{s.t.} \quad & l_\gD(\sigma_q) \leq \tilde{\epsilon}.
        \end{aligned}
        \end{equation}
\end{proof}

\section{Implementation Details}
\label{app:implementation}

\subsection{Datasets}
\label{app:implementation-dataset}
We provide details of the datasets used in this work, as shown in \Tabref{tab:dataset}. 
The combined dataset consisting of the six commonsense reasoning tasks contains the label set of ``[A, B, C, D, E, True, False]''.
 
\begin{table*}[h]
\caption{
    Dataset Statistics. {The size of the Anchor Set $\gD$ is used in \Tabref{tab:main-llama}, \ref{tab:main-ablation} and \ref{tab:app-ll}.}
}
\vspace{-1em}
\begin{center}
\resizebox{\linewidth}{!}{%
\setlength{\tabcolsep}{12pt}
\begin{tabular}{cccc ccccc}
	\toprule[0.12em]
     & WG-S
     & ARC-C 
     & ARC-E 
     & WG-M 
     & OBQA%
     & BoolQ 
     & Combined
     \\

     \midrule
    
     \multirow{1}{*}{Size of Label Space} 
     &  
     2 & 
     5 & 
     5 &  
     2 & 
     4 & 
     2 &
     7 \\

     \midrule

     \multirow{1}{*}{Size of Training Set} 
     & 
     640 & 
     1,119 & 
     2,251 &  
     2,258 & 
     4,957 & 
     9,427 & 
     20,652 \\

    \midrule

     \multirow{1}{*}{Size of Anchor Set $\gD$} 
     & 500 (78\%) 
     & 500 (45\%)
     & 500 (22\%)
     & 500 (22\%)
     & 500 (10\%)
     & 500 (5\%)
     & 500 (2\%)
     \\
     
     \midrule
     \multirow{1}{*}{Size of Test Set} 
     & 1,267
     & 299
     & 570
     & 1,267
     & 500
     & 3,270 
     & 7,173
     \\

    \bottomrule[0.12em]
    \end{tabular}

}
\end{center}
\label{tab:dataset}
\end{table*}

\subsection{Evaluation Metrics}
\label{app:implementation-evaluation}
Negative Log-Likelihood (\textbf{NLL}) and Expected Calibration Error (\textbf{ECE}~\cite{naeini2015obtaining}) are key metrics for uncertainty estimation. For a model $P_\vtheta$ and test dataset $\{\vx_n, y_n\}_{n=1}^N$, \textbf{NLL} penalizes models that assign low probabilities to correct labels, and is defined as:
\begin{equation}
    \text{NLL} = \frac{1}{N}\sum_{n=1}^{N} -\log P_\vtheta(y_n).
\end{equation}
\textbf{ECE} measures the alignment between model confidence and accuracy by binning predictions:
\begin{equation}
    \text{ECE} = \sum_{m=1}^{M} \frac{|B_m|}{n} \left| \text{acc}(B_m) - \text{conf}(B_m) \right|,
\end{equation}
where $\text{acc}(B_m) = \nicefrac{1}{|B_m|} \sum_{i \in B_m} \mathbbm{1}{(\hat{y}_i = y_i)}$ is the average accuracy and $\text{conf}(B_m)= \nicefrac{1}{|B_m|} \sum_{i \in B_m} P(\hat{y}_i)$ is the average confidence in bin $B_m$. We use bin size $|B_m| = 15$ throughout this paper.

\subsection{Searched $\sigma_q$ of \ours}
\label{app:implementation-optimal-sigma_q}
We report the searched $\sigma_q^*$ using \Algref{alg:tfb} in \Tabref{tab:optimal_sigma_q}, where the reported values are the mean values of three random seeds.

\begin{table*}[h!]
\caption{
    Searched $\sigma_q^*$ of \ours using \Algref{alg:tfb}.
}
\vspace{-1em}
\begin{center}
\resizebox{1\linewidth}{!}{%
\setlength{\tabcolsep}{16pt}
\begin{tabular}{cccc cccc}
	\toprule[0.12em]
     Base Model
     & WG-S
     & ARC-C 
     & ARC-E 
     & WG-M 
     & OBQA%
     & BoolQ 
     \\

     \midrule

     \multirow{1}{*}{\loramle} 

     & 0.004500  & 0.003917  & 0.004500  & 0.004354  & 0.003771  & 0.004063 
     \\
     \midrule

     \multirow{1}{*}{\loramap}

      & 0.004500  & 0.003479  & 0.003188  & 0.004208  & 0.003917  & 0.005083 
     \\

     \midrule

     \multirow{1}{*}{\blob-Mean}

     & 0.005813  & 0.005229  & 0.005229  & 0.006250  & 0.006250  & 0.005958 
     \\
     
    \bottomrule[0.12em]
    \end{tabular}
}
\end{center}
\label{tab:optimal_sigma_q}
\end{table*}

\subsection{Bayesianization (Training)}
\label{app:training-detail}
\textbf{Shared Configuration.}\quad 
We report the mean and standard deviation of all experimental results calculated over three random seeds. For all training processes in our experiments, we employ the AdamW optimizer. The learning rate follows a linear decay schedule with a warmup ratio of 0.06 and a maximum value of $2e-4$. The batch size is set to 4, and the maximum sentence length is limited to 300 tokens. 
The LoRA configuration includes LoRA $\alpha = 16$ and LoRA $r = 8$. 
PiSSA~\cite{meng2024pissa} follows the exact same configuration as the LoRA's. For VeRA~\cite{kopiczko2023vera}, due to its characteristic of shared weights across different layers which enables higher-rank setting with the same memory efficiency, we set its rank to $r=256$ and learning rate to $5e-3$ for the MLE training on the combined dataset.

\textbf{Baseline Configuration.}\quad 
The baseline configuration mainly follows \blob~\cite{wang2024blob}. 
\loramle follows the standard LoRA implementation. For \loramap, we implement it with a weight decay rate of $1e-5$. \loramcd consists of an ensemble of 10 LoRAs with a dropout rate of $p=0.1$. For \loraens, we fine-tune 3 LoRAs independently and combine them by averaging their logits during evaluation. We implement \loralap and apply it to the \loramap checkpoints. For BBB and \blob, we use the default settings from Bayesian-Torch library~\cite{krishnan2022bayesiantorch}, applying Bayesianization only to the $\mA$ matrix. During training, the number of samples is set to $K=1$ for both BBB and \blob. At test time, we use $N=10$ samples, matching the configuration of \ours.

\textbf{\ours Configuration.}\quad 
We randomly sample unlabeled training data points to construct the anchor dataset $\gD=\{\vx_i, \hat{y}_i\}_{i\in[M]}$ where $\hat{y}_i$ is the pseudo-label generated by the given LoRA adapter before Bayesianization; the anchor dataset size $M=500$ is fixed for all the datasets. 
We use NLL as the metric $l$ and set the performance change tolerance $\epsilon$ to 0.3\% of relative performance change for all the datasets.
To determine the optimal $\sigma_q^*$, we perform a 5-step binary search with the initial range of $[0.001, 0.015]$ using \Algref{alg:tfb}.
Similar to the other baseline methods, the final results of \ours are reported as averages across three random seeds using $\sigma_q^*$.

\section{Additional Experimental Results}
\label{app:more}
We present additional experimental results in this section. Due to space constraints (and large table size), we defer several detailed tables to the end of this section rather than presenting them alongside the corresponding analyses.

\begin{figure}[t]
    \centering
        \includegraphics[width=\linewidth]{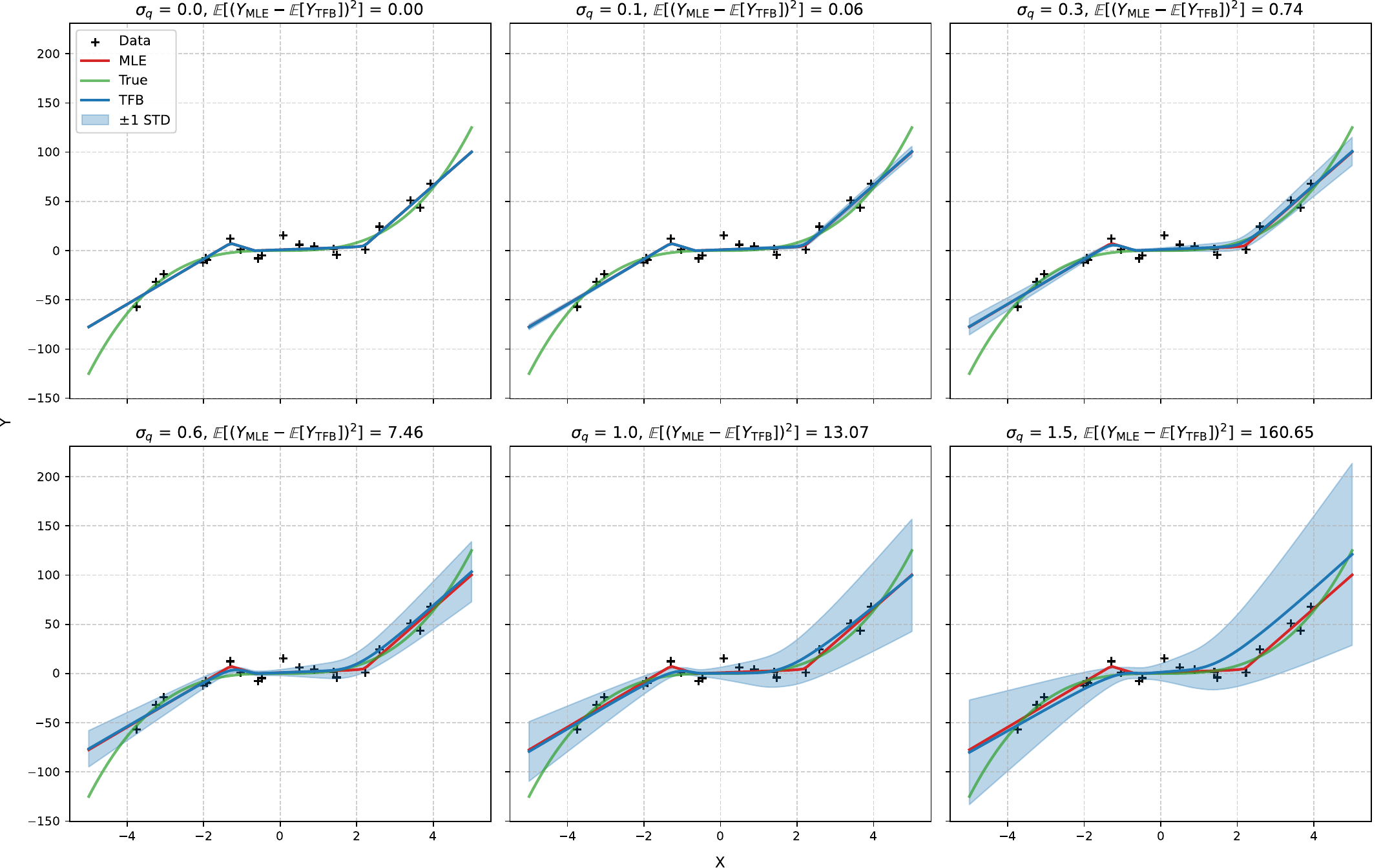}
    \caption{
        \textbf{Uncertainty estimation of \ours on a toy regression task.}  
        The true cubic function $y=x^3$ (\textbf{\textcolor{ForestGreen}{GREEN}}) and noisy training samples (\textbf{\textcolor{Gray}{GRAY}}) are shown alongside predictions from a deterministic MLP baseline (\textbf{\textcolor{Red}{RED}}) and our \ours  (\textbf{\textcolor{Blue}{BLUE}}). 
        The blue shaded region represents ±1 standard deviation from \ours  predictions with $\sigma_q=1.0$. 
        \ours  effectively captures predictive uncertainty, showing low variance in data-dense regions ($x \in [0,2]$) and increasing uncertainty for inputs outside the training distribution ($x \notin [-4,4]$), with uncertainty proportional to distance from the training domain.
    }
    \label{fig:visual}
\end{figure}

\subsection{\ours as a General Bayesian Neural Network: A Visual Study}
\label{app:visualization}
{
We demonstrate that \ours functions as a general Bayesian Neural Network (beyond its application to the LoRA adapters of LLMs) by evaluating its uncertainty estimation capabilities on a toy regression dataset, as illustrated in \Figref{fig:visual}. 
We follow the regression task framework~\cite{hernandez2015probabilistic,wang2016natural} with the following specifications:
\begin{itemize}
    \item \textbf{Input Features:} $\{\vx_i \sim U[-4,4]\}_{i\in [20]}$ uniformly sampled from the interval $[-4,4]$.
    \item \textbf{True Function:} $y=x^{3}$ (plotted in \textbf{\textcolor{ForestGreen}{GREEN}}). 
    \item \textbf{Noisy Labels:} $\vy_i=\vx_i^{3} + 9\cdot\epsilon_i$ where $\epsilon_i \sim \gN(0,1)$, representing the true function with added Gaussian noise. These data points appear as \textbf{\textcolor{Gray}{GRAY}} crosses.
    \item \textbf{MLE Baseline:} A two-layer MLP with 16 hidden neurons fit to the sampled dataset, providing a deterministic baseline without uncertainty quantification. The MLE predictions are shown in \textbf{\textcolor{Red}{RED}}.
    \item \textbf{MLE Training Configuration:} Adam optimizer~\cite{kingma2014adam} with learning rate 0.1, trained for 1000 steps until convergence.
    \item \textbf{\ours Implementation:} We Bayesianize the MLE baseline using \ours with full-rank noise, defining the variational distribution as $q(\mW|\mW_0, \sigma_q)=\prod_{ij}\gN(W_{ij}|W_{0,ij}, \sigma_q^2)$. For inference, we draw $N=10$ samples from the variational weight distribution and plot the average prediction as a \textbf{\textcolor{Blue}{BLUE}} curve. The shaded region represents $\pm1$ standard deviation of the sampled predictions. 
    We evaluate multiple settings of $\sigma_q\in [0.1, 0.3, 0.6, 1.0, 1.5]$ and report the average squared predictive difference between \ours and MLE in the figure. 
\end{itemize}
As \Figref{fig:visual} demonstrates, with an appropriately calibrated standard deviation of the approximate posterior ($\sigma_q=1.0$ in this case), \ours effectively captures predictive uncertainty: in regions with dense data sampling ($\vx\in[0,2]$), \ours produces low predictive uncertainty; conversely, for inputs outside the training domain ($\vx\notin[-4,4]$), the predictive uncertainty increases proportionally with distance from the training distribution. 
}

\subsection{{Test-Time Sample Size Analysis}}
\label{app:tts}

{Increasing the number of samples generally yields more accurate estimates of the expected output, thereby improving model performance in terms of ECE and NLL.
\textbf{We ran additional experiments and report \ours's performance for sample sizes ranging from $N=0$ (reduced to \blob-Mean) to $N=160$ on the WG-S dataset~\cite{wg} in \Figref{fig:tts}.} Each experiment is repeated three times with different random seeds on Llama3.1-8B.}

\begin{figure}[t]
    \centering
    \includegraphics[width=\linewidth]{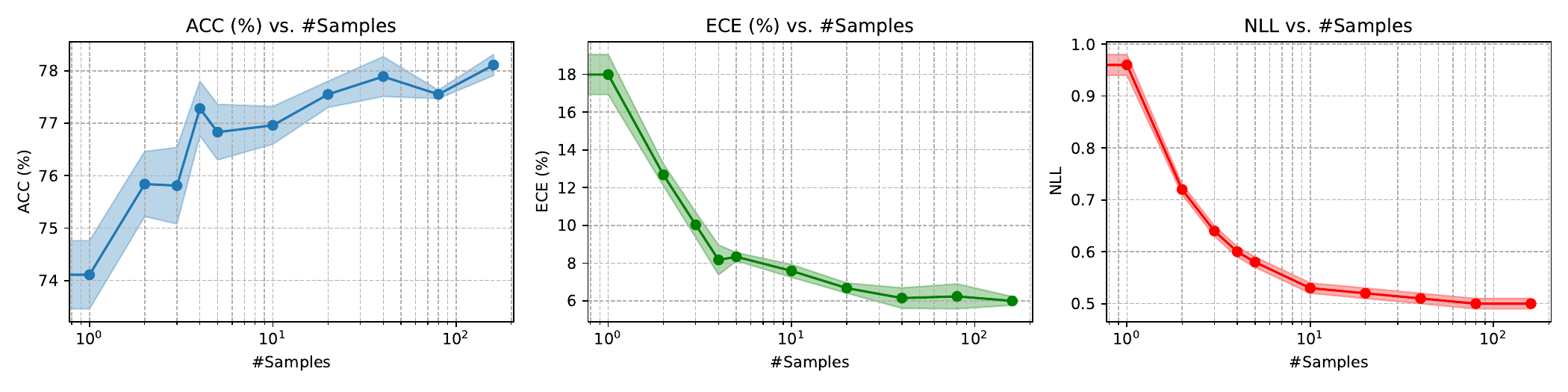}
    \caption{
        {\textbf{Performance of test-time scaling for \ours with varying numbers of samples $N$.}  
        The shaded region denotes the standard deviation across three random seeds.  
        While larger $N$ generally yields better performance, our choice of $N=10$ provides a favorable trade-off between accuracy and efficiency at test time.}
    }
    \vspace{-0em}
    \label{fig:tts}
\end{figure}

{As shown \Figref{fig:tts}, increasing the number of samples at test time generally improves the alignment between the Monte Carlo estimates and the expectations in theory, leading to better uncertainty estimation.
Our choice of $N=10$ samples in the paper balances performance and test-time efficiency: it achieves significant improvement in ECE and NLL while introducing acceptable computational overhead. The performance gap between this setting and extremely large $N$ is also mild.}

\subsection{Anchor Dataset Size Analysis}
\label{app:anchor-size-analysis}

{\textbf{We present the performance of BLoB-Mean + \ours on the ARC-Easy dataset with varying anchor dataset sizes ranging from 100 to 2000 in Figure~\ref{fig:anchor-size-analysis}.} Initially, we hypothesized a negative correlation between the performance variance and the anchor dataset size. However, as shown in the figure, across experiments with three different random seeds, neither the performance variance nor the average performance exhibits a significant correlation with the anchor dataset size. This suggests that \ours is robust to the size of the anchor dataset.}

\begin{figure}[h]
    \centering
        \includegraphics[width=\linewidth]{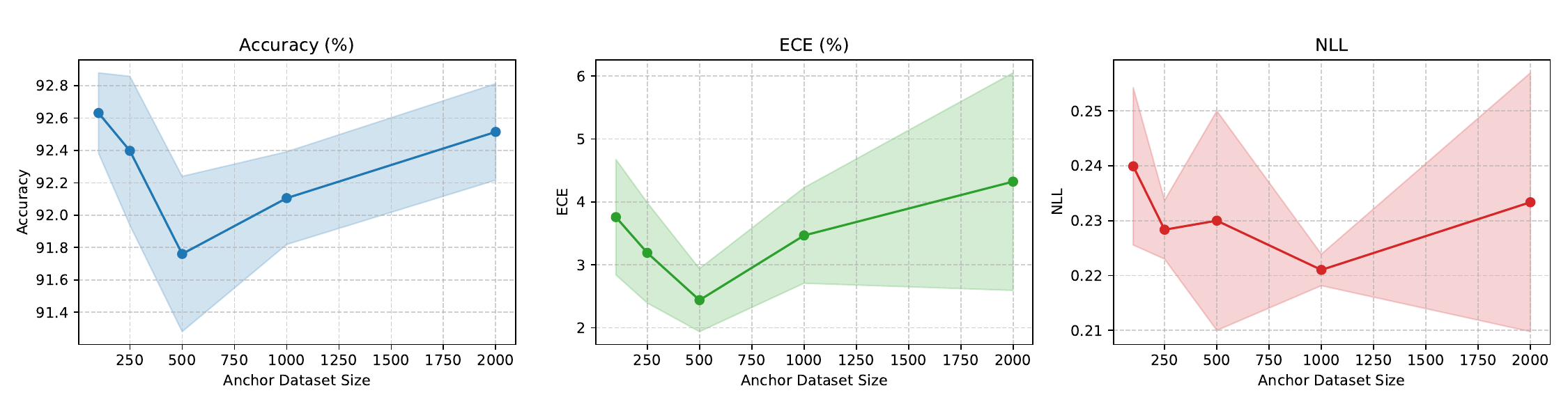}
    \caption{
        \textbf{Performance of BLoB-Mean + \ours with different size of anchor dataset on the ARC-Easy dataset.} The shaded area represents the standard deviation of results across three random seeds, indicating that \ours is not sensitive to the size of the anchor dataset.
    }
    \label{fig:anchor-size-analysis}
\end{figure}

\subsection{\ours Beyond Training Data as Anchor Dataset}
\label{app:more-eval-as-anchor}
In this section, we designate a portion of the unlabeled testing dataset as the anchor dataset, simulating a scenario where partial user input is accessible. The anchor dataset size remains fixed at 500 across all datasets. As illustrated in \Tabref{tab:app-eval-anchor}, the performance variation across different data sources for the anchor dataset is minimal, indicating that the choice of data source for the anchor dataset has negligible impact.

\subsection{\ours Beyond the NLL Metric} 
\label{app:more-acc-as-metric}
\textbf{We report the additional results of \ours when using Accuracy (ACC) as the evaluation metrics $l$ in \Tabref{tab:main-llama-acc}.} 
In our implementation, we adopt the change ratio of classification results as an accuracy-based evaluation metric ($l_{\text{ACC}}$) for the unsupervised anchor dataset.
Comparing the two evaluation metrics ($l_{\text{ACC}}$ vs $l_{\text{NLL}}$) in \Tabref{tab:main-llama-acc}, we observe comparable performance across all datasets. In some cases, accuracy-based evaluation ($l_{\text{ACC}}$) even yields slightly better results. For instance, BLoB-Mean+\ours achieves lower ECE on several datasets when using $l_{\text{ACC}}$. However, we adopt NLL as the primary evaluation metric in \Tabref{tab:main-llama} since it better aligns with our theoretical framework in \Thmref{thm:tfb}.

\subsection{\ours Beyond the Low-Rank Isotropic Gaussians}
\label{app:more-ablation-full}

In \Secref{sec:ablation}, we compare \ours with two alternative Gaussian distribution families that are controlled by a single parameter $\sigma_q$: 
\begin{itemize}
    \item Full-Rank Isotropic Gaussian~(\textbf{FR}): 
    given $\sigma_q$, the FR's variational distribution of the weight matrix $q(\vectorize(\mW))=\gN(\vectorize(\mW)|\vmu_q, \mSigma_q)$ where $\vmu_q=\mW_0+\mB\mA$ (same as \ours) and $\mSigma_q=\sigma_q^2\mI_{mn}$ is full-rank. 
    \item Constant Standard Deviation Matrix~(\textbf{C-STD}): 
    given $\sigma_q$, the C-STD's variational distribution of the weight matrix $q(\vectorize(\mW))=\gN(\vectorize(\mW)|\vmu_q, \mSigma_q)$ where $\vmu_q=\mW_0+\mB\mA$ (same as \ours) and $\mSigma_q= \sigma_q^2  \mI_n \otimes [\mB\mB^\top]$.
\end{itemize}
C-STD's covariance matrix $\mSigma_q$ is derived through \Lmmref{lemma:blob}:
\begin{align}
    \mSigma_q 
    &= [\mI_n \otimes \mB]  [\diag(\vectorize(\mOmega)^2)]  [\mI_n \otimes \mB^\top] \\
    &= [\mI_n \otimes \mB]  [\sigma_q^2  \mI_{rn}]  [\mI_n \otimes \mB^\top] \\
    &= \sigma_q^2   [\mI_n \otimes \mB]  [\mI_n \otimes \mB^\top] \\
    &= \sigma_q^2  \mI_n \otimes [\mB\mB^\top].
\end{align}
This depends on $\mB$ and thus varies for equivalent LoRA parameterizations ${\mB,\mA}$ of the same $\Delta\mW$.

\textbf{We report the additional results comparing \ours with other approximate families Gaussians (FR and C-STD as discussed in \Secref{sec:ablation}) when using Accuracy as the evaluation metrics $l$ in \Tabref{tab:main-ablation-cls}.}
When the evaluation metric is set to Accuracy, the advantage of \ours becomes more significant compared to the results shown in \Tabref{tab:main-ablation}.
\ours with low-rank isotropic Gaussian as the variational distribution demonstrates superior calibration performance compared to both \textbf{FR} and \textbf{C-STD} variants while maintaining competitive accuracy. For ECE, \ours achieves better results across most datasets, with notable improvements on in-distribution tasks: 8.78\% on WG-S (vs. 12.06\% for \textbf{FR} and 11.61\% for \textbf{C-STD}) and 1.28\% on BoolQ (vs. 3.26\% for \textbf{FR} and 2.65\% for \textbf{C-STD}). Similarly for NLL, \ours consistently outperforms or matches the baseline variants, particularly on WG-S (0.55 vs. 0.63 for \textbf{FR} and 0.61 for \textbf{C-STD}) while preserving comparable accuracy scores. These results suggest that \ours's approach to variance modeling is more effective than both full-rank isotropic and constant standard deviation alternatives.

\subsection{\ours Beyond the Llama3.1-8B Backbone}
\label{app:more-other-llms}

\textbf{We report the detailed performance of \ours applied to various LLM backbones in \Tabref{tab:llama-others-full}.} While the baseline MLE is typically trained for 2 epochs (shown with each backbone name), we also report results with reduced training (1 epoch) for comparison. Although training with fewer steps (early stopping) can effectively reduce model overconfidence, it typically leads to performance degradation.

The results demonstrate that \ours consistently improves model calibration across different backbones while maintaining competitive accuracy. Specifically, for Llama2-7B, \ours reduces the ECE from 4.50\% to 1.24\% on the combined dataset while preserving the accuracy (81.32\% vs 81.41\%). Similar improvements are observed with Llama3-8B, Llama3.1-8B, and Mistral-7B-v0.3, where \ours achieves better calibration than both the full training and early stopping baselines without sacrificing performance, suggesting its effectiveness as a general approach for enhancing LLM calibration.

\subsection{\ours Beyond the Naive LoRA}
\label{app:more-other-pefts}

\textbf{We report the detailed performance of \ours applied to various LoRA variants in \Tabref{tab:lora-others-full}.} The baseline models are trained for 2 epochs using pre-trained Llama3.1-8B on the concatenated dataset of six commonsense reasoning tasks. 
Specifically, we consider the two LoRA variants: 
\begin{itemize}
    \item VeRA~\cite{kopiczko2023vera}: Uses shared low-rank matrices $\mB$ and $\mA$ across layers, with layer-specific trainable scalar vector $\vd$ and bias vector $\vb$. Concretely, the parameterization of VeRA's updated weight matrix $\mW$ is modeled as:
        \begin{align}
            \mW &= \mW_0 + \Delta\mW = \mW_0 + [\diag(\vb)]\mB[\diag(\vd)]\mA.
        \end{align}
    Hence after the fine-tuning of VeRA, we can easily regroup the weight matrices into $\{\mB^\prime=[\diag(\vb)]\mB[\diag(\vd)], \mA^\prime=\mA\}$, and apply the \ours Bayesianization scheme illustrated in \Algref{alg:tfb}. 
    \item PiSSA~\cite{meng2024pissa}: Employs an alternative initialization scheme while maintaining LoRA's parameterization and training procedure. Hence the \ours process for PiSSA is trivial.  
\end{itemize}
The results in \Tabref{tab:lora-others-full} show that \ours consistently improves calibration across different LoRA variants while preserving model performance. Notably, when applied to the standard LoRA, \ours significantly reduces the ECE from 4.74\% to 1.05\% on the combined dataset with minimal impact on accuracy (86.45\% vs 86.70\%). Similar improvements are observed with VeRA and PiSSA variants, where \ours achieves better calibration (reducing ECE to 1.44\% and 1.17\% respectively) while maintaining comparable accuracy levels. These results demonstrate that \ours can effectively enhance model calibration across different LoRA architectures without compromising their performance.

\subsection{Improving the Inference-Time Efficiency of \ours}
\label{app:tfb-ll}

\textbf{We report the detailed performance of last-layer \ours (\texttt{LL} \ours) in \Tabref{tab:app-ll}}. 
As indicated in the table, with only $N=10$ samples, last-layer Bayesianization provides a less effective uncertainty estimation compared to full-model Bayesianization. However, increasing the number of samples to $N=100$ significantly enhances the posterior estimation, allowing last-layer Bayesianization to achieve better accuracy. This improvement further allows it to outperform the full-model Bayesianization in terms of NLL across most datasets.

\subsection{Additional Results on Llama2-7B}
\label{app:more-llama2}

\textbf{We report the detailed performance of \ours applied to the Llama2-7B pre-trained weights in \Tabref{tab:main-llama2}.} The performance change tolerance $\epsilon$ is set adaptively to either 1\% or 0.5\%, depending on the checkpoint's overfitting characteristics.
To determine the optimal $\sigma_q^*$, we conduct parallel experiments with eight values of $\sigma_q\in$ [0.01, 0.015, 0.02, 0.025, 0.03, 0.035, 0.04, 0.05] using a single random seed. 
We construct an approximate function $\hat{\sigma}_q(p)$ through piecewise linear interpolation of the observed performance and estimate $\sigma_q^*\approx\hat{\sigma}_q(p_0 - \epsilon)$.
Similar to other baseline methods, the final results of \ours are reported as averages across three random seeds using $\sigma_q^*$.

\textbf{In-Distribution (IND) Results.}\quad 
We observed several key patterns from the IND Datasets results. For example, the MLE baseline shows relatively strong accuracy but suffers from high ECE values (e.g., 29.83\% on WG-S), indicating significant overconfidence. This aligns with the common challenge of LLM overconfidence during conventional fine-tuning.

\ours applied to BLoB-Mean demonstrates strong overall performance across the IND datasets, achieving the highest accuracy on several datasets (69.94\% on WG-S, 70.72\% on ARC-C, and 86.74\% on ARC-E). More importantly, it achieves this while maintaining lower ECE values compared to methods like MCD and ENS, suggesting better calibrated predictions. The method also shows strong NLL performance, with values consistently among the lowest across datasets (0.62 for WG-S, 0.86 for ARC-C).

In summary, \ours consistently enhances the performance of baseline methods (MLE, MAP, and BLoB-Mean) across different evaluation scenarios, with notable improvements in both accuracy and calibration metrics. The improvements are particularly evident in the significant ECE reductions (e.g., from 29.83\% to 16.26\% for MLE on WG-S) while maintaining or improving accuracy, with the most substantial gains observed when \ours is combined with BLoB-Mean, achieving both the highest accuracy and lowest ECE values across most datasets.

\textbf{Out-of-Distribution (OOD) Results.}\quad
The OOD evaluation reveals interesting patterns across both smaller and larger distribution shifts. For smaller shifts (ARC-C and ARC-E), BLoB-Mean with \ours maintains strong performance, achieving 70.38\% and 80.16\% accuracy respectively, while keeping ECE values low (12.28\% and 8.07\%). This suggests robust generalization under moderate distribution shifts.

For larger shifts (Chem and Phy datasets), we see a more significant performance degradation across all methods, as expected. However, BLoB-Mean with \ours still maintains competitive performance, achieving 42.67\% accuracy on Chem and 30.67\% on Phy, while maintaining reasonable calibration metrics. The method's NLL values (1.35 and 1.46 respectively) remain competitive with other approaches, indicating relatively well-calibrated uncertainty estimates even under substantial distribution shifts.

Notable is the consistently strong performance of the BLoB variants (both w/ and w/o \ours) across different metrics and datasets, suggesting that this approach offers a robust framework for both in-distribution and out-of-distribution scenarios. The results demonstrate that the method successfully balances the trade-off between accuracy and calibration, particularly evident in the out-of-distribution scenarios where maintaining both aspects becomes more challenging.

\renewcommand{\thefootnote}{\fnsymbol{footnote}}
\begin{table*}[t]
\caption{
    \textbf{Performance of different methods applied to LoRA on Llama3.1-8B pre-trained weights,} where Accuracy~(\textbf{ACC}) and Expected Calibration Error~(\textbf{ECE}) are reported in percentages. 
    \textbf{``TF?''} denotes whether a method is \textbf{T}raining-\textbf{F}ree.
    The evaluation is done across six common-sense reasoning tasks with a shared hyper-parameter setting after fine-tuning of 5 epochs .
    We sample $N=10$ during inference in all sampling-based methods including \textbf{\blob~\cite{wang2024blob}} and \textbf{\ours}. 
    \hl{Rows with shading} indicate training-free Bayesianization methods that use a pre-trained LoRA as their mean. 
    For \ours, we randomly sample a subset of the training data without labels as the anchor dataset $\gD$. For accuracy-based evaluation~($l_{\text{ACC}}$), we set the performance drop tolerance to $\epsilon=1\%$. For NLL loss~($l_{\text{NLL}}$), we use the same settings as in \Tabref{tab:main-llama}.
    ``$\uparrow$'' and ``$\downarrow$'' indicate that higher and lower values are preferred, respectively. 
    \textbf{Boldface} and \underline{underlining} denote the best and the second-best performance, respectively.
}
\vspace{0em}
\begin{center}
\resizebox{1\linewidth}{!}{%
\setlength{\tabcolsep}{3pt}
\begin{tabular}{clcccc ccc cccc}
	\toprule[0.12em]
	\multirow{3}{*}[-0.25em]{\textbf{Metric}} 
    & \multirow{3}{*}[-0.25em]{\textbf{Method}} 
    & \multirow{3}{*}[-0.25em]{\textbf{TF?}}
    & \multicolumn{6}{c}{\multirow{2}{*}[-0.25em]{\textbf{In-Distribution Datasets}}} 
    & \multicolumn{4}{c}{\textbf{Out-of-Distribution Datasets} (OBQA$\rightarrow$X)}
     \\
     \cmidrule(lr){10-13}
     
     & & &
     & & & & & 
     & \multicolumn{2}{c}{\emph{Small Shift}}
     & \multicolumn{2}{c}{\emph{Large Shift}}
     \\
     
     \cmidrule(lr){4-9} \cmidrule(lr){10-11} \cmidrule(lr){12-13}
     & & & WG-S %
     & ARC-C%
     & ARC-E%
     & WG-M%
     & OBQA
     & BoolQ%
     & ARC-C%
     & ARC-E%
     & Chem
     & Phy
     \\
     \midrule

     \multirow{14}{*}{ACC~($\uparrow$)} 
     & \loramcd
     & \xmark 
     & \underline{78.03\scriptsize{$\pm$0.61}}  & 81.64\scriptsize{$\pm$1.79}  & 91.37\scriptsize{$\pm$0.38}  & 83.18\scriptsize{$\pm$0.84}  & 87.20\scriptsize{$\pm$1.02}  & 89.93\scriptsize{$\pm$0.16}  & 81.42\scriptsize{$\pm$1.38}  & 87.27\scriptsize{$\pm$0.84}  & 47.92\scriptsize{$\pm$2.25}  & \textbf{46.53\scriptsize{$\pm$0.49}}
     \\
     
     & \loraens
     & \xmark 
     & \textbf{78.82\scriptsize{$\pm$0.52}}  & 82.55\scriptsize{$\pm$0.42}  & \textbf{91.84\scriptsize{$\pm$0.36}}  & \textbf{83.99\scriptsize{$\pm$0.74}}  & 87.37\scriptsize{$\pm$0.67}  & \textbf{90.50\scriptsize{$\pm$0.14}}  & 79.62\scriptsize{$\pm$0.57}  & 86.56\scriptsize{$\pm$0.60}  & \textbf{49.65\scriptsize{$\pm$3.22}}  & 44.44\scriptsize{$\pm$1.96}
      \\
\end{tabular}
 }
 
\end{center}
\label{tab:main-llama-acc}
\vspace{0em}
\end{table*}

\renewcommand{\thefootnote}{\fnsymbol{footnote}}
\begin{table*}[t]
\caption{
    \textbf{Performance of different methods applied to LoRA on Llama3.1-8B pre-trained weights,} where Accuracy~(\textbf{ACC}) and Expected Calibration Error~(\textbf{ECE}) are reported in percentages. 
    ($\gD_{\text{train}}$) and ($\gD_{\text{test}}$) denote the anchor dataset is randomly sampled from the training dataset and the testing dataset, respectively.
    \textbf{``TF?''} denotes whether a method is \textbf{T}raining-\textbf{F}ree.
    The evaluation is done across six common-sense reasoning tasks with a shared hyper-parameter setting after fine-tuning of 5 epochs .
    We sample $N=10$ during inference in all sampling-based methods including \textbf{\blob~\cite{wang2024blob}} and \textbf{\ours}. 
    \hl{Rows with shading} indicate training-free Bayesianization methods that use a pre-trained LoRA as their mean. 
    ``$\uparrow$'' and ``$\downarrow$'' indicate that higher and lower values are preferred, respectively. 
    \textbf{Boldface} and \underline{underlining} denote the best and the second-best performance, respectively. 
}
\vspace{0em}
\begin{center}
\resizebox{1\linewidth}{!}{%
\setlength{\tabcolsep}{3pt}
\begin{tabular}{clcccc ccc cccc}
	\toprule[0.12em]
	\multirow{3}{*}[-0.25em]{\textbf{Metric}} 
    & \multirow{3}{*}[-0.25em]{\textbf{Method}} 
    & \multirow{3}{*}[-0.25em]{\textbf{TF?}}
    & \multicolumn{6}{c}{\multirow{2}{*}[-0.25em]{\textbf{In-Distribution Datasets}}} 
    & \multicolumn{4}{c}{\textbf{Out-of-Distribution Datasets} (OBQA$\rightarrow$X)}
     \\
     \cmidrule(lr){10-13}
     
     & & &
     & & & & & 
     & \multicolumn{2}{c}{\emph{Small Shift}}
     & \multicolumn{2}{c}{\emph{Large Shift}}
     \\
     
     \cmidrule(lr){4-9} \cmidrule(lr){10-11} \cmidrule(lr){12-13}
     & & & WG-S %
     & ARC-C%
     & ARC-E%
     & WG-M%
     & OBQA
     & BoolQ%
     & ARC-C%
     & ARC-E%
     & Chem
     & Phy
     \\
     \midrule

     \multirow{14}{*}{ACC~($\uparrow$)} 
     & \loramcd
     & \xmark 
     & \underline{78.03\scriptsize{$\pm$0.61}}  & 81.64\scriptsize{$\pm$1.79}  & 91.37\scriptsize{$\pm$0.38}  & 83.18\scriptsize{$\pm$0.84}  & 87.20\scriptsize{$\pm$1.02}  & 89.93\scriptsize{$\pm$0.16}  & 81.42\scriptsize{$\pm$1.38}  & 87.27\scriptsize{$\pm$0.84}  & 47.92\scriptsize{$\pm$2.25}  & \textbf{46.53\scriptsize{$\pm$0.49}}
     \\
     
     & \loraens
     & \xmark 
     & \textbf{78.82\scriptsize{$\pm$0.52}}  & 82.55\scriptsize{$\pm$0.42}  & {91.84\scriptsize{$\pm$0.36}}  & \textbf{83.99\scriptsize{$\pm$0.74}}  & 87.37\scriptsize{$\pm$0.67}  & \textbf{90.50\scriptsize{$\pm$0.14}}  & 79.62\scriptsize{$\pm$0.57}  & 86.56\scriptsize{$\pm$0.60}  & \textbf{49.65\scriptsize{$\pm$3.22}}  & 44.44\scriptsize{$\pm$1.96}
      \\
\end{tabular}
 }
\end{center}
\label{tab:app-eval-anchor}
\vspace{0em}
\end{table*}

\begin{table*}[t]
\caption{
    \textbf{Performance of \ours with accuracy-based evaluation metric~($l$=ACC) using different posterior families applied to the mean of \blob, based on Llama3.1-8B pre-trained weights.} 
    \textbf{FR:} Full-rank isotropic Gaussian noises are applied to $\Delta\mW$; 
    \textbf{C-STD:} Standard deviation matrix $\mOmega=[\Omega_{ij}=\sigma_q]$ is constant. 
    The evaluation protocol strictly follows \Tabref{tab:main-llama}.
    ``$\uparrow$'' and ``$\downarrow$'' indicate that higher and lower values are preferred, respectively. 
    \textbf{Boldface} and \underline{underlining} denote the best and the second-best performance, respectively (only for \ours variants).
}
\vspace{-1em}
\begin{center}
\resizebox{1\linewidth}{!}{%
\setlength{\tabcolsep}{4pt}

\begin{tabular}{clcccccc  cccc}
	\toprule[0.12em]
	\multirow{3}{*}[-0.25em]{\textbf{Metric}} & \multirow{3}{*}[-0.25em]{\textbf{Method}} & \multicolumn{6}{c}{\multirow{2}{*}[-0.25em]{\textbf{In-Distribution Datasets}}} & \multicolumn{4}{c}{\textbf{Out-of-Distribution Datasets} (OBQA$\rightarrow$X)}
     \\
     \cmidrule(lr){9-12}
     
     & & 
     & & & & & 
     & \multicolumn{2}{c}{\emph{Small Shift}}
     & \multicolumn{2}{c}{\emph{Large Shift}}
     \\
     
     \cmidrule(lr){3-8} \cmidrule(lr){9-10} \cmidrule(lr){11-12}
     & & WG-S %
     & ARC-C%
     & ARC-E%
     & WG-M%
     & OBQA
     & BoolQ%
     & ARC-C%
     & ARC-E%
     & Chem
     & Phy
     \\
     \midrule

     \multirow{4}{*}{ACC~($\uparrow$)} 
     
     & \blob-Mean
      & \underline{77.72\scriptsize{$\pm$0.12}}  & 82.60\scriptsize{$\pm$0.60}  & \underline{91.64\scriptsize{$\pm$0.55}}  & \textbf{83.92\scriptsize{$\pm$0.48}}  & 88.00\scriptsize{$\pm$0.80}  & \underline{89.86\scriptsize{$\pm$0.05}}   & {82.06\scriptsize{$\pm$1.15}}  & \textbf{88.54\scriptsize{$\pm$0.31}}  & 39.93\scriptsize{$\pm$5.20}  & {39.93\scriptsize{$\pm$4.02}}
     \\

     & \cellcolor{lightergray} \quad + \ours~(FR)& \cellcolor{lightergray}76.32\scriptsize{$\pm$0.45}  & \cellcolor{lightergray}82.13\scriptsize{$\pm$0.38}  & \cellcolor{lightergray}91.82\scriptsize{$\pm$0.65}  & \cellcolor{lightergray}83.33\scriptsize{$\pm$0.85}  & \cellcolor{lightergray}\textbf{88.40\scriptsize{$\pm$0.16}}  & \cellcolor{lightergray}90.03\scriptsize{$\pm$0.18}  & \cellcolor{lightergray}82.27\scriptsize{$\pm$1.05}  & \cellcolor{lightergray}\underline{88.24\scriptsize{$\pm$0.37}}  & \cellcolor{lightergray}\textbf{43.33\scriptsize{$\pm$5.79}}  & \cellcolor{lightergray}\underline{41.33\scriptsize{$\pm$3.30} }
 
     \\

     & \cellcolor{lightergray} \quad + \ours~(C-STD) & \cellcolor{lightergray}\textbf{78.40\scriptsize{$\pm$0.36}}  & \cellcolor{lightergray}\underline{82.75\scriptsize{$\pm$0.16}}  & \cellcolor{lightergray}\textbf{92.38\scriptsize{$\pm$0.08}}  & \cellcolor{lightergray}\underline{83.44\scriptsize{$\pm$0.46}}  & \cellcolor{lightergray}\underline{88.37\scriptsize{$\pm$0.50}}  & \cellcolor{lightergray}\textbf{90.21\scriptsize{$\pm$0.04}}  & \cellcolor{lightergray}\underline{82.38\scriptsize{$\pm$0.89}}  & \cellcolor{lightergray}87.98\scriptsize{$\pm$0.30}  & \cellcolor{lightergray}\underline{40.00\scriptsize{$\pm$4.68}}  & \cellcolor{lightergray}\textbf{41.67\scriptsize{$\pm$4.12}}
	 
     \\
  
     & \cellcolor{lightergray} \quad + \ours~(Final)
	 & \cellcolor{lightergray}75.28\scriptsize{$\pm$0.33}  & \cellcolor{lightergray}82.80\scriptsize{$\pm$0.33}  & \cellcolor{lightergray}91.64\scriptsize{$\pm$0.15}  & \cellcolor{lightergray}81.84\scriptsize{$\pm$0.74}  & \cellcolor{lightergray}88.00\scriptsize{$\pm$0.16}  & \cellcolor{lightergray}89.60\scriptsize{$\pm$0.35} & \cellcolor{lightergray}\textbf{82.93\scriptsize{$\pm$1.91}}  & \cellcolor{lightergray}86.97\scriptsize{$\pm$0.62}  & \cellcolor{lightergray}36.00\scriptsize{$\pm$5.66}  & \cellcolor{lightergray}36.00\scriptsize{$\pm$5.72} 
     \\

     \midrule
    
     \multirow{4}{*}{ECE~($\downarrow$)}

     & \blob-Mean
     & 15.43\scriptsize{$\pm$0.15}  & 12.41\scriptsize{$\pm$1.52}  & 4.91\scriptsize{$\pm$0.28}  & 9.37\scriptsize{$\pm$1.33}  & 6.44\scriptsize{$\pm$0.15}  & 6.26\scriptsize{$\pm$0.29}   & 11.22\scriptsize{$\pm$0.38}  & 6.34\scriptsize{$\pm$0.71}  & 26.65\scriptsize{$\pm$3.06}  & 25.40\scriptsize{$\pm$5.40} \\

     & \cellcolor{lightergray} \quad + \ours~(FR)& \cellcolor{lightergray}\underline{11.49\scriptsize{$\pm$0.35}}  & \cellcolor{lightergray}\textbf{5.74\scriptsize{$\pm$1.24}}  & \cellcolor{lightergray}\textbf{3.21\scriptsize{$\pm$0.43}}  & \cellcolor{lightergray}\textbf{6.43\scriptsize{$\pm$0.70}}  & \cellcolor{lightergray}\textbf{4.22\scriptsize{$\pm$0.14}}  & \cellcolor{lightergray}\underline{4.90\scriptsize{$\pm$0.49}} & \cellcolor{lightergray}\underline{8.85\scriptsize{$\pm$1.74}}  & \cellcolor{lightergray}\textbf{4.18\scriptsize{$\pm$0.72}}  & \cellcolor{lightergray}\underline{17.63\scriptsize{$\pm$4.29}}  & \cellcolor{lightergray}\underline{22.31\scriptsize{$\pm$5.67}} 

     \\

     & \cellcolor{lightergray} \quad + \ours~(C-STD) & \cellcolor{lightergray}14.07\scriptsize{$\pm$0.22}  & \cellcolor{lightergray}9.85\scriptsize{$\pm$1.48}  & \cellcolor{lightergray}\underline{4.17\scriptsize{$\pm$0.73}}  & \cellcolor{lightergray}8.94\scriptsize{$\pm$0.50}  & \cellcolor{lightergray}5.48\scriptsize{$\pm$0.50}  & \cellcolor{lightergray}5.72\scriptsize{$\pm$0.28} & \cellcolor{lightergray}9.29\scriptsize{$\pm$0.46}  & \cellcolor{lightergray}6.13\scriptsize{$\pm$0.41}  & \cellcolor{lightergray}24.78\scriptsize{$\pm$3.57}  & \cellcolor{lightergray}24.51\scriptsize{$\pm$2.65} 
	 
     \\

     & \cellcolor{lightergray} \quad + \ours~(Final)
	  & \cellcolor{lightergray}\textbf{3.04\scriptsize{$\pm$0.12}}  & \cellcolor{lightergray}\underline{6.76\scriptsize{$\pm$1.47}}  & \cellcolor{lightergray}4.81\scriptsize{$\pm$1.18}  & \cellcolor{lightergray}\underline{7.42\scriptsize{$\pm$1.24}}  & \cellcolor{lightergray}\underline{5.26\scriptsize{$\pm$0.71}}  & \cellcolor{lightergray}\textbf{3.22\scriptsize{$\pm$0.36}} & \cellcolor{lightergray}\textbf{6.55\scriptsize{$\pm$1.04}}  & \cellcolor{lightergray}\underline{5.54\scriptsize{$\pm$1.33}}  & \cellcolor{lightergray}\textbf{17.00\scriptsize{$\pm$4.71}}  & \cellcolor{lightergray}\textbf{16.65\scriptsize{$\pm$4.33} }
     \\

     \midrule
    
     \multirow{4}{*}{NLL~($\downarrow$)} 

     & \blob-Mean
    & 0.74\scriptsize{$\pm$0.02}  & 0.73\scriptsize{$\pm$0.04}  & 0.29\scriptsize{$\pm$0.03}  & 0.47\scriptsize{$\pm$0.03}  & 0.37\scriptsize{$\pm$0.02}  & 0.32\scriptsize{$\pm$0.02}   & 0.67\scriptsize{$\pm$0.07}  & 0.39\scriptsize{$\pm$0.03}  & 1.53\scriptsize{$\pm$0.13}  & 1.54\scriptsize{$\pm$0.15} 
     \\

     & \cellcolor{lightergray} \quad + \ours~(FR)& \cellcolor{lightergray}\underline{0.61\scriptsize{$\pm$0.00}}  & \cellcolor{lightergray}\textbf{0.49\scriptsize{$\pm$0.02}}  & \cellcolor{lightergray}\textbf{0.23\scriptsize{$\pm$0.03}}  & \cellcolor{lightergray}\textbf{0.43\scriptsize{$\pm$0.02} } & \cellcolor{lightergray}\textbf{0.32\scriptsize{$\pm$0.01}}  & \cellcolor{lightergray}\underline{0.29\scriptsize{$\pm$0.02}}  & \cellcolor{lightergray}\underline{0.61\scriptsize{$\pm$0.04}}  & \cellcolor{lightergray}\textbf{0.35\scriptsize{$\pm$0.02}}  & \cellcolor{lightergray}\textbf{1.36\scriptsize{$\pm$0.11}}  & \cellcolor{lightergray}1.51\scriptsize{$\pm$0.10} 
 
     \\

     & \cellcolor{lightergray} \quad + \ours~(C-STD)  & \cellcolor{lightergray}0.69\scriptsize{$\pm$0.01}  & \cellcolor{lightergray}0.59\scriptsize{$\pm$0.05}  & \cellcolor{lightergray}0.25\scriptsize{$\pm$0.01}  & \cellcolor{lightergray}0.46\scriptsize{$\pm$0.02}  & \cellcolor{lightergray}\underline{0.34\scriptsize{$\pm$0.01}}  & \cellcolor{lightergray}0.31\scriptsize{$\pm$0.01} & \cellcolor{lightergray}0.68\scriptsize{$\pm$0.05}  & \cellcolor{lightergray}\underline{0.37\scriptsize{$\pm$0.02}}  & \cellcolor{lightergray}1.49\scriptsize{$\pm$0.13}  & \cellcolor{lightergray}\underline{1.45\scriptsize{$\pm$0.13}}
	 
     \\

     & \cellcolor{lightergray} \quad + \ours~(Final)
	& \cellcolor{lightergray}\textbf{0.52\scriptsize{$\pm$0.01}}  & \cellcolor{lightergray}\underline{0.52\scriptsize{$\pm$0.03}}  & \cellcolor{lightergray}\underline{0.24\scriptsize{$\pm$0.01}}  & \cellcolor{lightergray}\underline{0.45\scriptsize{$\pm$0.01}}  & \cellcolor{lightergray}0.35\scriptsize{$\pm$0.01}  & \cellcolor{lightergray}\textbf{0.28\scriptsize{$\pm$0.00}}   & \cellcolor{lightergray}\textbf{0.55\scriptsize{$\pm$0.02}}  & \cellcolor{lightergray}\underline{0.37\scriptsize{$\pm$0.01}}  & \cellcolor{lightergray}\underline{1.38\scriptsize{$\pm$0.10}}  & \cellcolor{lightergray}\textbf{1.41\scriptsize{$\pm$0.09}} 
     \\

    \bottomrule[0.12em]
    \end{tabular}
 }
\end{center}
\label{tab:main-ablation-cls}
\vspace{-1em}
\end{table*}

\begin{table*}[t]
\caption{
    \textbf{\ours Performances with various LLM backbones~\cite{touvron2023llama,touvron2023llama2,dubey2024llama,jiang2023mistral},} where Accuracy~(\textbf{ACC}) and Expected Calibration Error~(\textbf{ECE}) are reported in percentages. 
    The MLE training for each different backbone is conducted for \textbf{{2 epochs}} on the concatenated dataset of six commonsense reasoning tasks, with a shared hyperparameter setting; 
    ``Fewer Epochs'' represents training for \textbf{1 epoch}. 
    We use $N=10$ samples for \ours during inference and 
    \hl{rows with shading} indicate training-free Bayesianization methods that use a pre-trained LoRA as their mean. 
    ``$\uparrow$'' and ``$\downarrow$'' indicate that higher and lower values are preferred, respectively. 
    \textbf{Boldface} and \underline{underlining} denote the best and the second-best performance, respectively. 
}
\vspace{-1em}
\begin{center}
\resizebox{1\linewidth}{!}{%
\setlength{\tabcolsep}{5pt}
\begin{tabular}{clccc cccc}
	\toprule[0.12em]
	\multirow{2}{*}[-0.25em]{\textbf{Metric}} & \multirow{2}{*}[-0.25em]{\textbf{Method}} & \multicolumn{7}{c}{\textbf{Datasets}}
     \\
     \cmidrule{3-9}
     & & WG-S
     & ARC-C 
     & ARC-E 
     & WG-M 
     & OBQA%
     & BoolQ
     & Combined
     \\
     \midrule
     \multirow{12}{*}{ACC~($\uparrow$)} 
     
     & Llama2-7B
	 & \textbf{72.30\scriptsize{$\pm$0.90}}
	 & \underline{73.24\scriptsize{$\pm$1.34}}
	 & \textbf{87.66\scriptsize{$\pm$0.81}}
	 & \underline{72.30\scriptsize{$\pm$0.90}}
	 & \underline{83.27\scriptsize{$\pm$1.53}}
	 & \textbf{87.84\scriptsize{$\pm$0.57}}
	 & \textbf{81.41\scriptsize{$\pm$0.64}}
     \\

     & \cellcolor{white} \quad + Fewer Epochs
     & 63.85\scriptsize{$\pm$3.68}
	 & 69.23\scriptsize{$\pm$0.33}
	 & 86.73\scriptsize{$\pm$0.97}
	 & 63.85\scriptsize{$\pm$3.68}
	 & 79.67\scriptsize{$\pm$1.55}
	 & 86.08\scriptsize{$\pm$0.23}
     & 76.88\scriptsize{$\pm$1.11}
     \\
     
     & \cellcolor{lightergray} \quad + \ours~(Ours)
	 & \cellcolor{lightergray}\underline{72.03\scriptsize{$\pm$0.88}}
	 & \cellcolor{lightergray}\textbf{74.36\scriptsize{$\pm$1.58}}
	 & \cellcolor{lightergray}\underline{87.31\scriptsize{$\pm$1.14}}
	 & \cellcolor{lightergray}\textbf{72.85\scriptsize{$\pm$0.96}}
	 & \cellcolor{lightergray}\textbf{83.73\scriptsize{$\pm$0.70}}
	 & \cellcolor{lightergray}\underline{87.44\scriptsize{$\pm$0.34}}
	 & \cellcolor{lightergray}\underline{81.32\scriptsize{$\pm$0.51}}
     \\
     
     \cmidrule{2-9}
    
     & Llama3-8B
	 & \textbf{81.45\scriptsize{$\pm$0.00}}
	 & \textbf{84.95\scriptsize{$\pm$1.53}}
	 & \underline{92.63\scriptsize{$\pm$0.93}}
	 & \textbf{ 81.45\scriptsize{$\pm$0.00}}
	 & \textbf{88.20\scriptsize{$\pm$0.53}}
	 & \textbf{90.19\scriptsize{$\pm$0.13}}
	 & \textbf{86.93\scriptsize{$\pm$0.09}}
     \\

     & \cellcolor{white} \quad + Fewer Epochs
	 & 79.08\scriptsize{$\pm$1.18}
	 & 82.72\scriptsize{$\pm$0.39}
	 & 92.22\scriptsize{$\pm$0.83}
	 & 79.08\scriptsize{$\pm$1.18}
	 & 86.07\scriptsize{$\pm$0.90}
	 & 79.94\scriptsize{$\pm$14.97}
	 & 82.01\scriptsize{$\pm$5.76}
     \\
     
     & \cellcolor{lightergray} \quad + \ours~(Ours)
	 & \cellcolor{lightergray}\underline{81.19\scriptsize{$\pm$0.51}}
	 & \cellcolor{lightergray}\underline{84.73\scriptsize{$\pm$1.68}}
	 & \cellcolor{lightergray}\textbf{92.98\scriptsize{$\pm$0.80}}
	 & \cellcolor{lightergray}\underline{81.11\scriptsize{$\pm$0.55}}
	 & \cellcolor{lightergray}\underline{87.73\scriptsize{$\pm$0.64}}
	 & \cellcolor{lightergray}\underline{89.75\scriptsize{$\pm$0.10}}
	 & \cellcolor{lightergray}\underline{86.61\scriptsize{$\pm$0.20}}
     \\
     
     \cmidrule{2-9}

     & Llama3.1-8B
	 & \textbf{81.24\scriptsize{$\pm$0.05}}
	 & \underline{82.72\scriptsize{$\pm$0.19}}
	 & \textbf{92.11\scriptsize{$\pm$1.05}}
	 & \textbf{81.24\scriptsize{$\pm$0.05}}
	 & \textbf{87.80\scriptsize{$\pm$2.03}}
	 & \textbf{90.20\scriptsize{$\pm$0.11}}
	 & \textbf{86.70\scriptsize{$\pm$0.08}}
     \\

     & \cellcolor{white} \quad + Fewer Epochs
	 & 78.11\scriptsize{$\pm$0.12}
	 & \textbf{83.95\scriptsize{$\pm$1.00}}
	 & 91.17\scriptsize{$\pm$1.17}
	 & 78.11\scriptsize{$\pm$0.12}
	 & 85.33\scriptsize{$\pm$0.90}
	 & 89.38\scriptsize{$\pm$0.35}
	 & 84.96\scriptsize{$\pm$0.22}
     \\
     
     & \cellcolor{lightergray} \quad + \ours~(Ours)
	 & \cellcolor{lightergray}\underline{80.66\scriptsize{$\pm$0.70}}
	 & \cellcolor{lightergray}{82.50\scriptsize{$\pm$0.84}}
	 & \cellcolor{lightergray}\underline{91.93\scriptsize{$\pm$1.05}}
	 & \cellcolor{lightergray}\underline{81.22\scriptsize{$\pm$0.83}}
	 & \cellcolor{lightergray}\underline{87.73\scriptsize{$\pm$1.29}}
	 & \cellcolor{lightergray}\underline{89.96\scriptsize{$\pm$0.23}}
	 & \cellcolor{lightergray}\underline{86.45\scriptsize{$\pm$0.33}}
     \\
     
     \cmidrule{2-9}

     & Mistral-7B-v0.3
	 & \textbf{82.45\scriptsize{$\pm$0.82}}
	 & \textbf{84.28\scriptsize{$\pm$1.53}}
	 & {90.94\scriptsize{$\pm$0.27}}
	 & \textbf{82.45\scriptsize{$\pm$0.82}}
	 & \underline{87.73\scriptsize{$\pm$0.31}}
	 & \textbf{89.71\scriptsize{$\pm$0.48}}
	 & \textbf{86.88\scriptsize{$\pm$0.51}}
     \\

     & \cellcolor{white} \quad + Fewer Epochs
	 & 79.72\scriptsize{$\pm$0.00}
	 & 83.95\scriptsize{$\pm$0.33}
	 & \textbf{91.58\scriptsize{$\pm$0.63}}
	 & 79.72\scriptsize{$\pm$0.00}
	 & 87.53\scriptsize{$\pm$0.31}
	 & 89.20\scriptsize{$\pm$0.20}
	 & 85.71\scriptsize{$\pm$0.11}
     \\
     
     & \cellcolor{lightergray} \quad + \ours~(Ours)
	 & \cellcolor{lightergray}\underline{81.74\scriptsize{$\pm$0.43}}
	 & \cellcolor{lightergray}\underline{84.06\scriptsize{$\pm$1.68}}
	 & \cellcolor{lightergray}\underline{90.99\scriptsize{$\pm$0.73}}
	 & \cellcolor{lightergray}\underline{81.74\scriptsize{$\pm$0.75}}
	 & \cellcolor{lightergray}\textbf{87.93\scriptsize{$\pm$0.42}}
	 & \cellcolor{lightergray}\underline{89.71\scriptsize{$\pm$0.32}}
	 & \cellcolor{lightergray}\underline{86.64\scriptsize{$\pm$0.28}}
     \\ 
     
     \midrule
    
     \multirow{12}{*}{ECE~($\downarrow$)} 
     & Llama2-7B
	 & 9.17\scriptsize{$\pm$0.74}
	 & 9.37\scriptsize{$\pm$1.27}
	 & \textbf{2.65\scriptsize{$\pm$0.16}}
	 & 9.17\scriptsize{$\pm$0.74}
	 & 5.54\scriptsize{$\pm$0.66}
	 & \textbf{1.59\scriptsize{$\pm$0.49}}
	 & 4.50\scriptsize{$\pm$0.37}
     \\

     & \cellcolor{white} \quad + Fewer Epochs
	 & \textbf{4.83\scriptsize{$\pm$1.17}}
	 & \textbf{5.67\scriptsize{$\pm$0.92}}
	 & {4.46\scriptsize{$\pm$0.23}}
	 & \textbf{4.83\scriptsize{$\pm$1.17}}
	 & \underline{4.41\scriptsize{$\pm$0.83}}
	 & 6.90\scriptsize{$\pm$1.73}
	 & \underline{2.00\scriptsize{$\pm$0.34}}
     \\
     
     & \cellcolor{lightergray} \quad + \ours~(Ours)
	 & \cellcolor{lightergray}\underline{5.44\scriptsize{$\pm$0.80}}
	 & \cellcolor{lightergray}\underline{6.06\scriptsize{$\pm$1.54}}
	 & \cellcolor{lightergray}\underline{3.83\scriptsize{$\pm$0.74}}
	 & \cellcolor{lightergray}\underline{5.50\scriptsize{$\pm$1.55}}
	 & \cellcolor{lightergray}\textbf{3.87\scriptsize{$\pm$1.15}}
	 & \cellcolor{lightergray}\underline{2.51\scriptsize{$\pm$0.35}}
	 & \cellcolor{lightergray}\textbf{1.24\scriptsize{$\pm$0.22}}
     \\
     
     \cmidrule{2-9}
    
     & Llama3-8B
	 & 8.49\scriptsize{$\pm$0.14}
	 & 6.76\scriptsize{$\pm$1.77}
	 & \textbf{2.57\scriptsize{$\pm$0.84}}
	 & 8.49\scriptsize{$\pm$0.14}
	 & 3.84\scriptsize{$\pm$0.37}
	 & \textbf{1.88\scriptsize{$\pm$1.18}}
	 & 4.28\scriptsize{$\pm$0.54}
     \\

     & \cellcolor{white} \quad + Fewer Epochs
	 & \underline{4.45\scriptsize{$\pm$0.32}}
	 & \textbf{4.99\scriptsize{$\pm$2.00}}
	 & \underline{2.83\scriptsize{$\pm$0.58}}
	 & \underline{4.45\scriptsize{$\pm$0.32}}
	 & \textbf{3.14\scriptsize{$\pm$0.13}}
	 & \underline{2.71\scriptsize{$\pm$0.25}}
	 & \underline{1.79\scriptsize{$\pm$1.16}}
     \\
     
     & \cellcolor{lightergray} \quad + \ours~(Ours)
	 & \cellcolor{lightergray}\textbf{3.47\scriptsize{$\pm$0.74}}
	 & \cellcolor{lightergray}\underline{5.58\scriptsize{$\pm$0.58}}
	 & \cellcolor{lightergray}4.34\scriptsize{$\pm$1.59}
	 & \cellcolor{lightergray}\textbf{4.07\scriptsize{$\pm$0.28}}
	 & \cellcolor{lightergray}\underline{3.79\scriptsize{$\pm$0.90}}
	 & \cellcolor{lightergray}3.49\scriptsize{$\pm$1.42}
	 & \cellcolor{lightergray}\textbf{1.64\scriptsize{$\pm$0.64}}
     \\
     
     \cmidrule{2-9}

     & Llama3.1-8B
	 & 8.58\scriptsize{$\pm$0.56}
	 & 8.58\scriptsize{$\pm$0.29}
	 & \textbf{2.92\scriptsize{$\pm$0.92}}
	 & 8.58\scriptsize{$\pm$0.56}
	 & \underline{3.85\scriptsize{$\pm$1.18}}
	 & \textbf{2.32\scriptsize{$\pm$0.27}}
	 & 4.74\scriptsize{$\pm$0.28}
     \\

     & \cellcolor{white} \quad + Fewer Epochs
	 & \underline{4.76\scriptsize{$\pm$0.91}}
	 & \textbf{4.23\scriptsize{$\pm$0.95}}
	 & 3.11\scriptsize{$\pm$0.76}
	 & \underline{4.76\scriptsize{$\pm$0.91}}
	 & 3.99\scriptsize{$\pm$0.93}
	 & \underline{3.02\scriptsize{$\pm$0.59}}
	 & \underline{1.45\scriptsize{$\pm$0.38}}
     \\
     
     & \cellcolor{lightergray} \quad + \ours~(Ours)
	 & \cellcolor{lightergray}\textbf{4.45\scriptsize{$\pm$0.36}}
	 & \cellcolor{lightergray}\underline{4.34\scriptsize{$\pm$1.29}}
	 & \cellcolor{lightergray}\underline{2.97\scriptsize{$\pm$0.26}}
	 & \cellcolor{lightergray}\textbf{4.56\scriptsize{$\pm$0.68}}
	 & \cellcolor{lightergray}\textbf{3.55\scriptsize{$\pm$0.55}}
	 & \cellcolor{lightergray}3.16\scriptsize{$\pm$0.45}
	 & \cellcolor{lightergray}\textbf{1.05\scriptsize{$\pm$0.06}}
     \\
     
     \cmidrule{2-9}

     & Mistral-7B-v0.3
	 & 8.02\scriptsize{$\pm$1.68}
	 & 6.98\scriptsize{$\pm$1.18}
	 & 4.12\scriptsize{$\pm$0.13}
	 & 8.02\scriptsize{$\pm$1.68}
	 & 5.99\scriptsize{$\pm$0.48}
	 & 3.17\scriptsize{$\pm$0.55}
	 & 5.05\scriptsize{$\pm$0.88}
     \\

     & \cellcolor{white} \quad + Fewer Epochs
	 & \underline{5.72\scriptsize{$\pm$2.01}}
	 & \underline{4.74\scriptsize{$\pm$1.31}}
	 & \textbf{2.52\scriptsize{$\pm$0.79}}
	 & \underline{5.72\scriptsize{$\pm$2.01}}
	 & \underline{3.50\scriptsize{$\pm$0.75}}
	 & \underline{1.70\scriptsize{$\pm$0.47}}
	 & \underline{2.47\scriptsize{$\pm$1.09}}
     \\
     
     & \cellcolor{lightergray} \quad + \ours~(Ours)
	 & \cellcolor{lightergray}\textbf{4.47\scriptsize{$\pm$2.00}}
	 & \cellcolor{lightergray}\textbf{4.72\scriptsize{$\pm$0.83}}
	 & \cellcolor{lightergray}\underline{2.62\scriptsize{$\pm$0.20}}
	 & \cellcolor{lightergray}\textbf{4.01\scriptsize{$\pm$1.08}}
	 & \cellcolor{lightergray}\underline{4.10\scriptsize{$\pm$0.26}}
	 & \cellcolor{lightergray}\textbf{0.97\scriptsize{$\pm$0.18}}
	 & \cellcolor{lightergray}\textbf{1.68\scriptsize{$\pm$0.53}}
     \\ 
     
     \midrule
    
     \multirow{12}{*}{NLL~($\downarrow$)} 

     & Llama2-7B
	 & \underline{0.58\scriptsize{$\pm$0.01}}
	 & \underline{0.69\scriptsize{$\pm$0.03}}
	 & \textbf{0.35\scriptsize{$\pm$0.00}}
	 & \underline{0.58\scriptsize{$\pm$0.01}}
	 & \underline{0.48\scriptsize{$\pm$0.03}}
	 & \textbf{0.30\scriptsize{$\pm$0.00}}
	 & \textbf{0.43\scriptsize{$\pm$0.00}}
     \\

     & \cellcolor{white} \quad + Fewer Epochs
	 & 0.64\scriptsize{$\pm$0.03}
	 & 0.78\scriptsize{$\pm$0.01}
	 & \underline{0.39\scriptsize{$\pm$0.01}}
	 & 0.64\scriptsize{$\pm$0.03}
	 & 0.56\scriptsize{$\pm$0.02}
	 & 0.36\scriptsize{$\pm$0.01}
	 & \underline{0.50\scriptsize{$\pm$0.01}}
     \\
     
     & \cellcolor{lightergray} \quad + \ours~(Ours)
	 & \cellcolor{lightergray}\textbf{0.56\scriptsize{$\pm$0.01}}
	 & \cellcolor{lightergray}\textbf{0.68\scriptsize{$\pm$0.02}}
	 & \cellcolor{lightergray}\textbf{0.35\scriptsize{$\pm$0.02}}
	 & \cellcolor{lightergray}\textbf{0.57\scriptsize{$\pm$0.01}}
	 & \cellcolor{lightergray}\textbf{0.46\scriptsize{$\pm$0.03}}
	 & \cellcolor{lightergray}\underline{0.31\scriptsize{$\pm$0.00}}
	 & \cellcolor{lightergray}\textbf{0.43\scriptsize{$\pm$0.00}}
     \\
     
     \cmidrule{2-9}
    
     & Llama3-8B
	 & 0.48\scriptsize{$\pm$0.01}
	 & \underline{0.47\scriptsize{$\pm$0.03}}
	 & \textbf{0.22\scriptsize{$\pm$0.01}}
	 & 0.48\scriptsize{$\pm$0.01}
	 & \textbf{0.35\scriptsize{$\pm$0.01}}
	 & \textbf{0.25\scriptsize{$\pm$0.00}}
	 & \textbf{0.34\scriptsize{$\pm$0.00}}
     \\

     & \cellcolor{white} \quad + Fewer Epochs
	 & \underline{0.46\scriptsize{$\pm$0.01}}
	 & 0.48\scriptsize{$\pm$0.01}
	 & \textbf{0.22\scriptsize{$\pm$0.02}}
	 & \underline{0.46\scriptsize{$\pm$0.01}}
	 & \underline{0.37\scriptsize{$\pm$0.02}}
	 & 0.41\scriptsize{$\pm$0.20}
	 & \underline{0.40\scriptsize{$\pm$0.08}}
     \\
     
     & \cellcolor{lightergray} \quad + \ours~(Ours)
	 & \cellcolor{lightergray}\textbf{0.44\scriptsize{$\pm$0.01}}
	 & \cellcolor{lightergray}\textbf{0.45\scriptsize{$\pm$0.02}}
	 & \cellcolor{lightergray}\underline{0.23\scriptsize{$\pm$0.00}}
	 & \cellcolor{lightergray}\textbf{0.44\scriptsize{$\pm$0.01}}
	 & \cellcolor{lightergray}\textbf{0.35\scriptsize{$\pm$0.01}}
	 & \cellcolor{lightergray}\underline{0.27\scriptsize{$\pm$0.01}}
	 & \cellcolor{lightergray}\textbf{0.34\scriptsize{$\pm$0.00}}
     \\
     
     \cmidrule{2-9}

     & Llama3.1-8B
	 & \underline{0.48\scriptsize{$\pm$0.01}}
	 & 0.53\scriptsize{$\pm$0.01}
	 & \underline{0.24\scriptsize{$\pm$0.03}}
	 & \underline{0.48\scriptsize{$\pm$0.01}}
	 & \textbf{0.33\scriptsize{$\pm$0.03}}
	 & \textbf{0.25\scriptsize{$\pm$0.00}}
	 & \underline{0.35\scriptsize{$\pm$0.00}}
     \\

     & \cellcolor{white} \quad + Fewer Epochs
	 & \underline{0.48\scriptsize{$\pm$0.00}}
	 & \textbf{0.45\scriptsize{$\pm$0.00}}
	 & \textbf{0.23\scriptsize{$\pm$0.01}}
	 & \underline{0.48\scriptsize{$\pm$0.00}}
	 & \underline{0.37\scriptsize{$\pm$0.01}}
	 & \underline{0.27\scriptsize{$\pm$0.00}}
	 & 0.36\scriptsize{$\pm$0.00}
     \\
     
     & \cellcolor{lightergray} \quad + \ours~(Ours)
	 & \cellcolor{lightergray}\textbf{0.44\scriptsize{$\pm$0.01}}
	 & \cellcolor{lightergray}\underline{0.46\scriptsize{$\pm$0.00}}
	 & \cellcolor{lightergray}\textbf{0.23\scriptsize{$\pm$0.02}}
	 & \cellcolor{lightergray}\textbf{0.44\scriptsize{$\pm$0.01}}
	 & \cellcolor{lightergray}\textbf{0.33\scriptsize{$\pm$0.02}}
	 & \cellcolor{lightergray}\underline{0.27\scriptsize{$\pm$0.00}}
	 & \cellcolor{lightergray}\textbf{0.34\scriptsize{$\pm$0.00}}
     \\
     
     \cmidrule{2-9}

     & Mistral-7B-v0.3
	 & \underline{0.46\scriptsize{$\pm$0.04}}
	 & 0.47\scriptsize{$\pm$0.02}
	 & 0.28\scriptsize{$\pm$0.01}
	 & \underline{0.46\scriptsize{$\pm$0.04}}
	 & 0.36\scriptsize{$\pm$0.03}
	 & \textbf{0.26\scriptsize{$\pm$0.01}}
	 & \underline{0.35\scriptsize{$\pm$0.02}}
     \\

     & \cellcolor{white} \quad + Fewer Epochs
	 & 0.47\scriptsize{$\pm$0.02}
	 & \underline{0.46\scriptsize{$\pm$0.01}}
	 & \textbf{0.25\scriptsize{$\pm$0.01}}
	 & 0.47\scriptsize{$\pm$0.02}
	 & \textbf{0.35\scriptsize{$\pm$0.01}}
	 & \textbf{0.26\scriptsize{$\pm$0.00}}
	 & \underline{0.35\scriptsize{$\pm$0.01}}
     \\
     
     & \cellcolor{lightergray} \quad + \ours~(Ours)
	 & \cellcolor{lightergray}\textbf{0.42\scriptsize{$\pm$0.02}}
	 & \cellcolor{lightergray}\textbf{0.43\scriptsize{$\pm$0.02}}
	 & \cellcolor{lightergray}\underline{0.26\scriptsize{$\pm$0.01}}
	 & \cellcolor{lightergray}\textbf{0.42\scriptsize{$\pm$0.02}}
	 & \cellcolor{lightergray}\textbf{0.33\scriptsize{$\pm$0.01}}
	 & \cellcolor{lightergray}\textbf{0.26\scriptsize{$\pm$0.01}}
	 & \cellcolor{lightergray}\textbf{0.33\scriptsize{$\pm$0.01}}
     \\ 
     
    \bottomrule[0.12em]
    \end{tabular}
 }
\end{center}
\label{tab:llama-others-full}
\vspace{-1em}
\end{table*}

\begin{table*}[t]
\caption{
    \textbf{Performance of \ours when applied to variants of LoRAs~\cite{hu2022lora,zhang2023adalora,kopiczko2023vera,meng2024pissa},} where Accuracy~(\textbf{ACC}) and Expected Calibration Error~(\textbf{ECE}) are reported in percentages. 
    The MLE training for each LoRA variant is conducted with pre-trained Llama3.1-8B model for \textbf{2 epochs} on the concatenated dataset of six commonsense reasoning tasks, with a shared hyperparameter setting. 
    We set the number of samples to $N=10$ for \ours during inference and 
    \hl{rows with shading} indicate training-free Bayesianization methods that use a pre-trained LoRA as their mean. 
    ``$\uparrow$'' and ``$\downarrow$'' indicate that higher and lower values are preferred, respectively. 
    \textbf{Boldface} denotes the best performance. 
}
\begin{center}
\resizebox{1\linewidth}{!}{%
\setlength{\tabcolsep}{5pt}

\begin{tabular}{clccc cccc}
	\toprule[0.12em]
	\multirow{2}{*}[-0.25em]{\textbf{Metric}} & \multirow{2}{*}[-0.25em]{\textbf{Method}} & \multicolumn{7}{c}{\textbf{Datasets}}
     \\
     \cmidrule{3-9}
     & & WG-S
     & ARC-C 
     & ARC-E 
     & WG-M 
     & OBQA%
     & BoolQ
     & Combined
     \\
     \midrule
     \multirow{6}{*}{ACC~($\uparrow$)} 
     
     & LoRA
     & \textbf{81.24\scriptsize{$\pm$0.05}}
	 & \textbf{82.72\scriptsize{$\pm$0.19}}
	 & \textbf{92.11\scriptsize{$\pm$1.05}}
	 & \textbf{81.24\scriptsize{$\pm$0.05}}
	 & \textbf{87.80\scriptsize{$\pm$2.03}}
	 & \textbf{90.20\scriptsize{$\pm$0.11}}
	 & \textbf{86.70\scriptsize{$\pm$0.08}}
     \\
     
     & \cellcolor{lightergray} \quad + \ours~(Ours)
	 & \cellcolor{lightergray}80.66\scriptsize{$\pm$0.70}
	 & \cellcolor{lightergray}82.50\scriptsize{$\pm$0.84}
	 & \cellcolor{lightergray}91.93\scriptsize{$\pm$1.05}
	 & \cellcolor{lightergray}81.22\scriptsize{$\pm$0.83}
	 & \cellcolor{lightergray}87.73\scriptsize{$\pm$1.29}
	 & \cellcolor{lightergray}89.96\scriptsize{$\pm$0.23}
	 & \cellcolor{lightergray}86.45\scriptsize{$\pm$0.33}
     \\
     
     \cmidrule{2-9}
    
     
     

     & VeRA
	 & \textbf{78.24\scriptsize{$\pm$1.03}}
	 & \textbf{82.39\scriptsize{$\pm$2.55}}
	 & \textbf{90.47\scriptsize{$\pm$1.17}}
	 & \textbf{78.24\scriptsize{$\pm$1.03}}
	 & \textbf{86.13\scriptsize{$\pm$0.23}}
	 & \textbf{89.27\scriptsize{$\pm$0.27}}
	 & \textbf{84.93\scriptsize{$\pm$0.50}}
     \\
     
     & \cellcolor{lightergray} \quad + \ours~(Ours)
	 & \cellcolor{lightergray}76.82\scriptsize{$\pm$0.97}
	 & \cellcolor{lightergray}81.27\scriptsize{$\pm$2.34}
	 & \cellcolor{lightergray}90.35\scriptsize{$\pm$0.91}
	 & \cellcolor{lightergray}77.03\scriptsize{$\pm$1.04}
	 & \cellcolor{lightergray}86.07\scriptsize{$\pm$0.64}
	 & \cellcolor{lightergray}88.99\scriptsize{$\pm$0.32}
	 & \cellcolor{lightergray}84.28\scriptsize{$\pm$0.48}
     \\
     
     \cmidrule{2-9}

     & PiSSA
	 & \textbf{81.45\scriptsize{$\pm$1.45}}
	 & \textbf{83.95\scriptsize{$\pm$1.77}}
	 & 92.22\scriptsize{$\pm$0.54}
	 & \textbf{81.45\scriptsize{$\pm$1.45}}
	 & \textbf{88.40\scriptsize{$\pm$0.69}}
	 & \textbf{90.09\scriptsize{$\pm$0.11}}
	 & \textbf{86.83\scriptsize{$\pm$0.51}}
     \\
     
     & \cellcolor{lightergray} \quad + \ours~(Ours)
     & \cellcolor{lightergray}80.77\scriptsize{$\pm$1.42}
	 & \cellcolor{lightergray}82.94\scriptsize{$\pm$1.21}
	 & \cellcolor{lightergray}\textbf{92.40\scriptsize{$\pm$0.66}}
	 & \cellcolor{lightergray}81.32\scriptsize{$\pm$0.78}
	 & \cellcolor{lightergray}88.13\scriptsize{$\pm$0.42}
	 & \cellcolor{lightergray}90.01\scriptsize{$\pm$0.23}
	 & \cellcolor{lightergray}86.61\scriptsize{$\pm$0.43}

     \\ 
     
     \midrule
    
     \multirow{6}{*}{ECE~($\downarrow$)} 
     
     & LoRA
     & 8.58\scriptsize{$\pm$0.56}
	 & 8.58\scriptsize{$\pm$0.29}
	 & \textbf{2.92\scriptsize{$\pm$0.92}}
	 & 8.58\scriptsize{$\pm$0.56}
	 & 3.85\scriptsize{$\pm$1.18}
	 & \textbf{2.32\scriptsize{$\pm$0.27}}
	 & 4.74\scriptsize{$\pm$0.28}
     \\
     
     & \cellcolor{lightergray} \quad + \ours~(Ours)
	 & \cellcolor{lightergray}\textbf{4.45\scriptsize{$\pm$0.36}}
	 & \cellcolor{lightergray}\textbf{4.34\scriptsize{$\pm$1.29}}
	 & \cellcolor{lightergray}2.97\scriptsize{$\pm$0.26}
	 & \cellcolor{lightergray}\textbf{4.56\scriptsize{$\pm$0.68}}
	 & \cellcolor{lightergray}\textbf{3.55\scriptsize{$\pm$0.55}}
	 & \cellcolor{lightergray}3.16\scriptsize{$\pm$0.45}
	 & \cellcolor{lightergray}\textbf{1.05\scriptsize{$\pm$0.06}}
     \\
     
     \cmidrule{2-9}
    
     
     

     & VeRA
	 & 9.54\scriptsize{$\pm$0.47}
	 & 7.26\scriptsize{$\pm$2.62}
	 & 3.72\scriptsize{$\pm$0.86}
	 & 9.54\scriptsize{$\pm$0.47}
	 & 5.41\scriptsize{$\pm$0.78}
	 & 2.28\scriptsize{$\pm$0.40}
	 & 5.11\scriptsize{$\pm$0.55}
     \\
     
     & \cellcolor{lightergray} \quad + \ours~(Ours)
	 & \cellcolor{lightergray}\textbf{5.03\scriptsize{$\pm$0.92}}
	 & \cellcolor{lightergray}\textbf{5.92\scriptsize{$\pm$1.53}}
	 & \cellcolor{lightergray}\textbf{2.80\scriptsize{$\pm$0.57}}
	 & \cellcolor{lightergray}\textbf{5.09\scriptsize{$\pm$0.87}}
	 & \cellcolor{lightergray}\textbf{3.31\scriptsize{$\pm$0.84}}
	 & \cellcolor{lightergray}\textbf{1.78\scriptsize{$\pm$0.40}}
	 & \cellcolor{lightergray}\textbf{1.44\scriptsize{$\pm$0.44}}
     \\
     
     \cmidrule{2-9}

     & PiSSA
	 & 7.36\scriptsize{$\pm$0.40}
	 & 8.12\scriptsize{$\pm$1.28}
	 & 2.83\scriptsize{$\pm$1.09}
	 & 7.36\scriptsize{$\pm$0.40}
	 & 3.73\scriptsize{$\pm$1.07}
	 & 2.59\scriptsize{$\pm$0.30}
	 & 4.26\scriptsize{$\pm$0.14}
     \\
     
     & \cellcolor{lightergray} \quad + \ours~(Ours)
	 & \cellcolor{lightergray}\textbf{4.59\scriptsize{$\pm$0.63}}
	 & \cellcolor{lightergray}\textbf{4.97\scriptsize{$\pm$0.63}}
	 & \cellcolor{lightergray}\textbf{2.71\scriptsize{$\pm$0.65}}
	 & \cellcolor{lightergray}\textbf{4.37\scriptsize{$\pm$0.32}}
	 & \cellcolor{lightergray}\textbf{2.96\scriptsize{$\pm$0.16}}
	 & \cellcolor{lightergray}\textbf{1.41\scriptsize{$\pm$0.64}}
	 & \cellcolor{lightergray}\textbf{1.17\scriptsize{$\pm$0.22}}
     \\ 
     
     \midrule
    
     \multirow{6}{*}{NLL~($\downarrow$)} 

     & LoRA
     & 0.48\scriptsize{$\pm$0.01}
	 & 0.53\scriptsize{$\pm$0.01}
	 & 0.24\scriptsize{$\pm$0.03}
	 & 0.48\scriptsize{$\pm$0.01}
	 & \textbf{0.33\scriptsize{$\pm$0.03}}
	 & \textbf{0.25\scriptsize{$\pm$0.00}}
	 & 0.35\scriptsize{$\pm$0.00}
     \\
     
     & \cellcolor{lightergray} \quad + \ours~(Ours)
	 & \cellcolor{lightergray}\textbf{0.44\scriptsize{$\pm$0.01}}
	 & \cellcolor{lightergray}\textbf{0.46\scriptsize{$\pm$0.00}}
	 & \cellcolor{lightergray}\textbf{0.23\scriptsize{$\pm$0.02}}
	 & \cellcolor{lightergray}\textbf{0.44\scriptsize{$\pm$0.01}}
	 & \cellcolor{lightergray}\textbf{0.33\scriptsize{$\pm$0.02}}
	 & \cellcolor{lightergray}0.27\scriptsize{$\pm$0.00}
	 & \cellcolor{lightergray}\textbf{0.34\scriptsize{$\pm$0.00}}
     \\
     
     \cmidrule{2-9}
    
     
     

     & VeRA
	 & 0.54\scriptsize{$\pm$0.01}
	 & 0.53\scriptsize{$\pm$0.05}
	 & 0.29\scriptsize{$\pm$0.03}
	 & 0.54\scriptsize{$\pm$0.01}
	 & 0.41\scriptsize{$\pm$0.03}
	 & \textbf{0.27\scriptsize{$\pm$0.01}}
	 & 0.39\scriptsize{$\pm$0.01}
     \\
     
     & \cellcolor{lightergray} \quad + \ours~(Ours)
	 & \cellcolor{lightergray}\textbf{0.51\scriptsize{$\pm$0.01}}
	 & \cellcolor{lightergray}\textbf{0.51\scriptsize{$\pm$0.02}}
	 & \cellcolor{lightergray}\textbf{0.27\scriptsize{$\pm$0.02}}
	 & \cellcolor{lightergray}\textbf{0.50\scriptsize{$\pm$0.01}}
	 & \cellcolor{lightergray}\textbf{0.39\scriptsize{$\pm$0.02}}
	 & \cellcolor{lightergray}0.28\scriptsize{$\pm$0.01}
	 & \cellcolor{lightergray}\textbf{0.38\scriptsize{$\pm$0.01}}
     \\
     
     \cmidrule{2-9}

     & PiSSA
	 & 0.47\scriptsize{$\pm$0.01}
	 & 0.49\scriptsize{$\pm$0.02}
	 & \textbf{0.23\scriptsize{$\pm$0.02}}
	 & 0.47\scriptsize{$\pm$0.01}
	 & \textbf{0.32\scriptsize{$\pm$0.03}}
	 & \textbf{0.26\scriptsize{$\pm$0.00}}
	 & 0.35\scriptsize{$\pm$0.00}
     \\
     
     & \cellcolor{lightergray} \quad + \ours~(Ours)
	 & \cellcolor{lightergray}\textbf{0.44\scriptsize{$\pm$0.01}}
	 & \cellcolor{lightergray}\textbf{0.46\scriptsize{$\pm$0.02}}
	 & \cellcolor{lightergray}\textbf{0.23\scriptsize{$\pm$0.01}}
	 & \cellcolor{lightergray}\textbf{0.44\scriptsize{$\pm$0.01}}
	 & \cellcolor{lightergray}\textbf{0.32\scriptsize{$\pm$0.02}}
	 & \cellcolor{lightergray}\textbf{0.26\scriptsize{$\pm$0.00}}
	 & \cellcolor{lightergray}\textbf{0.33\scriptsize{$\pm$0.00}}
     \\ 
     
    \bottomrule[0.12em]
    \end{tabular}
 }
\end{center}
\label{tab:lora-others-full}
\vspace{-1em}
\end{table*}

\begin{table*}[h]
\caption{
        \textbf{Performance of  {Last-Layer \ours (\texttt{LL} \ours)} applied to LoRA on Llama3.1-8B pre-trained weights,} where Accuracy~(\textbf{ACC}) and Expected Calibration Error~(\textbf{ECE}) are reported in percentages. The evaluation is done across six common-sense reasoning tasks with a shared hyper-parameter setting after 5 epochs.
    We sample $N$ times during inference in the sampling-based methods. 
    \hl{Rows with shading} indicate training-free Bayesianization methods that use a pre-trained LoRA as their mean. 
    ``$\uparrow$'' and ``$\downarrow$'' indicate that higher and lower values are preferred, respectively. 
        \textbf{Boldface} and    \underline{underlining} denote the best and the second-best performance, respectively. 
}
\begin{center}
\resizebox{1\linewidth}{!}{%
\setlength{\tabcolsep}{5pt}
\begin{tabular}{clcccc ccc}
	\toprule[0.12em]
	\multirow{2}{*}[-0.25em]{   \textbf{Metric}} & \multirow{2}{*}[-0.25em]{   \textbf{Method}} & \multirow{2}{*}[-0.25em]{ \textbf{\#Sample~(N)}} & \multicolumn{6}{c}{\textbf{Datasets}}
     \\
     \cmidrule{4-9}
     & & & WG-S
     & ARC-C
     & ARC-E
     & WG-M
     & OBQA
     & BoolQ
     \\
     \midrule
     \multirow{13}{*}{ACC~($\uparrow$)}

     & \loramle & -
     & \textbf{77.87\scriptsize{$\pm$0.54}}  & 81.08\scriptsize{$\pm$0.48}  & \underline{91.67\scriptsize{$\pm$0.36}}  & 82.30\scriptsize{$\pm$0.53}  & 87.90\scriptsize{$\pm$0.87}  & 89.58\scriptsize{$\pm$0.26} 
     \\

     & \cellcolor{lightergray} \quad + \ours
     & \cellcolor{lightergray}10
	 & \cellcolor{lightergray}77.44\scriptsize{$\pm$0.30}  & \cellcolor{lightergray}82.53\scriptsize{$\pm$1.00}  & \cellcolor{lightergray}91.33\scriptsize{$\pm$0.37}  & \cellcolor{lightergray}82.53\scriptsize{$\pm$0.56}  & \cellcolor{lightergray}\textbf{88.53\scriptsize{$\pm$0.57}}  & \cellcolor{lightergray}89.75\scriptsize{$\pm$0.25} 
     \\

     & \cellcolor{lightergray} \quad + \texttt{LL} \ours
     & \cellcolor{lightergray}10  & \cellcolor{lightergray}76.96\scriptsize{$\pm$0.46}  & \cellcolor{lightergray}82.00\scriptsize{$\pm$0.40}  & \cellcolor{lightergray}90.97\scriptsize{$\pm$0.34}  & \cellcolor{lightergray}82.67\scriptsize{$\pm$0.49}  & \cellcolor{lightergray}87.80\scriptsize{$\pm$1.07}  & \cellcolor{lightergray}89.62\scriptsize{$\pm$0.12}
	
     \\

     & \cellcolor{lightergray} \quad + \texttt{LL} \ours
     & \cellcolor{lightergray}100 & \cellcolor{lightergray}77.39\scriptsize{$\pm$0.32}  & \cellcolor{lightergray}82.13\scriptsize{$\pm$0.82}  & \cellcolor{lightergray}91.33\scriptsize{$\pm$0.37}  & \cellcolor{lightergray}82.61\scriptsize{$\pm$0.55}  & \cellcolor{lightergray}87.80\scriptsize{$\pm$0.91}  & \cellcolor{lightergray}89.66\scriptsize{$\pm$0.28} 
	 
     \\
     
     & \loramap & -
     & 76.90\scriptsize{$\pm$0.97}  & 81.08\scriptsize{$\pm$2.48}  & 91.61\scriptsize{$\pm$0.44}  & 82.59\scriptsize{$\pm$0.28}  & 85.73\scriptsize{$\pm$0.19}  & \underline{90.09\scriptsize{$\pm$0.28}}
     \\

     & \cellcolor{lightergray} \quad + \ours~(Ours) 
     & \cellcolor{lightergray}10
	 & \cellcolor{lightergray}76.43\scriptsize{$\pm$0.72}  & \cellcolor{lightergray}82.80\scriptsize{$\pm$1.42}  & \cellcolor{lightergray}91.39\scriptsize{$\pm$0.37}  & \cellcolor{lightergray}82.64\scriptsize{$\pm$0.58}  & \cellcolor{lightergray}86.00\scriptsize{$\pm$0.16}  & \cellcolor{lightergray}89.96\scriptsize{$\pm$0.18}
     \\

     & \cellcolor{lightergray} \quad + \texttt{LL} \ours & \cellcolor{lightergray}10 & \cellcolor{lightergray}76.35\scriptsize{$\pm$0.89}  & \cellcolor{lightergray}\underline{83.07\scriptsize{$\pm$1.97}}  & \cellcolor{lightergray}91.15\scriptsize{$\pm$0.52}  & \cellcolor{lightergray}82.27\scriptsize{$\pm$0.53}  & \cellcolor{lightergray}85.27\scriptsize{$\pm$0.19}  & \cellcolor{lightergray}\underline{90.09\scriptsize{$\pm$0.20}}
	
     \\

     & \cellcolor{lightergray} \quad + \texttt{LL} \ours & \cellcolor{lightergray}100  & \cellcolor{lightergray}76.72\scriptsize{$\pm$0.77}  & \cellcolor{lightergray}\underline{83.07\scriptsize{$\pm$2.12}}  & \cellcolor{lightergray}91.15\scriptsize{$\pm$0.60}  & \cellcolor{lightergray}82.53\scriptsize{$\pm$0.33}  & \cellcolor{lightergray}85.60\scriptsize{$\pm$0.16}  & \cellcolor{lightergray}90.02\scriptsize{$\pm$0.14}
	
     \\

     & \blob & 10 & 76.45\scriptsize{$\pm$0.37}  & 82.32\scriptsize{$\pm$1.15}  & 91.14\scriptsize{$\pm$0.54}  & 82.01\scriptsize{$\pm$0.56}  & 87.57\scriptsize{$\pm$0.21}  & 89.65\scriptsize{$\pm$0.15} 
     \\

     & \blob-Mean & 
        -& 77.72\scriptsize{$\pm$0.12}  & {82.60\scriptsize{$\pm$0.60}}  & 91.64\scriptsize{$\pm$0.55}  & \textbf{83.92\scriptsize{$\pm$0.48}}  & {88.00\scriptsize{$\pm$0.80}}  & 89.86\scriptsize{$\pm$0.05}
     \\

     & \cellcolor{lightergray} \quad + \ours~(Ours) 
     & \cellcolor{lightergray}10
	 & \cellcolor{lightergray}\underline{77.81\scriptsize{$\pm$0.36}}  & \cellcolor{lightergray}\textbf{83.33\scriptsize{$\pm$0.19}}  & \cellcolor{lightergray}\textbf{91.76\scriptsize{$\pm$0.48}}  & \cellcolor{lightergray}\underline{83.81\scriptsize{$\pm$0.39}}  & \cellcolor{lightergray}87.80\scriptsize{$\pm$0.16}  & \cellcolor{lightergray}\textbf{90.11\scriptsize{$\pm$0.28}}
     \\

     & \cellcolor{lightergray} \quad + \texttt{LL} \ours & \cellcolor{lightergray}10
	& \cellcolor{lightergray}77.57\scriptsize{$\pm$1.02}  & \cellcolor{lightergray}82.80\scriptsize{$\pm$0.33}  & \cellcolor{lightergray}91.45\scriptsize{$\pm$0.54}  & \cellcolor{lightergray}83.23\scriptsize{$\pm$0.57}  & \cellcolor{lightergray}\underline{88.33\scriptsize{$\pm$0.09}}  & \cellcolor{lightergray}89.85\scriptsize{$\pm$0.13}
     \\

     & \cellcolor{lightergray} \quad + \texttt{LL} \ours & \cellcolor{lightergray}100
	 & \cellcolor{lightergray}77.60\scriptsize{$\pm$0.62}  & \cellcolor{lightergray}\textbf{83.33\scriptsize{$\pm$0.82}}  & \cellcolor{lightergray}91.39\scriptsize{$\pm$0.60}  & \cellcolor{lightergray}83.63\scriptsize{$\pm$0.62}  & \cellcolor{lightergray}87.60\scriptsize{$\pm$0.43}  & \cellcolor{lightergray}90.03\scriptsize{$\pm$0.03} 
     \\

     \midrule
    
     \multirow{13}{*}{ECE~($\downarrow$)} 

     & \loramle & -
     & 17.02\scriptsize{$\pm$0.46}  & 16.35\scriptsize{$\pm$0.68}  & 7.00\scriptsize{$\pm$0.53}  & 13.83\scriptsize{$\pm$0.65}  & 9.77\scriptsize{$\pm$0.81}  & 8.69\scriptsize{$\pm$0.21}
     \\

     & \cellcolor{lightergray} \quad + \ours~(Ours) 
     & \cellcolor{lightergray}10
	 & \cellcolor{lightergray}12.98\scriptsize{$\pm$0.37}  & \cellcolor{lightergray}11.63\scriptsize{$\pm$0.68}  & \cellcolor{lightergray}5.14\scriptsize{$\pm$0.14}  & \cellcolor{lightergray}10.01\scriptsize{$\pm$0.70}  & \cellcolor{lightergray}7.20\scriptsize{$\pm$0.47}  & \cellcolor{lightergray}7.39\scriptsize{$\pm$0.26}
     \\

     & \cellcolor{lightergray} \quad + \texttt{LL} \ours & \cellcolor{lightergray}10 & \cellcolor{lightergray}14.42\scriptsize{$\pm$0.41}  & \cellcolor{lightergray}13.86\scriptsize{$\pm$0.45}  & \cellcolor{lightergray}6.92\scriptsize{$\pm$0.62}  & \cellcolor{lightergray}10.32\scriptsize{$\pm$0.90}  & \cellcolor{lightergray}8.56\scriptsize{$\pm$0.96}  & \cellcolor{lightergray}7.52\scriptsize{$\pm$0.12}
	 
     \\

     & \cellcolor{lightergray} \quad + \texttt{LL} \ours & \cellcolor{lightergray}100 & \cellcolor{lightergray}13.45\scriptsize{$\pm$0.30}  & \cellcolor{lightergray}13.17\scriptsize{$\pm$0.62}  & \cellcolor{lightergray}6.84\scriptsize{$\pm$0.67}  & \cellcolor{lightergray}10.76\scriptsize{$\pm$0.88}  & \cellcolor{lightergray}8.68\scriptsize{$\pm$0.60}  & \cellcolor{lightergray}7.46\scriptsize{$\pm$0.10}
	 
     \\
     
     & \loramap & - & 18.71\scriptsize{$\pm$0.74}  & 15.77\scriptsize{$\pm$1.60}  & 6.62\scriptsize{$\pm$0.64}  & 14.26\scriptsize{$\pm$0.92}  & 12.19\scriptsize{$\pm$0.55}  & 8.40\scriptsize{$\pm$0.25} \\

     & \cellcolor{lightergray} \quad + \ours~(Ours) & \cellcolor{lightergray}10 & \cellcolor{lightergray}14.95\scriptsize{$\pm$0.65}  & \cellcolor{lightergray}11.27\scriptsize{$\pm$2.53}  & \cellcolor{lightergray}5.76\scriptsize{$\pm$0.63}  & \cellcolor{lightergray}10.97\scriptsize{$\pm$1.19}  & \cellcolor{lightergray}9.70\scriptsize{$\pm$0.69}  & \cellcolor{lightergray}6.86\scriptsize{$\pm$0.31}
     \\

     & \cellcolor{lightergray} \quad + \texttt{LL} \ours & \cellcolor{lightergray}10 & \cellcolor{lightergray}16.03\scriptsize{$\pm$0.64}  & \cellcolor{lightergray}12.72\scriptsize{$\pm$1.33}  & \cellcolor{lightergray}6.54\scriptsize{$\pm$0.68}  & \cellcolor{lightergray}12.06\scriptsize{$\pm$1.09}  & \cellcolor{lightergray}11.36\scriptsize{$\pm$0.34}  & \cellcolor{lightergray}7.51\scriptsize{$\pm$0.23}
	
     \\

     & \cellcolor{lightergray} \quad + \texttt{LL} \ours & \cellcolor{lightergray}100 & \cellcolor{lightergray}15.56\scriptsize{$\pm$0.97}  & \cellcolor{lightergray}12.84\scriptsize{$\pm$2.17}  & \cellcolor{lightergray}6.38\scriptsize{$\pm$0.66}  & \cellcolor{lightergray}11.80\scriptsize{$\pm$1.14}  & \cellcolor{lightergray}11.22\scriptsize{$\pm$0.38}  & \cellcolor{lightergray}7.30\scriptsize{$\pm$0.41} 

     \\
     
     & \blob & 10 & 9.93\scriptsize{$\pm$0.22}  & \textbf{5.41\scriptsize{$\pm$1.17}}  & \underline{2.70\scriptsize{$\pm$0.87}}  & 4.28\scriptsize{$\pm$0.64}  & \underline{2.91\scriptsize{$\pm$0.92}}  & \underline{2.58\scriptsize{$\pm$0.25}}
     \\

     & \blob-Mean & - & 15.43\scriptsize{$\pm$0.15}  & 12.41\scriptsize{$\pm$1.52}  & 4.91\scriptsize{$\pm$0.28}  & 9.37\scriptsize{$\pm$1.33}  & 6.44\scriptsize{$\pm$0.15}  & 6.26\scriptsize{$\pm$0.29} 
     \\

     & \cellcolor{lightergray} \quad + \ours~(Ours)
	& \cellcolor{lightergray}10 & \cellcolor{lightergray}\textbf{8.16\scriptsize{$\pm$0.48}}  & \cellcolor{lightergray}{6.48\scriptsize{$\pm$0.36}}  & \cellcolor{lightergray}\textbf{2.44\scriptsize{$\pm$0.50}}  & \cellcolor{lightergray}\underline{3.83\scriptsize{$\pm$0.43}}  & \cellcolor{lightergray}\textbf{2.67\scriptsize{$\pm$0.18}}  & \cellcolor{lightergray}3.10\scriptsize{$\pm$0.59}
     \\

     & \cellcolor{lightergray} \quad + \texttt{LL} \ours & \cellcolor{lightergray}10 & \cellcolor{lightergray}9.68\scriptsize{$\pm$0.70}  & \cellcolor{lightergray}7.20\scriptsize{$\pm$0.91}  & \cellcolor{lightergray}3.01\scriptsize{$\pm$0.66}  & \cellcolor{lightergray}3.94\scriptsize{$\pm$0.78}  & \cellcolor{lightergray}3.33\scriptsize{$\pm$0.93}  & \cellcolor{lightergray}2.96\scriptsize{$\pm$0.30}
	 
     \\

     & \cellcolor{lightergray} \quad + \texttt{LL} \ours & \cellcolor{lightergray}100 & \cellcolor{lightergray}\underline{8.88\scriptsize{$\pm$0.32}}  & \cellcolor{lightergray}\underline{6.47\scriptsize{$\pm$1.55}}  & \cellcolor{lightergray}2.84\scriptsize{$\pm$0.50}  & \cellcolor{lightergray}\textbf{3.40\scriptsize{$\pm$0.82}}  & \cellcolor{lightergray}3.70\scriptsize{$\pm$0.27}  & \cellcolor{lightergray}\textbf{2.51\scriptsize{$\pm$0.46}}
	 
     \\

     \midrule
    
     \multirow{13}{*}{NLL~($\downarrow$)}

     & \loramle & -
      & 0.88\scriptsize{$\pm$0.04}  & 1.20\scriptsize{$\pm$0.11}  & 0.46\scriptsize{$\pm$0.04}  & 0.68\scriptsize{$\pm$0.01}  & 0.61\scriptsize{$\pm$0.06}  & 0.52\scriptsize{$\pm$0.01}
     \\

     & \cellcolor{lightergray} \quad + \ours~(Ours) & \cellcolor{lightergray}10
	 & \cellcolor{lightergray}0.68\scriptsize{$\pm$0.03}  & \cellcolor{lightergray}0.85\scriptsize{$\pm$0.02}  & \cellcolor{lightergray}0.33\scriptsize{$\pm$0.03}  & \cellcolor{lightergray}0.53\scriptsize{$\pm$0.01}  & \cellcolor{lightergray}0.46\scriptsize{$\pm$0.04}  & \cellcolor{lightergray}0.42\scriptsize{$\pm$0.00}
     \\

     & \cellcolor{lightergray} \quad + \texttt{LL} \ours & \cellcolor{lightergray}10 & \cellcolor{lightergray}0.70\scriptsize{$\pm$0.02}  & \cellcolor{lightergray}0.96\scriptsize{$\pm$0.12}  & \cellcolor{lightergray}0.41\scriptsize{$\pm$0.06}  & \cellcolor{lightergray}0.53\scriptsize{$\pm$0.02}  & \cellcolor{lightergray}0.50\scriptsize{$\pm$0.06}  & \cellcolor{lightergray}0.42\scriptsize{$\pm$0.01} 
	 
     \\

     & \cellcolor{lightergray} \quad + \texttt{LL} \ours & \cellcolor{lightergray}100  & \cellcolor{lightergray}0.66\scriptsize{$\pm$0.02}  & \cellcolor{lightergray}0.84\scriptsize{$\pm$0.08}  & \cellcolor{lightergray}0.39\scriptsize{$\pm$0.07}  & \cellcolor{lightergray}0.53\scriptsize{$\pm$0.02}  & \cellcolor{lightergray}0.49\scriptsize{$\pm$0.05}  & \cellcolor{lightergray}0.40\scriptsize{$\pm$0.00}
	  
     \\
     
     & \loramap & -
     & 0.99\scriptsize{$\pm$0.07}  & 1.12\scriptsize{$\pm$0.23}  & 0.46\scriptsize{$\pm$0.03}  & 0.74\scriptsize{$\pm$0.07}  & 0.79\scriptsize{$\pm$0.02}  & 0.52\scriptsize{$\pm$0.01}
     \\

     & \cellcolor{lightergray} \quad + \ours~(Ours) & \cellcolor{lightergray}10
	 & \cellcolor{lightergray}0.77\scriptsize{$\pm$0.05}  & \cellcolor{lightergray}0.80\scriptsize{$\pm$0.15}  & \cellcolor{lightergray}0.38\scriptsize{$\pm$0.03}  & \cellcolor{lightergray}0.57\scriptsize{$\pm$0.05}  & \cellcolor{lightergray}0.61\scriptsize{$\pm$0.03}  & \cellcolor{lightergray}0.40\scriptsize{$\pm$0.01}
     \\

     & \cellcolor{lightergray} \quad + \texttt{LL} \ours & \cellcolor{lightergray}10 & \cellcolor{lightergray}0.80\scriptsize{$\pm$0.07}  & \cellcolor{lightergray}0.88\scriptsize{$\pm$0.19}  & \cellcolor{lightergray}0.43\scriptsize{$\pm$0.02}  & \cellcolor{lightergray}0.60\scriptsize{$\pm$0.05}  & \cellcolor{lightergray}0.65\scriptsize{$\pm$0.01}  & \cellcolor{lightergray}0.43\scriptsize{$\pm$0.02}
	
     \\

      & \cellcolor{lightergray} \quad + \texttt{LL} \ours & \cellcolor{lightergray}100 & \cellcolor{lightergray}0.77\scriptsize{$\pm$0.06}  & \cellcolor{lightergray}0.86\scriptsize{$\pm$0.18}  & \cellcolor{lightergray}0.41\scriptsize{$\pm$0.02}  & \cellcolor{lightergray}0.57\scriptsize{$\pm$0.04}  & \cellcolor{lightergray}0.63\scriptsize{$\pm$0.02}  & \cellcolor{lightergray}0.40\scriptsize{$\pm$0.03} 
	
     \\

     & \blob & 10 & 0.58\scriptsize{$\pm$0.00}  & \textbf{0.51\scriptsize{$\pm$0.03}}  & \textbf{0.23\scriptsize{$\pm$0.01}}  & {0.43\scriptsize{$\pm$0.01}}  &  {0.34\scriptsize{$\pm$0.01}}  & \textbf{0.26\scriptsize{$\pm$0.01}} 
     \\

     & \blob-Mean & - & 0.74\scriptsize{$\pm$0.02}  & 0.73\scriptsize{$\pm$0.04}  & {0.29\scriptsize{$\pm$0.03}}  & 0.47\scriptsize{$\pm$0.03}  & 0.37\scriptsize{$\pm$0.02}  & 0.32\scriptsize{$\pm$0.02}
     \\

      & \cellcolor{lightergray} \quad + \ours & \cellcolor{lightergray}10 & \cellcolor{lightergray}\underline{0.55\scriptsize{$\pm$0.01}}  & \cellcolor{lightergray}\underline{0.53\scriptsize{$\pm$0.04}}  & \cellcolor{lightergray}\textbf{0.23\scriptsize{$\pm$0.02}}  & \cellcolor{lightergray}\underline{0.40\scriptsize{$\pm$0.01}}  & \cellcolor{lightergray}\underline{0.33\scriptsize{$\pm$0.02}}  & \cellcolor{lightergray}\underline{0.27\scriptsize{$\pm$0.01}}
     \\

     & \cellcolor{lightergray} \quad + \texttt{LL} \ours & \cellcolor{lightergray}10 & \cellcolor{lightergray}0.56\scriptsize{$\pm$0.02}  & \cellcolor{lightergray}0.60\scriptsize{$\pm$0.05}  & \cellcolor{lightergray}0.26\scriptsize{$\pm$0.02}  & \cellcolor{lightergray}0.41\scriptsize{$\pm$0.01}  & \cellcolor{lightergray}\underline{0.33\scriptsize{$\pm$0.01}}  & \cellcolor{lightergray}\underline{0.27\scriptsize{$\pm$0.01}} 
	
     \\

     & \cellcolor{lightergray} \quad + \texttt{LL} \ours & \cellcolor{lightergray}100 & \cellcolor{lightergray}\textbf{0.53\scriptsize{$\pm$0.01}}  & \cellcolor{lightergray}{0.54\scriptsize{$\pm$0.04}}  & \cellcolor{lightergray}\underline{0.24\scriptsize{$\pm$0.01}}  & \cellcolor{lightergray}\textbf{0.39\scriptsize{$\pm$0.01}}  & \cellcolor{lightergray}\textbf{0.31\scriptsize{$\pm$0.01}}  & \cellcolor{lightergray}\textbf{0.26\scriptsize{$\pm$0.01}}
	 
     \\

    \bottomrule[0.12em]
    \end{tabular}
 }
\end{center}
\label{tab:app-ll}
\vspace{-1em}
\end{table*}

\begin{table*}[t]
\caption{
    \textbf{Performance of different methods applied to LoRA on Llama2-7B pre-trained weights,} where Accuracy~(\textbf{ACC}) and Expected Calibration Error~(\textbf{ECE}) are reported in percentages. 
    \textbf{``TF?''} denotes whether a method is \textbf{T}raining-\textbf{F}ree.
    The evaluation is done across six common-sense reasoning tasks with a shared hyper-parameter setting after 5,000 gradient steps.
    We sample $N=10$ during inference in all sampling-based methods including \textbf{\blob~\cite{wang2024blob}} and \textbf{\ours}. 
    \hl{Rows with shading} indicate training-free Bayesianization methods that use a pre-trained LoRA as their mean. 
    For \ours, the anchor dataset $\gD$ is set to a randomly sampled subset of the original training set, the performance evaluation metric $l$ is set to accuracy, and the performance drop tolerance is set adaptively to 1\% or 0.5\% based on whether the given mean overfits. 
    ``$\uparrow$'' and ``$\downarrow$'' indicate that higher and lower values are preferred, respectively. 
    \textbf{Boldface} and \underline{underlining} denote the best and the second-best performance, respectively. 
}
\vspace{-1em}
\begin{center}
\resizebox{1\linewidth}{!}{%
\setlength{\tabcolsep}{3pt}
\begin{tabular}{clcccc ccc cccc}
	\toprule[0.12em]
	\multirow{3}{*}[-0.25em]{\textbf{Metric}} 
    & \multirow{3}{*}[-0.25em]{\textbf{Method}} 
    & \multirow{3}{*}[-0.25em]{\textbf{TF?}}
    & \multicolumn{6}{c}{\multirow{2}{*}[-0.25em]{\textbf{In-Distribution Datasets}}} 
    & \multicolumn{4}{c}{\textbf{Out-of-Distribution Datasets} (OBQA$\rightarrow$X)}
     \\
     \cmidrule(lr){10-13}
     
     & & &
     & & & & & 
     & \multicolumn{2}{c}{\emph{Small Shift}}
     & \multicolumn{2}{c}{\emph{Large Shift}}
     \\
     
     \cmidrule(lr){4-9} \cmidrule(lr){10-11} \cmidrule(lr){12-13}
     & & & WG-S %
     & ARC-C%
     & ARC-E%
     & WG-M%
     & OBQA
     & BoolQ%
     & ARC-C%
     & ARC-E%
     & Chem
     & Phy
     \\
     \midrule

     \multirow{11}{*}{ACC~($\uparrow$)} 
     & \loramcd
     & \xmark 
     & 69.46\scriptsize{$\pm$0.62} 
     & 68.69\scriptsize{$\pm$1.30} 
     & 86.21\scriptsize{$\pm$0.46} 
     & \textbf{76.45\scriptsize{$\pm$0.04}} 
     & 81.72\scriptsize{$\pm$0.10} 
     & 87.29\scriptsize{$\pm$0.13} 
     & 69.03\scriptsize{$\pm$0.70} 
     & 76.00\scriptsize{$\pm$1.58} 
     & 42.71\scriptsize{$\pm$0.01}
     & 29.17\scriptsize{$\pm$4.54}
     \\
     
     & \loraens
     & \xmark 
     & {69.57\scriptsize{$\pm$0.66}} 
     & 66.20\scriptsize{$\pm$2.01} 
     & 84.40\scriptsize{$\pm$0.81} 
     & 75.32\scriptsize{$\pm$0.21} 
     & 81.38\scriptsize{$\pm$0.91} 
     & 87.09\scriptsize{$\pm$0.11} 
     & 67.34\scriptsize{$\pm$0.70} 
     & 75.18\scriptsize{$\pm$2.03}
     & \underline{43.75\scriptsize{$\pm$1.04}}
     & 30.56\scriptsize{$\pm$2.62}
      \\
\end{tabular}
 }
\end{center}
\label{tab:main-llama2}
\vspace{-1em}
\end{table*}

\clearpage
\section*{NeurIPS Paper Checklist}

\begin{enumerate}

\item {\bf Claims}
    \item[] Question: Do the main claims made in the abstract and introduction accurately reflect the paper's contributions and scope?
    \item[] Answer: \answerYes{}
    \item[] Justification: Every claim in the Abstract and Introduction is supported either empirically as in \Secref{sec:experiments} or theoretically as in \Secref{sec:theory}.

\item {\bf Limitations}
    \item[] Question: Does the paper discuss the limitations of the work performed by the authors?
    \item[] Answer: \answerYes{} 
    \item[] Justification: Limitations are addressed in \appref{sec:limitations}.

\item {\bf Theory assumptions and proofs}
    \item[] Question: For each theoretical result, does the paper provide the full set of assumptions and a complete (and correct) proof?
    \item[] Answer: \answerYes{} 
    \item[] Justification: Yes, the assumptions are explicitly listed in \Assumptionref{assumption:tfb}, \Secref{sec:theory}.

    \item {\bf Experimental result reproducibility}
    \item[] Question: Does the paper fully disclose all the information needed to reproduce the main experimental results of the paper to the extent that it affects the main claims and/or conclusions of the paper (regardless of whether the code and data are provided or not)?
    \item[] Answer: \answerYes{} 
    \item[] Justification: Details of implementations of this paper and the baseline models are disclosed in \Secref{sec:experiments} and \appref{app:implementation}.

\item {\bf Open access to data and code}
    \item[] Question: Does the paper provide open access to the data and code, with sufficient instructions to faithfully reproduce the main experimental results, as described in supplemental material?
    \item[] Answer: \answerNo{} 
    \item[] Justification: At this time, we do not provide code for our method. We will release the code once accepted. 

\item {\bf Experimental setting/details}
    \item[] Question: Does the paper specify all the training and test details (e.g., data splits, hyperparameters, how they were chosen, type of optimizer, etc.) necessary to understand the results?
    \item[] Answer: \answerYes{} 
    \item[] Justification: Our algorithm has only one parameter to optimized and is fully discussed in our algorithmic description and experimental setup.

\item {\bf Experiment statistical significance}
    \item[] Question: Does the paper report error bars suitably and correctly defined or other appropriate information about the statistical significance of the experiments?
    \item[] Answer: \answerYes{} 
    \item[] Justification: All of our experiments are repeated with 3 different random seeds, the mean and the standard deviation of which are reported in our paper.

\item {\bf Experiments compute resources}
    \item[] Question: For each experiment, does the paper provide sufficient information on the computer resources (type of compute workers, memory, time of execution) needed to reproduce the experiments?
    \item[] Answer: \answerYes{} 
    \item[] Justification: Our proposed method is training-free, with only forward passes of models required, which does not rely on high-profile computational resources. We have detailed our computational resource type in \appref{app:more-efficiency}.
    
\item {\bf Code of ethics}
    \item[] Question: Does the research conducted in the paper conform, in every respect, with the NeurIPS Code of Ethics \url{https://neurips.cc/public/EthicsGuidelines}?
    \item[] Answer: \answerYes{} 
    \item[] Justification: N/A.

\item {\bf Broader impacts}
    \item[] Question: Does the paper discuss both potential positive societal impacts and negative societal impacts of the work performed?
    \item[] Answer: \answerYes{} 
    \item[] Justification: Discussed in \appref{sec:limitations}.
    
\item {\bf Safeguards}
    \item[] Question: Does the paper describe safeguards that have been put in place for responsible release of data or models that have a high risk for misuse (e.g., pretrained language models, image generators, or scraped datasets)?
    \item[] Answer: \answerNA{} 
    \item[] Justification: This paper poses no such risks.

\item {\bf Licenses for existing assets}
    \item[] Question: Are the creators or original owners of assets (e.g., code, data, models), used in the paper, properly credited and are the license and terms of use explicitly mentioned and properly respected?
    \item[] Answer: \answerYes{} 
    \item[] Justification: No licensing issues involved. 

\item {\bf New assets}
    \item[] Question: Are new assets introduced in the paper well documented and is the documentation provided alongside the assets?
    \item[] Answer: \answerNA{} 
    \item[] Justification: For now there is no new assets added by this submission. 

\item {\bf Crowdsourcing and research with human subjects}
    \item[] Question: For crowdsourcing experiments and research with human subjects, does the paper include the full text of instructions given to participants and screenshots, if applicable, as well as details about compensation (if any)? 
    \item[] Answer: \answerNA{} 
    \item[] Justification: No crowdsourcing is involved. 

\item {\bf Institutional review board (IRB) approvals or equivalent for research with human subjects}
    \item[] Question: Does the paper describe potential risks incurred by study participants, whether such risks were disclosed to the subjects, and whether Institutional Review Board (IRB) approvals (or an equivalent approval/review based on the requirements of your country or institution) were obtained?
    \item[] Answer: \answerNA{} 
    \item[] Justification: No IRB is included.

\item {\bf Declaration of LLM usage}
    \item[] Question: Does the paper describe the usage of LLMs if it is an important, original, or non-standard component of the core methods in this research? Note that if the LLM is used only for writing, editing, or formatting purposes and does not impact the core methodology, scientific rigorousness, or originality of the research, declaration is not required.
    \item[] Answer: \answerNA{} 
    \item[] Justification: No LLM usage for the core method is involved. 

\end{enumerate}

\end{document}